\documentclass[12pt]{article}
\usepackage[section]{placeins}
\usepackage{graphicx} % Required for inserting images
\usepackage{amsmath, amsthm, amsfonts,amssymb}
\usepackage{bm}
\usepackage{indentfirst}
\usepackage{amsmath}
\usepackage{xr}
\usepackage{mathrsfs}
\usepackage{subfiles}
\usepackage{verbatim}
\usepackage{enumitem}
\usepackage{float}
\usepackage{caption}
\usepackage{subcaption}
\usepackage{algorithm}
\usepackage{algpseudocode}
\usepackage{longtable}
\usepackage{multirow}
\usepackage{multicol}
\usepackage{xcolor}
\usepackage{latexsym}
\usepackage{amsthm}
\usepackage{amsthm,amsmath,amssymb}
\usepackage[left=1in,right=1in,top=1.1in,bottom=1.2in]{geometry}
\usepackage{mathrsfs}
\usepackage{natbib}
\setcitestyle{authoryear,round}
\usepackage[colorlinks, citecolor=blue]{hyperref}
\usepackage{enumitem}
\newtheorem{theorem}{Theorem}
\newtheorem{corollary}{Corollary}
\newtheorem{lemma}{Lemma}
\newtheorem{assumption}{Assumption}

\newcommand{\argmin}{\mathop{\rm arg\min}}
\newcommand{\argmax}{\mathop{\rm arg\max}}

\newtheorem{remark}{Remark}
\newcommand{\bE}{\mathbb{E}}
\newcommand{\cP}{\mathcal{P}}

\newcommand{\cR}{\mathcal{R}}

\newcommand{\AUC}{\mathrm{AUC}}
\newcommand{\AUPRC}{\mathrm{AUPRC}}
\newcommand{\aug}{\mathrm{aug}}
\newcommand{\raw}{\mathrm{raw}}
\newcommand{\syn}{\mathrm{syn}}
\newcommand{\Rec}{\mathrm{Rec}}
\newcommand{\Prec}{\mathrm{Prec}}
\newcommand{\BA}{\mathrm{BA}}
\newcommand{\F}{\mathrm{F}}
\newcommand{\FPR}{\mathrm{FPR}}

\newtheorem{example}{Example}

\title{When Does Synthetic Data Augmentation Improve Score-Based Imbalanced Classification?}

\author{Zhengchi Ma\thanks{Department of Electrical \& Computer Engineering, Duke University}, ~ Pengfei Lyu\thanks{Department of Biostatistics \& Bioinformatics, Duke University}, ~ and ~ Anru R. Zhang\thanks{Department of Biostatistics \& Bioinformatics and Department of Computer Science, Duke University}}
\date{}

\begin{document}
\maketitle
% main paper
\begin{abstract}
Synthetic data augmentation is widely used to mitigate class imbalance, but its
theoretical effects on score-based classification remain poorly understood. This
paper develops a framework for characterizing when synthetic minority
augmentation can improve threshold-integrated and threshold-optimized metrics,
including AUROC, AUPRC, best-threshold balanced accuracy, and best-threshold
\(\F_1\) score. We separate the effect of augmentation into two components: a
change in effective class weighting and a discrepancy between the synthetic and
true minority distributions. Under well-specified score models, the raw
estimator already targets the likelihood-ratio ordering, which is
population-optimal for the metrics considered. Consequently, augmentation cannot
provide a fundamental population-level improvement beyond possible finite-sample
variance reduction, and may introduce additional bias through synthetic
distributional error. We further establish minimax lower bounds showing that
the raw estimator already achieves the optimal metric-regret rate in the
well-specified regime. Under misspecification, however, augmentation can play a
qualitatively different role: by changing the effective class balance, it can
alter the restricted-class projection and correct ranking errors induced by the
raw imbalanced objective. We provide explicit improvement bounds quantifying the
roles of approximation error, finite-sample estimation error, and synthetic
distributional error. Simulation studies corroborate the theory, demonstrating
limited gains under well-specification and nontrivial but nonmonotone
improvements under misspecification.
\end{abstract}

\section{Introduction}\label{sec:intro}

\paragraph*{Imbalanced Classification and Synthetic Augmentation.}

A widespread challenge in modern statistics is imbalanced classification, in
which the target classes are observed at markedly different frequencies. In such
settings, standard empirical risk minimization with unweighted losses can be
inadequate because the majority class may dominate the objective, often at the
expense of minority-class performance. The consequences are especially severe
when the minority class represents rare but high-impact events, such as clinical
outcomes in medical data analysis
\citep{salmi2024handling,siddavatam2025hybrid}, financial risks
\citep{chen2024interpretable,breskuviene2024enhancing}, rare ecological
phenomena \citep{zbinden2024imbalance}, or active compounds in drug discovery
\citep{almeida2024overcoming}. These examples underscore the need for
statistically principled approaches to class imbalance, particularly methods
that move beyond treating data augmentation as a purely heuristic remedy
\citep{chen2024survey}.

Synthetic data augmentation has emerged as a practical strategy for mitigating
class imbalance. Rather than relying solely on the limited number of observed
minority-class examples, augmentation methods generate additional samples to
increase the representation of rare classes in the training data. Classical
approaches include bootstrap-type resampling and interpolation-based methods
such as SMOTE and its refinements
\citep{efron1994introduction,chawla2002smote,he2008adasyn,han2005borderline,bunkhumpornpat2009safe,bunkhumpornpat2012dbsmote}.
More recent approaches use conditional generation, Mixup, VAEs, normalizing
flows, GANs, diffusion or score-based models, and attention-based language
models to produce synthetic examples
\citep{zhang2017mixup,tian2025conditional,kingma2013auto,papamakarios2021normalizing,goodfellow2014generative,ho2020denoising,nakada2024synthetic}.
These methods range from simple resampling heuristics to flexible distributional
modeling, making synthetic augmentation an important component of modern
imbalanced learning.

Despite its widespread use, the effect of synthetic augmentation remains
unsettled. Some empirical studies report that oversampling or synthetic
augmentation can improve classification performance, while others find little or
no AUROC improvement, degraded calibration, or performance comparable
to simpler reweighting or threshold-adjustment strategies
\citep{blagus2013smote,van2022harm,piccininni2024understanding,yang2024impact}.
Recent theoretical and methodological work has begun to clarify when synthetic
samples can be useful, emphasizing generator mismatch, local class geometry,
synthetic sample size, bias correction, and inference validity
\citep{nakada2024synthetic,xia2026classification,ma2026synthetic,shen2023boosting,lyu2025bias,ahmad2025concentration,keret2025glm,raisa2025consistent,xu2026generative}.
Taken together, these findings suggest that synthetic augmentation is not merely
a computational device for enlarging training sets; it raises fundamental
questions about bias, uncertainty, distributional mismatch, and the metrics used
to evaluate predictive performance.

\paragraph*{Score-Based Classification.}

This paper studies synthetic augmentation through the lens of score-based classification, where a learned score orders observations by their evidence for the minority class and hard decisions are obtained by thresholding this score. The score-based view is especially useful in imbalanced classification because it is more informative than hard-label prediction. A score preserves the relative strength of evidence across samples and supports threshold selection, uncertainty assessment \citep{sadinle2019least}, ranking \citep{clemenccon2008ranking}, and trade-off analysis between false positives and false negatives \citep{cook2020consult}. Performance can therefore be studied both through ranking-based or threshold-integrated criteria such as AUROC and AUPRC, and through threshold-optimized decision criteria such as best-threshold balanced accuracy and best-threshold \(\F_1\) score. We focus on these four metrics because they evaluate the classification outcome without fixing a single operating threshold in advance.

Empirical evidence on whether synthetic augmentation improves these metrics is
mixed. Several studies in clinical prediction and high-dimensional
classification report that random resampling or SMOTE-type augmentation
provides little or no AUROC improvement
\citep{piccininni2024understanding,van2022harm,zheng2025using,nguyen2019predicting}.
Other applications find that synthetic augmentation can improve AUROC or AUPRC
in downstream prediction tasks
\citep{kim2024synthetic,kannan2025enhancement,li2023generating,chen2026boosting}.
These apparently contradictory findings raise two basic theoretical questions:
\begin{quote}
    {\it When can synthetic data improve threshold-integrated metrics such as
    AUROC and AUPRC, and when should no improvement be expected?}
\end{quote}
and
\begin{quote}
    {\it When is threshold tuning of the raw score sufficient, and when can
    synthetic augmentation improve the learned score beyond what threshold
    tuning alone can achieve?}
\end{quote}

These questions are distinct because threshold tuning and augmentation act at
different stages of the learning pipeline. Threshold tuning changes only the
final decision rule for a fixed score; it cannot change the ordering induced by
that score. Synthetic augmentation, by contrast, changes the training objective
and can therefore alter the fitted score itself. The central distinction developed in this paper is between well-specified settings, where augmentation cannot change the oracle score-ranking target, and misspecified settings, where changing the effective class balance can alter the restricted-class projection and correct ranking errors induced by the raw imbalanced objective.

\subsection{Our Contributions}

\begin{figure}[hbtp]
    \centering
    \includegraphics[width=\linewidth]{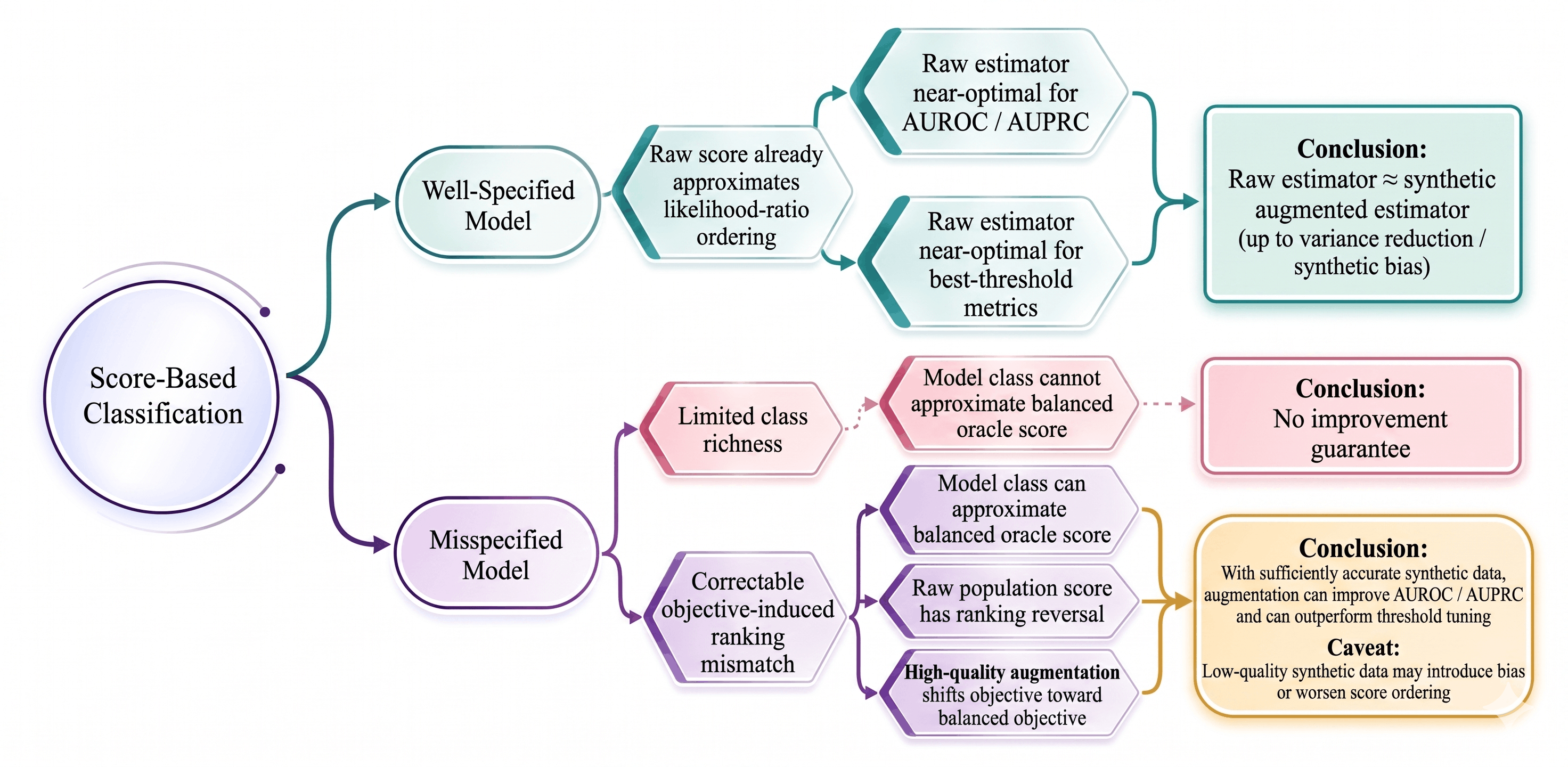}
    \caption{Flowchart of the main results.}
    \label{fig:flowchart}
\end{figure}

This paper develops a theoretical framework for understanding when synthetic
minority-class augmentation can and cannot improve score-based imbalanced
classification. A flowchart summarizing the main results is shown in
Figure~\ref{fig:flowchart}. Our main contributions are as follows.

First, we formulate synthetic data augmentation through population and empirical
risk objectives that separate two effects of augmentation: the change in
effective class weighting and the distributional discrepancy between the
synthetic minority distribution and the true minority distribution. This
formulation allows us to study augmentation for score-based criteria, including
AUROC, AUPRC, best-threshold balanced accuracy, and best-threshold \(\F_1\)
score.

Second, we show that under well-specified score models, synthetic augmentation
does not change the population-optimal ranking target. In this regime, both the
raw and ideal augmented objectives target the same likelihood-ratio ordering.
Augmentation may reduce finite-sample variance by increasing the effective
sample size, but it also introduces synthetic distributional error when
\(P_{\mathrm{syn}}\neq P_1\). We establish finite-sample metric-regret bounds
for raw and augmented estimators and prove minimax lower bounds showing that,
for AUROC and best-threshold balanced accuracy, the raw estimator already
achieves the optimal rate up to dimension-dependent factors.

Third, we identify a qualitatively different phenomenon under model
misspecification. When the model class is too restrictive, augmentation is not
guaranteed to help, because generating more data cannot remove intrinsic
approximation error. In contrast, when the model class can approximate a
better-balanced oracle score but the raw imbalanced objective selects a
misaligned ordering, high-quality synthetic augmentation can improve the learned
score itself. In this setting, augmentation can improve AUROC, AUPRC, and
best-threshold performance beyond what can be achieved by threshold tuning the
raw score alone.

Fourth, we provide explicit improvement guarantees under misspecification. These
bounds quantify how the gain from augmentation is reduced by three terms:
approximation error under the chosen effective class weight, finite-sample
estimation error, and synthetic distributional error. The results show that
synthetic augmentation yields a provable advantage when it corrects a ranking
mismatch induced by the raw imbalanced objective and when the synthetic
distribution is sufficiently accurate.

Finally, we complement the theoretical analysis with simulation studies
illustrating the regimes identified by the theory. The simulations show limited potential for fundamental gains in the well-specified regime, and nontrivial but nonmonotone
improvements under misspecification when high-quality synthetic samples correct
objective-induced ranking errors.

\paragraph*{Organization.}
The rest of this paper is organized as follows.
Section~\ref{sec:notation} introduces the notation and problem formulation for
synthetic minority augmentation and establishes the oracle optimality of the
likelihood-ratio benchmark. Section~\ref{sec:synthetic-augmentation} studies
synthetic augmentation under well-specified models, deriving metric-regret
bounds and minimax lower bounds. Section~\ref{sec:misspecified} turns to
misspecified models and characterizes when synthetic augmentation can, and
cannot, improve score ordering beyond threshold tuning. Section~\ref{sec:simulation}
presents simulation studies illustrating the theoretical findings. Finally,
Section~\ref{sec:discussion} concludes with implications and limitations.

\section{Notation and Problem Formulation}\label{sec:notation}

\paragraph*{Notation.}
We write $\mathbb P(\cdot)$ for the probability of an event and
$\mathbb E_P[\cdot]$ for expectation under a distribution $P$, omitting the
subscript when no ambiguity arises. For vectors, $\|\cdot\|$ denotes the
Euclidean norm. Gradients and Hessians with respect to $\theta$ are denoted by
$\nabla$ and $\nabla^2$, respectively. For symmetric matrices $A$ and $B$, we
write $A\succeq B$ if $A-B$ is positive semidefinite. The 1-Wasserstein distance
between probability measures $P$ and $Q$ is denoted by $W_1(P,Q)$ and is defined by
\[
W_1(P,Q)
=
\inf_{\gamma\in\Pi(P,Q)}
\int d(x,y)\,d\gamma(x,y),
\]
where $\Pi(P,Q)$ is the set of all couplings of $P$ and $Q$. Throughout,
$c$ and $C$ denote positive constants whose values may change from line to
line. We use $p$ for the feature dimension and $d$ for the parameter dimension.

\paragraph*{Score-Based Binary Classification Model.}

We consider a binary classification problem with features
$X\in\mathcal X\subseteq\mathbb R^p$ and label $Y\in\{0,1\}$, where $Y=1$
denotes the minority class and $Y=0$ denotes the majority class. The class
priors are
\[
\pi_0=\mathbb P(Y=0),
\qquad
\pi_1=\mathbb P(Y=1).
\]
Let $P_0$ and $P_1$ denote the class-conditional feature distributions:
\[
X\mid Y=i \sim P_i,
\qquad i=0,1.
\]

The observed training set consists of $n_0$ majority samples drawn from $P_0$
and $n_1$ minority samples drawn from $P_1$. To study synthetic minority
augmentation, we also consider $\tilde n$ synthetic minority samples drawn from
a synthetic distribution $P_{\mathrm{syn}}$.

A score is a real-valued function $s:\mathcal X\to\mathbb R$ assigning each
feature vector $x$ a numerical value, with larger values indicating stronger
evidence for the minority class. When the score is parameterized by
$\theta\in\Theta\subseteq\mathbb R^d$, we write it as $s_\theta$. Given training
data, the learned score is denoted by $\hat s$ in general, or by
$s_{\hat\theta}$ in the parameterized setting. A hard classifier can be obtained
by thresholding the learned score:
\[
\hat y_\tau(x)=\mathbf 1\{\hat s(x)\ge \tau\}.
\]

We do not assume that there is a unique population score that $\hat s$ estimates
pointwise. For the evaluation metrics considered in this paper, the relevant
population object is the ordering induced by the score. The posterior probability is one canonical scoring function, but it is only one representative of an
equivalence class: any strictly increasing transformation of the likelihood
ratio, when it exists, induces the same ranking and the same threshold family.
Thus, throughout the paper, a score refers to a real-valued function used to
rank observations; it need not be calibrated as a probability or constrained to
lie in a prespecified interval.

In many imbalanced classification problems, the primary object of interest is
the ranking induced by the score rather than prediction at a single threshold.
We therefore evaluate scores using population metrics that aggregate performance
over thresholds or optimize over thresholds.

\paragraph*{Population Evaluation Metrics.}

All performance metrics in this paper are evaluated with respect to the original
population distribution with class priors $(\pi_0,\pi_1)$, unless explicitly
subscripted by a data-generating parameter in minimax statements. Synthetic
samples are used only to modify the training objective and do not change the
population distribution used for evaluation. For a score function $s$, we write
$\AUC(s)$ for the population area under the ROC curve and $\AUPRC(s)$ for the
population area under the precision--recall curve.

For independent samples $X_1\sim P_1$ and $X_0\sim P_0$, the AUROC of $s$ is
\[
    \AUC(s)
    =
    \mathbb{E}_{\substack{X_1\sim P_1\\ X_0\sim P_0}}\left[
        \mathbf{1}\{s(X_1)>s(X_0)\}
        + \frac{1}{2}\mathbf{1}\{s(X_1)=s(X_0)\}
    \right].
\]
Thus, AUROC is the probability that a randomly selected minority-class sample
receives a higher score than a randomly selected majority-class sample, with
ties counted as one half.

For a threshold $\tau\in\mathbb R$, define recall and false-positive rate by
\[
\mathrm{Rec}_s(\tau)
=
\mathbb P_1(s(X)\ge \tau),
\qquad
\mathrm{FPR}_s(\tau)
=
\mathbb P_0(s(X)\ge \tau).
\]
The corresponding precision is
\[
    \mathrm{Prec}_s(\tau)
    =
    \frac{\pi_1 \mathrm{Rec}_s(\tau)}
    {\pi_1 \mathrm{Rec}_s(\tau)+\pi_0 \mathrm{FPR}_s(\tau)},
\]
whenever the denominator is positive. The AUPRC of $s$ is the area under the
precision--recall curve traced out as $\tau$ varies, with respect to the
original class priors $(\pi_0,\pi_1)$. When the relevant score distributions are
continuous, we may equivalently parameterize the precision--recall curve by
recall using a quantile threshold $t_r(s)$ satisfying
\[
\mathbb P_1(s(X)\ge t_r(s))=r.
\]
Then the precision at recall level $r$ and the AUPRC are
\[
\mathrm{Prec}_s^*(r)
=
\frac{\pi_1 r}
{\pi_1 r+\pi_0 \mathbb P_0(s(X)\ge t_r(s))},
\qquad
\AUPRC(s)
=
\int_0^1 \mathrm{Prec}_s^*(r)\,dr.
\]

We also consider two threshold-dependent metrics, optimized over the threshold.
Balanced accuracy and $\F_1$ score at threshold $\tau$ are
\[
    \BA(s,\tau)
    =
    \frac{1}{2}
    \left\{
    \mathbb P_1(s(X)\geq\tau)
    +
    \mathbb P_0(s(X)<\tau)
    \right\},
\]
and
\[
    \F_1(s,\tau)
    =
    \frac{2\pi_1 \mathrm{Rec}_s(\tau)}
    {\pi_1\{1+\mathrm{Rec}_s(\tau)\}+\pi_0 \mathrm{FPR}_s(\tau)}.
\]
For any measurable set $B$, the set version of the two scores can be defined as
\[
\BA(B)=\frac12\left(\mathbb P_1(B)+\mathbb P_0(B^c)\right),\quad \F_1(B)=\frac{2\pi_1\mathbb P_1(B)}{\pi_1(1+\mathbb P_1(B))+\pi_0 \mathbb P_0(B)}.
\]
We use their best-threshold versions for score evaluation:
\[
\BA^*(s):=\sup_{\tau\in\mathbb R}\BA(s,\tau),
\qquad
\F_1^*(s):=\sup_{\tau\in\mathbb R}\F_1(s,\tau).
\]

\paragraph*{Likelihood-Ratio Benchmark.}

The likelihood ratio provides a population-level benchmark for score-based
classification. When emphasizing the data-generating parameter, we write the
class-conditional distributions $P_0$ and $P_1$ as $P_{0,\eta^*}$ and
$P_{1,\eta^*}$. Suppose the data are generated from a distribution indexed by
$\eta^*$ and that $P_{0,\eta^*}$ and $P_{1,\eta^*}$ admit densities
$p_{0,\eta^*}$ and $p_{1,\eta^*}$. The ordinary likelihood ratio is
\[
L_{\eta^*}(x)
:=
\frac{p_{1,\eta^*}(x)}{p_{0,\eta^*}(x)},
\]
where we assume $P_{1,\eta^*}\ll P_{0,\eta^*}$. We also define the normalized
likelihood-ratio score
\[
\Lambda_{\eta^*}(x)
:=
\frac{p_{1,\eta^*}(x)}
{p_{0,\eta^*}(x)+p_{1,\eta^*}(x)}
=
\frac{L_{\eta^*}(x)}{1+L_{\eta^*}(x)}.
\]
Since $t\mapsto t/(1+t)$ is strictly increasing on $[0,\infty)$,
$L_{\eta^*}$ and $\Lambda_{\eta^*}$ induce the same ranking and the same
threshold family up to a monotone reparameterization of the threshold.
Therefore, for ranking-based metrics and best-threshold metrics, either score
may be used as the oracle benchmark.

The normalized likelihood-ratio score $\Lambda_{\eta^*}$ should be distinguished
from the posterior probability. By Bayes' rule,
\[
\eta_{\eta^*}(x)
=
\mathbb P(Y=1\mid X=x)
=
\frac{\pi_1 p_{1,\eta^*}(x)}
{\pi_0 p_{0,\eta^*}(x)+\pi_1 p_{1,\eta^*}(x)}
=
\frac{\pi_1 L_{\eta^*}(x)}
{\pi_0+\pi_1 L_{\eta^*}(x)}.
\]
Thus $\eta_{\eta^*}$, $\Lambda_{\eta^*}$, and $L_{\eta^*}$ all induce the same
ranking and the same threshold family up to monotone reparameterization. When no
confusion arises, we write
\[
L:=L_{\eta^*},
\qquad
\Lambda:=\Lambda_{\eta^*}.
\]

The Neyman--Pearson lemma states that, for testing between two simple
hypotheses, the most powerful test at a fixed type-I error level is obtained by
thresholding the likelihood ratio. In the present classification setting, this
means that the likelihood-ratio ordering is the population-optimal ordering of
samples by their relative evidence for the two classes. A learned score
function $\hat s$ can therefore be viewed as an attempt to approximate this
oracle ordering from the training sample.

We use the following lemma to formalize the optimality of the likelihood-ratio
benchmark for the four metrics considered in this paper.

\begin{lemma}[Likelihood-Ratio Score Achieves Optimal Population Metrics]
\label{thm:likelihood-ratio-optimality}
For $X_0\sim P_0$ and $X_1\sim P_1$, let $s:\mathcal X\to\mathbb R$ be any
measurable score such that $\Lambda_{\eta^*}(X_0)$, $\Lambda_{\eta^*}(X_1)$,
$s(X_0)$, and $s(X_1)$ have continuous distributions. Then
\begin{align*}
\AUC(s)
&\leq
\AUC(\Lambda_{\eta^*})
=
\sup_{s'\in\mathcal S_c} \AUC(s'),\\
\AUPRC(s)
&\leq
\AUPRC(\Lambda_{\eta^*})
=
\sup_{s'\in\mathcal S_c} \AUPRC(s'),\\
\BA^*(s)
&\leq
\BA^*(\Lambda_{\eta^*})
=
\sup_{B\subseteq \mathcal{X}\text{ measurable}}\BA(B),\\
\F_1^*(s)
&\leq
\F_1^*(\Lambda_{\eta^*})=\sup_{B\subseteq \mathcal X \text{ measurable}}\F_1(B),
\end{align*}
where
\[
\mathcal S_c
=
\{\text{measurable }s:\ s(X_0),s(X_1)\text{ have continuous distributions}\}.
\]
\end{lemma}

Lemma~\ref{thm:likelihood-ratio-optimality} shows that the likelihood-ratio
ordering is the population-optimal benchmark for the evaluation criteria used
in this paper. For AUROC and AUPRC, optimality follows from the fact that
likelihood-ratio thresholds are optimal at each fixed false-positive or recall
level. For best-threshold balanced accuracy and $\F_1$ score, the optimal decision
regions are also induced by thresholding the likelihood ratio, although the
maximizing threshold depends on the metric. 
Thus, throughout the paper, we use $\Lambda_{\eta^*}$ as the oracle score and study metric regrets that quantify the gap between a learned score and this
oracle benchmark.

\paragraph*{Learning Objectives with Synthetic Minority Augmentation.}

We now define the population and empirical objectives used to learn the score
function. The purpose of this formulation is to separate three objects: the risk
induced by the original imbalanced sample, the risk induced by augmenting the
minority class with synthetic samples, and an auxiliary ideal augmented risk
that uses the post-augmentation class weight but keeps the minority distribution
equal to the true distribution $P_1$.

Let $\ell(s;x,y)$ denote the loss incurred by score $s$ on an example $(x,y)$.
Here $\ell(s;x,y)$ is shorthand for a loss that depends on $s$ through $s(x)$,
for example $\ell(s(x),y)$. If the score is parameterized by $\theta$, we write
\[
\ell(\theta;x,y):=\ell(s_\theta(x),y).
\]

The raw population risk corresponding to training on the original imbalanced
data is
\[
    \cR_{\raw}(s)
    :=
    \frac{n_0}{n_0+n_1}\mathbb E_{P_0}\ell(s;X,0)
    +
    \frac{n_1}{n_0+n_1}\mathbb E_{P_1}\ell(s;X,1).
\]
When the empirical class proportions reflect the population priors, these
weights converge to $(\pi_0,\pi_1)$.

When synthetic minority samples are added, the effective class proportions in
the training objective change. The augmented population risk is
\[
    \cR_{\aug}(s)
    :=
    \frac{n_0}{n_0+n_1+\tilde n}\mathbb E_{P_0}\ell(s;X,0)
    +
    \frac{n_1}{n_0+n_1+\tilde n}\mathbb E_{P_1}\ell(s;X,1)
    +
    \frac{\tilde n}{n_0+n_1+\tilde n}
    \mathbb E_{P_{\mathrm{syn}}}\ell(s;X,1).
\]
Let $s_{\raw}^*$ and $s_{\aug}^*$ denote minimizers of $\cR_{\raw}$ and
$\cR_{\aug}$, respectively. If the score model is parameterized, we denote the
corresponding population minimizers by $\theta_{\raw}^*$ and
$\theta_{\aug}^*$.

The corresponding empirical objectives are constructed from the observed
samples. Let $\{x_i^{(0)}\}_{i=1}^{n_0}$ be the majority samples,
$\{x_i^{(1)}\}_{i=1}^{n_1}$ the real minority samples, and
$\{\tilde x_i\}_{i=1}^{\tilde n}$ the synthetic minority samples. The raw
empirical risk is
\[
    \hat{\cR}_{\raw}(s)
    :=
    \frac{1}{n_0+n_1}
    \sum_{i=1}^{n_0}\ell(s;x_i^{(0)},0)
    +
    \frac{1}{n_0+n_1}
    \sum_{i=1}^{n_1}\ell(s;x_i^{(1)},1).
\]
The augmented empirical risk is
\[
    \hat{\cR}_{\aug}(s)
    :=
    \frac{1}{n_0+n_1+\tilde n}
    \left\{
    \sum_{i=1}^{n_0}\ell(s;x_i^{(0)},0)
    +
    \sum_{i=1}^{n_1}\ell(s;x_i^{(1)},1)
    +
    \sum_{i=1}^{\tilde n}\ell(s;\tilde x_i,1)
    \right\}.
\]
We denote the corresponding empirical risk minimizers by
$\hat s_{\raw}$ and $\hat s_{\aug}$. If the score model is parameterized, we
denote the corresponding empirical minimizers by
$\hat\theta_{\raw}$ and $\hat\theta_{\aug}$.

It is useful to introduce a unified notation for class-weighted risks. For any
effective minority weight $\alpha\in(0,1)$, define
\[
\cR_\alpha(s)
:=
(1-\alpha)\mathbb E_{P_0}\ell(s;X,0)
+
\alpha\mathbb E_{P_1}\ell(s;X,1).
\]
In particular,
\[
\alpha_{\raw}
:=
\frac{n_1}{n_0+n_1},
\qquad
\alpha_{\aug}
:=
\frac{n_1+\tilde n}{n_0+n_1+\tilde n}.
\]
Then $\cR_{\raw}=\cR_{\alpha_{\raw}}$, while $\cR_{\alpha_{\aug}}$ represents
the idealized objective that would arise if the additional minority samples were
drawn from the true minority distribution $P_1$. Let $s_{\alpha_{\aug}}^*$
denote a minimizer of $\cR_{\alpha_{\aug}}$. Comparing $\cR_{\alpha_{\aug}}$
with $\cR_{\raw}$ isolates the effect of changing class weights, while comparing
$\cR_{\aug}$ with $\cR_{\alpha_{\aug}}$ isolates the distributional discrepancy
caused by using $P_{\mathrm{syn}}$ instead of $P_1$.

Unless otherwise stated, we assume that all minimizers are well defined and that
the loss function and model class satisfy the regularity conditions needed for
empirical risk minimization and the population-level comparisons developed
below.

Because synthetic data generators are often trained on the original data, the
resulting synthetic samples may depend on real samples. For
theoretical simplicity, we assume independence between the real and synthetic
samples. In practice, this can be enforced by sample splitting: one subset of
the original data is used to train the generator, and a disjoint subset is used
to train the classifier \citep{tian2025conditional}. Throughout the paper, we
assume that such preprocessing has been applied when necessary.

For a probability-like score $u\in[0,1]$, we say that the loss $\ell(u,y)$ is strictly proper if, for every $\eta\in[0,1]$, the conditional risk
\[
\eta \ell(u,1)+(1-\eta)\ell(u,0)
\]
is uniquely minimized at $u=\eta$. Equivalently, among all scores interpreted as class probabilities, the Bayes action is the true conditional probability $\eta$.

\paragraph*{Class-Weighted Bayes Scores.}

Assume that $P_1\ll P_0$ so that $L=dP_1/dP_0$ is well defined. If the loss
$\ell$ is strictly proper, then the unconstrained minimizer of $\cR_\alpha$ over
measurable scores $s:\mathcal X\to[0,1]$ is
\[
\eta_\alpha(x)
=
\frac{\alpha L(x)}
{(1-\alpha)+\alpha L(x)}
=
g_\alpha(\Lambda(x)),
\]
where
\[
g_\alpha(t)
=
\frac{\alpha t}
{(1-\alpha)(1-t)+\alpha t}.
\]
The map $g_\alpha$ is strictly increasing on $[0,1]$, and $g_{1/2}(t)=t$. Thus,
changing the class weight changes the calibration of the Bayes score but not
its ranking. For a restricted score class $\mathcal S$, we write
\[
s_\alpha^*
\in
\argmin_{s\in\mathcal S}\cR_\alpha(s)
\]
for the corresponding restricted population minimizer.

\section{Synthetic Augmentation under Well-Specified Models}
\label{sec:synthetic-augmentation}

In this section, we analyze synthetic minority augmentation under well-specified
score models. We begin with the perfectly specified maximum likelihood setting,
where the training loss coincides with the negative log-likelihood of the
data-generating model. This case illustrates a basic principle: when the raw
population objective already recovers a score with the likelihood-ratio
ordering, synthetic augmentation does not change the population-optimal ranking
target. It can only affect finite-sample estimation, and may be harmful when
$P_{\mathrm{syn}}\neq P_1$.

We then abstract this principle into a general well-specification condition,
requiring only that the learned score be close to a strictly increasing
transformation of the likelihood ratio. Under this condition, we show that
score-level closeness controls both threshold-independent and best-threshold
metrics. We then specialize the general results to parametric empirical risk
minimization and derive finite-sample bounds for both raw and augmented
training. Finally, we prove minimax lower bounds showing that, in the
well-specified regime, raw-data estimation already achieves the optimal
metric-regret rate.

\subsection{Perfect Specification: The MLE Benchmark}
\label{sec:mle}

We first consider the ideal case in which the population score model is
perfectly specified. In this setting, the functional form of the population
score is correctly specified up to an unknown data-generating parameter
$\eta^*$. That is, the relevant oracle score belongs to a parametric score
family
\[
\{s_\eta:\eta\in\Theta\}.
\]
The training loss is chosen to match this model, so empirical risk minimization
with the raw data is equivalent to maximum likelihood estimation.

Examples include correctly specified generalized linear score models. For
instance, in logistic regression,
\[
s_\eta(x)=(1+\exp(-\eta^\top x))^{-1},
\]
and perfect specification means that
\[
\mathbb P(Y=1\mid X=x)=s_{\eta^*}(x)
\]
for some $\eta^*\in\Theta$. Similarly, in a probit model,
$s_\eta(x)=\Phi(\eta^\top x)$, where $\Phi$ is the standard normal distribution
function. In these examples, the negative log-likelihood is the matching loss,
and the population minimizer recovers $s_{\eta^*}$, which induces the oracle
likelihood-ratio ordering.

In the perfectly specified MLE setting, the loss $\ell(\eta;x,y)$ is the
negative log-likelihood associated with the score model. Thus the raw empirical
minimizer is the MLE, and the learned score is parametrized by the MLE estimator:
\[
\hat\eta_{\raw}
\in
\argmin_{\eta\in\Theta}
\frac{1}{n_0+n_1}
\sum_{i=1}^{n_0+n_1}
\ell(\eta;X_i,Y_i),\qquad \hat s_{\raw}(x)=s_{\hat\eta_{\raw}}(x).
\]
At the population level, correct specification and identifiability imply that
the raw risk is minimized at the true parameter:
\[
\eta^*
\in
\argmin_{\eta\in\Theta}
\mathcal R_{\raw}(\eta).
\]
Consequently, the raw population objective targets a score that is a strictly
increasing transformation of the oracle likelihood-ratio score
$\Lambda_{\eta^*}$. Therefore, in the perfectly specified MLE setting,
synthetic minority augmentation cannot improve the population-optimal ranking
target. It can only affect the finite-sample estimator, and may be harmful when
$P_{\mathrm{syn}}\neq P_1$.

To connect this observation with class reweighting and ideal minority
augmentation, consider the unified class-weighted risk
\[
\mathcal R_\alpha(\eta)
=
(1-\alpha)\mathbb E_{P_0}\ell_{\mathrm{MLE}}(\eta;X,0)
+
\alpha\mathbb E_{P_1}\ell_{\mathrm{MLE}}(\eta;X,1),
\qquad
\alpha\in(0,1).
\]
Let $\eta_\alpha^*$ be a minimizer of $\mathcal R_\alpha$. For correctly
specified likelihood models in which changing class weights only changes the
effective class prior, the weighted population minimizer changes the
calibration of the score but not the likelihood-ratio ordering. Hence there
exists a strictly increasing function $g_\alpha$ such that
\[
s_{\eta_\alpha^*}(x)=g_\alpha(\Lambda_{\eta^*}(x)).
\]
Thus raw training, class reweighting, and ideal minority augmentation preserve
the same population ranking target.

The MLE case is a canonical example of a broader well-specified regime. The
essential property is not maximum likelihood estimation itself, but
likelihood-ratio alignment: the learned score should be close to a strictly
increasing transformation of the likelihood ratio. In the next subsections, we
formulate this property directly and study its implications.

\subsection{General Well-Specification and Metric Convergence}

We use the following assumption to formalize the notion that a learned score is well-specified for the population metrics considered in this paper.

\begin{assumption}[Well-Specified Score]
\label{assumption:well-specification}
There exists a strictly increasing function $g$ such that
\[
\left\|
\hat s_n-g\circ\Lambda_{\eta^*}
\right\|_{L_\infty(P_0+P_1)}
=
O_P(r_n),
\]
where $r_n\to0$. Equivalently, with probability tending to one, for both $X\sim P_{0,\eta^*}$ and $P_{1,\eta^*}$,
\[
\left|
\hat s_n(X)-g(\Lambda_{\eta^*}(X))
\right|
\lesssim r_n.
\]
\end{assumption}

Assumption~\ref{assumption:well-specification} says that the learned score need
not estimate a unique population score pointwise. Instead, it only needs to
approximate some strictly increasing transformation of the likelihood-ratio
score. This is the relevant notion of well-specification for the
threshold-independent and best-threshold metrics studied here, because these
metrics depend on the ordering and threshold sets induced by the score rather
than on its numerical calibration.

\begin{theorem}[AUROC and AUPRC Regrets]
\label{thm:AUC-invariant}
Suppose Assumption~\ref{assumption:well-specification} holds. Let
$X_0\sim P_0$ and $X_1\sim P_1$ be independent. Assume that
$g(\Lambda_{\eta^*}(X_0))$, $g(\Lambda_{\eta^*}(X_1))$,
$\hat s_n(X_0)$, and $\hat s_n(X_1)$ have continuous distributions and that
there exist constants $C,t_0>0$ such that, for all $t\in[0,t_0]$,
\[
\mathbb P\left(
\left|
g(\Lambda_{\eta^*}(X_1))
-
g(\Lambda_{\eta^*}(X_0))
\right|
\le t
\right)
\le Ct.
\]
Then
\[
\AUC(\Lambda_{\eta^*})-\AUC(\hat s_n)=O_P(r_n),\quad \AUPRC(\Lambda_{\eta^*})-\AUPRC(\hat s_n)
=
O_P\left(
\sqrt{\frac{\pi_0 r_n}{\pi_1}}
\right).
\]
\end{theorem}

Theorem~\ref{thm:AUC-invariant} shows that, under well-specification, both
$\AUC(\hat s_n)$ and $\AUPRC(\hat s_n)$ converge to the corresponding metrics
induced by the likelihood-ratio benchmark. The convergence rate for AUPRC is
slower than that for AUROC, reflecting the fact that AUPRC depends on the full
precision--recall curve and on the class priors. The factor $\pi_0/\pi_1$
highlights the greater sensitivity of AUPRC in severely imbalanced settings.

\begin{theorem}[Best-Threshold BA and $\F_1$ Regrets]
\label{thm:doesnotimproveBA}
Suppose Assumption~\ref{assumption:well-specification} holds. Assume that there
exist constants $C,t_0>0$ such that, for every threshold $\tau$, every
$r\in[0,t_0]$, and each $y\in\{0,1\}$,
\[
\mathbb P_y\left(
\left|
g(\Lambda_{\eta^*}(X))-\tau
\right|
\le r
\right)
\le Cr.
\]
Then
\[
\BA^*(\Lambda_{\eta^*})
-
\BA^*(\hat s_n)
=
O_P(r_n), \qquad \F_1^*(\Lambda_{\eta^*})
-
\F_1^*(\hat s_n)
=
O_P\left(\frac{r_n}{\pi_1}\right).
\]
\end{theorem}

Theorem~\ref{thm:doesnotimproveBA} shows that best-threshold decision
performance also converges to that of the likelihood-ratio benchmark. Although
$\BA(s,\tau)$ and $\F_1(s,\tau)$ depend on a fixed threshold, optimizing over
$\tau$ removes dependence on a particular threshold choice and evaluates the
best decision rule induced by the score. The additional factor $1/\pi_1$ in the
$\F_1$ bound reflects the stronger sensitivity of precision-based metrics to
class imbalance.

\subsection{Specification in Parametric Empirical Risk Minimization Models}

We now specialize the preceding score-level theory to parametric empirical risk
minimization. In this framework, the score is induced by a parameter
$\theta\in\Theta$, and different training procedures correspond to minimizing
class-weighted population or empirical risks. The resulting bounds separate two
effects of synthetic augmentation: the potential variance reduction from
increasing the effective sample size, and the bias introduced when
$P_{\mathrm{syn}}$ differs from the true minority distribution $P_1$.

Recall from Section~\ref{sec:notation} that, for any effective minority weight
$\alpha\in(0,1)$, the unconstrained minimizer of $\cR_\alpha$ is
$\eta_\alpha=g_\alpha\circ\Lambda$. Let
\[
\alpha_{\rm raw}
=
\frac{n_1}{n_0+n_1}
<
\frac12,
\qquad
\alpha_{\rm aug}
=
\frac{n_1+\tilde n}{n_0+n_1+\tilde n}.
\]

\begin{assumption}[Population Well-Specification]
\label{assumption:population-ERM-well-specify}
For every $\alpha\in(0,1)$, there exists a strictly increasing function
$g_\alpha$ such that
\[
s_{\theta_\alpha^*}(x)
=
g_\alpha(\Lambda_{\eta^*}(x)),\quad \text{where }\theta_\alpha^*
\in
\argmin_{\theta\in\Theta}\cR_\alpha(\theta).
\]
\end{assumption}

Assumption~\ref{assumption:population-ERM-well-specify} states that, for each
class-weighted population objective, the corresponding population minimizer
induces a score that is a strictly increasing transformation of the
likelihood-ratio benchmark. Therefore, changing class weights may change the
scale or calibration of the score, but it does not change the oracle ranking
implied by $\Lambda_{\eta^*}$. When an empirical estimator is close to its
corresponding population minimizer, this population-level well-specification
transfers to Assumption~\ref{assumption:well-specification}, with $r_n$
determined by the parameter estimation error.

\begin{assumption}[Regularity Conditions for Parametric ERM]
\label{assumption:ERM-regularity}
The following conditions hold.

\begin{enumerate}[label=(\alph*),leftmargin=*]
    \item \textbf{No ties and anti-concentration.}
    The relevant score distributions have no ties. Moreover, there exist
    constants $C,t_0>0$ such that, for all $t\in[0,t_0]$,
    \[
    \mathbb P\left(
    |s_\theta(X_1)-s_\theta(X_0)|\le t
    \right)
    \le Ct
    \]
    for the relevant parameters $\theta$.

    \item \textbf{Score Lipschitzness.}
    There exists $L_s>0$ such that, for all $\theta,\theta'\in\Theta$,
    \[
    |s_\theta(x)-s_{\theta'}(x)|
    \le
    L_s\|\theta-\theta'\|.
    \]

    \item \textbf{Local curvature.}
    The relevant population risks are locally strongly convex at their
    minimizers: for some $\lambda>0$,
    \[
    \nabla^2\mathcal R_{\raw}(\theta_{\raw}^*)\succeq \lambda I,
    \qquad
    \nabla^2\mathcal R_{\aug}(\theta_{\aug}^*)\succeq \lambda I.
    \]

    \item \textbf{Gradient concentration.}
    The coordinatewise gradients of the loss at the relevant population
    minimizers are uniformly bounded by $B>0$.

    \item \textbf{Synthetic distribution error.}
    The synthetic minority distribution satisfies
    \[
    W_1(P_{\mathrm{syn}},P_1)
    \le
    \epsilon_{\mathrm{syn}}.
    \]
    Moreover, $x\mapsto\nabla\ell(\theta_{\aug}^*;x,1)$ is
    $L_g$-Lipschitz.

    \item \textbf{Ideal augmented-risk curvature.}
    The ideal augmented risk $\mathcal R_{\alpha_{\rm aug}}$ satisfies the
    quadratic growth condition
    \[
    \mathcal R_{\alpha_{\rm aug}}(\theta)
    -
    \mathcal R_{\alpha_{\rm aug}}(\theta_{\alpha_{\rm aug}}^*)
    \ge
    L_{\alpha}\|\theta-\theta_{\alpha_{\rm aug}}^*\|^2
    \]
    in a local neighborhood of $\theta_{\alpha_{\rm aug}}^*$ containing
    $\theta_{\rm aug}^*$. In addition, the following local strong-convexity
    inequality holds in the same neighborhood:
    \[
    \cR_{\alpha_{\rm aug}}(\theta_{\alpha_{\rm aug}}^*)
    \geq
    \cR_{\alpha_{\rm aug}}(\theta)
    +
    \nabla \cR_{\alpha_{\rm aug}}(\theta)^T
    (\theta_{\alpha_{\rm aug}}^*-\theta)
    +
    \frac{\mu}{2}
    \left\|\theta_{\alpha_{\rm aug}}^*-\theta\right\|^2.
    \]

    \item \textbf{Local Hessian regularity and convexity.}
    The parameter space $\Theta$ is convex, and
    $\theta_{\rm raw}^*,\theta_{\rm aug}^*\in\operatorname{int}(\Theta)$.
    For every $(x,y)$, the map
    $\theta\mapsto \ell(\theta;x,y)$ is convex and twice continuously
    differentiable. There exist constants $r,L_H,B_H>0$ such that
    \[
    \{\theta:\|\theta-\theta_{\rm raw}^*\|\le r\}
    \cup
    \{\theta:\|\theta-\theta_{\rm aug}^*\|\le r\}
    \subseteq \operatorname{int}\Theta,
    \qquad
    L_Hr\le \frac{\lambda}{4}.
    \]
    Moreover, for all $\theta$ satisfying
    $\|\theta-\theta_{\rm pop}^*\|\le r$, where
    $\theta_{\rm pop}^*\in\{\theta_{\rm raw}^*,\theta_{\rm aug}^*\}$,
    \[
    \left\|
    \nabla^2\ell(\theta;x,y)
    -
    \nabla^2\ell(\theta_{\rm pop}^*;x,y)
    \right\|_{\rm op}
    \le
    L_H\|\theta-\theta_{\rm pop}^*\|,
    \]
    for $P_j$-almost every $x$ when $y=j\in\{0,1\}$, and for
    $P_{\rm syn}$-almost every $x$ when $y=1$.

    In addition, for $i\in\{0,1\}$ and $X\sim P_i$,
    \[
    \left\|
    \nabla^2\ell(\theta_{\rm pop}^*;X,i)
    -
    E_{P_i}\!\left[\nabla^2\ell(\theta_{\rm pop}^*;X,i)\right]
    \right\|_{\rm op}
    \le B_H
    \]
    almost surely. Furthermore, for $X\sim P_{\rm syn}$,
    \[
    \left\|
    \nabla^2\ell(\theta_{\rm aug}^*;X,1)
    -
    E_{P_{\rm syn}}\!\left[\nabla^2\ell(\theta_{\rm aug}^*;X,1)\right]
    \right\|_{\rm op}
    \le B_H
    \]
    almost surely.
\end{enumerate}
\end{assumption}

These regularity assumptions translate parameter estimation error into
metric-regret bounds. The anti-concentration conditions ensure that small score
perturbations do not change too many rankings or threshold decisions. Score
Lipschitzness converts parameter error into score error. Curvature and bounded
gradients yield concentration of empirical risk minimizers around their
population counterparts. Finally, the Wasserstein bound and gradient
Lipschitzness quantify the bias introduced by replacing true minority samples
with synthetic samples.

\begin{theorem}[AUROC Regret in Well-Specified ERM]
\label{thm:ERM-well-specified}
Under Assumptions~\ref{assumption:population-ERM-well-specify}
and~\ref{assumption:ERM-regularity}, for sufficiently large sample sizes
$n_0,n_1$ and for a constant $c>0$, each of the following bounds holds with
probability at least $1-\delta$:
\[
\AUC(\Lambda_{\eta^*})-\AUC(s_{\hat\theta_{\raw}})
\le
\frac{CL_s c}{\lambda}
\sqrt{
\frac{B^2d\log(6d/\delta)}
{n_0+n_1}
},
\]
and
\[
\AUC(\Lambda_{\eta^*})-\AUC(s_{\hat\theta_{\aug}})
\le
\frac{CL_s c}{\lambda}
\sqrt{
\frac{B^2d\log(8d/\delta)}
{n_0+n_1+\tilde n}
}
+
CL_gL_s\sqrt{\frac{2}{\mu L_{\alpha}}}
\frac{\tilde n}{n_0+n_1+\tilde n}
\epsilon_{\mathrm{syn}}.
\]
\end{theorem}

\begin{theorem}[AUPRC Regret in Well-Specified ERM]
\label{thm:AUPRC-converge-ERM}
Under Assumptions~\ref{assumption:population-ERM-well-specify}
and~\ref{assumption:ERM-regularity}, for sufficiently large sample sizes, and
assuming that $s_\theta(X_1)$ with $X_1\sim P_1$ has a continuous CDF, for a constant
$c>0$, each of the following bounds holds with probability at least $1-\delta$:
\begin{align}\label{eq:rawLambdadiffAUPRC}
\AUPRC(\Lambda_{\eta^*})-\AUPRC(s_{\hat\theta_{\raw}})
\le
\sqrt{
\frac{CL_s c}{\lambda}
\frac{\pi_0}{\pi_1}
}
\left(
\frac{B^2d\log(6d/\delta)}
{n_0+n_1}
\right)^{1/4},
\end{align}
and
\begin{align}\label{eq:augLambdadiffAUPRC}
\nonumber \AUPRC(\Lambda_{\eta^*})-\AUPRC(s_{\hat\theta_{\aug}})
\le\;&
\sqrt{
\frac{CL_s c}{\lambda}
\frac{\pi_0}{\pi_1}
}
\left(
\frac{B^2d\log(8d/\delta)}
{n_0+n_1+\tilde n}
\right)^{1/4}
\\
&+
2
\sqrt{
2CL_s\frac{\pi_0}{\pi_1}
}
\left[
\frac{L_g^2}{2\mu L_{\alpha}}
\left(
\frac{\tilde n}{n_0+n_1+\tilde n}
\epsilon_{\mathrm{syn}}
\right)^2
\right]^{1/4}.
\end{align}
\end{theorem}

Theorems~\ref{thm:ERM-well-specified} and~\ref{thm:AUPRC-converge-ERM} show
that, in the well-specified ERM setting, both raw and augmented estimators
approach the likelihood-ratio benchmark for AUROC and AUPRC. For the raw
estimator, the error is driven by the usual finite-sample estimation term based
on the original sample size $n_0+n_1$. For the augmented estimator, the variance
term can improve because the effective sample size becomes
$n_0+n_1+\tilde n$, but an additional bias term appears due to the discrepancy
between $P_{\mathrm{syn}}$ and $P_1$, measured by $\epsilon_{\mathrm{syn}}$.
Thus, in the well-specified setting, synthetic data augmentation does not yield
a fundamental improvement over raw training unless the reduction in empirical
error offsets the synthetic distributional error.

Similarly, we obtain best-threshold bounds for balanced accuracy and $\F_1$ score.

\begin{theorem}[Best-Threshold BA Regret in Well-Specified ERM]
\label{thm:parametric-ERM-best-BA}
Under Assumptions~\ref{assumption:population-ERM-well-specify}
and~\ref{assumption:ERM-regularity} and for sufficiently large sample sizes,
assume further that there exist constants $C,t_0>0$ such that, for every
relevant parameter $\theta$, threshold $\tau$, $r\in[0,t_0]$, and
$y\in\{0,1\}$,
\begin{equation}\label{eq:thm5anti-concentration}
\mathbb P_y(|s_\theta(X)-\tau|\le r)\le Cr.
\end{equation}
Then, for a constant $c>0$, each of the following bounds holds with probability
at least $1-\delta$:
\begin{align}\label{eq:raw-BA-empirical}
\BA^*(\Lambda_{\eta^*})
-
\BA^*(s_{\hat\theta_{\raw}})
\le
\frac{CL_s c}{\lambda}
\sqrt{
\frac{B^2d\log(6d/\delta)}
{n_0+n_1}
},
\end{align}
and
\begin{align}\label{eq:aug-BA-empirical}
\BA^*(\Lambda_{\eta^*})
-
\BA^*(s_{\hat\theta_{\aug}})
\le
\frac{CL_s c}{\lambda}
\sqrt{
\frac{B^2d\log(8d/\delta)}
{n_0+n_1+\tilde n}
}
+
\frac{CL_sL_g}{\sqrt{2\mu L_\alpha}}
\frac{\tilde n}{n_0+n_1+\tilde n}
\epsilon_{\mathrm{syn}}.
\end{align}
\end{theorem}

\begin{theorem}[$\F_1$ Regret in Well-Specified ERM]
\label{thm:parametric-ERM-F1}
Under the conditions of Theorem~\ref{thm:parametric-ERM-best-BA} and for
sufficiently large sample sizes, for a constant $c>0$, each of the following
bounds holds with probability at least $1-\delta$:
\begin{align}\label{eq:F-1-to-oracle-raw}
\F_1^*(\Lambda_{\eta^*})
-
\F_1^*(s_{\hat\theta_{\raw}})
\le
\frac{CL_s c}{\pi_1\lambda}
\sqrt{
\frac{B^2d\log(6d/\delta)}
{n_0+n_1}
},
\end{align}
and
\begin{align}\label{eq:F-1-to-oracle-aug}
\F_1^*(\Lambda_{\eta^*})
-
\F_1^*(s_{\hat\theta_{\aug}})
\le
\frac{CL_s c}{\pi_1\lambda}
\sqrt{
\frac{B^2d\log(8d/\delta)}
{n_0+n_1+\tilde n}
}
+
\frac{4CL_sL_g}{\pi_1\sqrt{2\mu L_\alpha}}
\frac{\tilde n}{n_0+n_1+\tilde n}
\epsilon_{\rm syn}.
\end{align}
\end{theorem}

Theorems~\ref{thm:parametric-ERM-best-BA} and~\ref{thm:parametric-ERM-F1} show
that the same variance--bias tradeoff persists for best-threshold decision
metrics. Threshold tuning extracts the best decision-level performance available
from the learned score, but it cannot remove the synthetic distributional error
introduced by inaccurate augmentation. The additional factor $1/\pi_1$ in the
$\F_1$ bound again reflects its sensitivity to class imbalance.

The synthetic distributional error $\epsilon_{\mathrm{syn}}$ should be
interpreted as the statistical price of learning the synthetic generator from
the available minority data. Since the generator is typically trained using
only the $n_1$ minority-class observations, the discrepancy
$W_1(P_{\mathrm{syn}},P_1)$ is limited by the accuracy with which the minority
distribution can be estimated. In nonparametric settings, this error may suffer
from the curse of dimensionality, whereas stronger structural assumptions may
yield parametric or near-parametric rates. Therefore, in the well-specified
regime, adding synthetic data does not necessarily improve metric-regret
performance: it may reduce finite-sample variance, but it also introduces a bias
term that can dominate when the synthetic distribution is inaccurate.

The correctly specified MLE benchmark discussed in Section~\ref{sec:mle} is a
special case of the preceding ERM theory, obtained by taking the loss to be the
negative log-likelihood.

\begin{corollary}[Correctly Specified MLE]
\label{cor:mle-well-specified-augmentation}
Suppose Assumption~\ref{assumption:ERM-regularity} and
Equation~(\ref{eq:thm5anti-concentration}) hold with $\ell$ equal to the negative
log-likelihood of a correctly specified and identifiable parametric likelihood
model with a continuous CDF for $s_\theta(X_1)$, where $X_1\sim P_1$. Suppose the
parametric model class is sufficiently rich so that the data-generating score
for any class priors $(\tilde\pi_0,\tilde\pi_1)$ is attainable; that is,
\[
x\mapsto\frac{\tilde\pi_1 p_{1,\eta^*}(x)}
{\tilde\pi_1 p_{1,\eta^*}(x)+\tilde\pi_0 p_{0,\eta^*}(x)}
\]
belongs to the model class. Then, for every $\alpha\in(0,1)$, the corresponding
population MLE score $s_{\eta_\alpha^*}$ is a strictly increasing transformation
of the likelihood-ratio score $\Lambda_{\eta^*}$. Furthermore, the raw MLE and
the synthetic augmented MLE satisfy the metric-regret bounds in
Theorems~\ref{thm:ERM-well-specified},
\ref{thm:AUPRC-converge-ERM},
\ref{thm:parametric-ERM-best-BA}, and
\ref{thm:parametric-ERM-F1}.
\end{corollary}

In particular, the corollary shows that the raw MLE achieves the
likelihood-ratio benchmark at the usual parametric rate, whereas the augmented
MLE has the same variance--bias tradeoff as in the ERM bounds. 

\begin{remark}[Asymptotic versus Limited-Data Regimes]\label{remark-asymp-limitdata}
The upper bounds of
Theorems~\ref{thm:ERM-well-specified}--\ref{thm:parametric-ERM-F1} should be
interpreted differently in the asymptotic and non-asymptotic regimes.
Asymptotically, with $n=n_0+n_1\to\infty$ at fixed model complexity, the raw
estimator attains the parametric rate $O_P(n^{-1/2})$, whereas the augmented
estimator attains only $O_P(n^{-1/2}+\epsilon_{\syn})$. In nonparametric
settings, $\epsilon_{\syn}\asymp O_P(n_1^{-1/p})$, where $p$ is the feature
dimension, which is normally slower than $n^{-1/2}$. Therefore, augmentation cannot
improve the asymptotic rate unless $\epsilon_{\syn}=o_P(n^{-1/2})$, and may
even be harmful. This point is most relevant for low-complexity models, where
the raw estimator reaches its large-sample behavior at moderate $n$.

The limited-data regime common in modern applications is different. With highly
flexible models such as deep networks, the effective complexity can be large
relative to $n$, so the non-asymptotic bounds are the informative ones. Writing
\[
\operatorname{error}_{\raw}
\approx
C\sqrt{\frac{\operatorname{Complexity}(\mathcal S)}{n}},
\qquad
\operatorname{error}_{\aug}
\approx
C\sqrt{\frac{\operatorname{Complexity}(\mathcal S)}{n+\tilde n}}
+
C\epsilon_{\syn},
\]
where $\operatorname{Complexity}(\mathcal S)$ is any suitable capacity measure
such as parameter dimension, VC dimension, or Rademacher complexity, the raw
estimation term is large when complexity is large relative to $n$, and adding
$\tilde n$ synthetic samples can reduce it substantially. Since the resulting
bias $\epsilon_{\syn}$ depends on the feature-space complexity of the generator,
not the full classifier complexity, $\operatorname{error}_{\aug}\ll
\operatorname{error}_{\raw}$ is possible. Thus, even under well-specification,
augmentation can improve finite-sample performance. This is achieved not by changing the
population target, but by reducing estimation error enough to offset the
synthetic bias.
\end{remark}

\subsection{Minimax Lower Bounds for Metric Regret}

We complement the preceding upper bounds with minimax lower bounds for metric
regret. The goal is to show that the convergence rates above are not merely
artifacts of the analysis, but are unavoidable in a broad class of
well-specified models. We consider a local parametric family around a fixed
parameter value and construct nearby distributions whose likelihood-ratio
scores induce different rankings or decision regions. As a result, no estimator
can uniformly achieve a faster rate than the minimax lower bound. Here \(d\) denotes the parameter dimension, so local alternatives are constructed along directions \(u\in\mathbb S^{d-1}\) in parameter space.

\begin{assumption}[Regularity and Separation Conditions for AUROC]
\label{assumption:minimaxAUC}
The following conditions hold.

\begin{enumerate}[label=(\alph*),leftmargin=*]
    \item There exist $\eta_0\in\Theta$ and $r_0>0$ such that
    $B_2(\eta_0,r_0)\subseteq\Theta$.

    \item For each $y\in\{0,1\}$, $P_{y,\eta}$ has density $p_{y,\eta}$ with
    respect to a common dominating measure $\nu$, and $p_{y,\eta}$ is
    continuously differentiable in $\eta$ on $B_2(\eta_0,r_0)$ for every $x$.
    Also, $P_{1,\eta}\ll P_{0,\eta}$ for all
    $\eta\in B_2(\eta_0,r_0)$. Define
    \(
    S_y(x;\eta)=\nabla_\eta\log p_{y,\eta}(x)
    \).
    Assume there exist constants $I_y<\infty$ and $B_{\rm lr}<\infty$ such that
    \[
    \sup_{\eta\in B_2(\eta_0,r_0)}
    \mathbb E_{y,\eta}\|S_y(X;\eta)\|_2^2
    \le I_y,
    \]
    and for all $\eta,\eta'\in B_2(\eta_0,r_0)$, all $t\in[0,1]$, and
    $\eta_t=\eta'+t(\eta-\eta')$,
    \(
    p_{y,\eta_t}(x)/p_{y,\eta'}(x)
    \le B_{\rm lr}
    \)
    for almost every $x$. Let $I_{\max}:=\max\{I_0,I_1\}$.

    \item Define the pairwise likelihood-ratio contrast
    \[
    \Delta_\eta(x,z)
    :=
    \frac{\Lambda_\eta(x)}{1-\Lambda_\eta(x)}
    -
    \frac{\Lambda_\eta(z)}{1-\Lambda_\eta(z)}.
    \]
    Let $Q_\eta:=P_{0,\eta}\otimes P_{0,\eta}$ and $Q_0:=Q_{\eta_0}$. Assume
    that there exists $c_Q>0$ such that,
    for all $\eta\in B_2(\eta_0,r_0)$,
    $Q_\eta\ge c_Q Q_0$ as measures.

    For $h>0$ and $u,u'\in\mathbb S^{d-1}$, define
    \[
    \eta_{h,u}:=\eta_0+hu,
    \qquad
    \eta_{h,u'}:=\eta_0+hu',
    \]
    and define the rank-disagreement set
    \[
    \mathcal R_h(u,u')
    :=
    \left\{
    (x,z):
    \Delta_{\eta_{h,u}}(x,z)
    \Delta_{\eta_{h,u'}}(x,z)<0
    \right\}.
    \]
    Assume there exist constants $\alpha\in(0,1)$, $c_{\rm rk}>0$, and
    $h_0>0$ such that, for every $h\in(0,h_0]$ and every
    $u,u'\in\mathbb S^{d-1}$ satisfying $\|u-u'\|_2\ge\alpha$,
    \[
    \int_{\mathcal R_h(u,u')}
    \min\left\{
    |\Delta_{\eta_{h,u}}(x,z)|,
    |\Delta_{\eta_{h,u'}}(x,z)|
    \right\}
    \,dQ_0(x,z)
    \ge
    c_{\rm rk}h.
    \]
\end{enumerate}
\end{assumption}

Assumption~\ref{assumption:minimaxAUC} imposes the regularity and separation
conditions needed to prove a minimax lower bound for AUROC. The first two
conditions ensure that the model contains a local neighborhood around $\eta_0$
and that the class-conditional distributions vary smoothly with the parameter.
The last condition requires that, for two sufficiently separated local
directions, the corresponding likelihood-ratio scores disagree on a
non-negligible set of pairs. Thus, although the distributions are close in
statistical distance, their oracle rankings differ by order $h$.

\begin{theorem}[AUROC Minimax Lower Bound]
\label{thm:AUROC-minimax}
Under Assumption~\ref{assumption:minimaxAUC}, there exists a constant $c_d$
such that, for any $d\ge c_d$, there exists a constant $c>0$ such that, for all
sufficiently large sample sizes,
\[
\inf_{\widehat s}
\sup_{\eta\in\Theta}
\mathbb E_\eta
\left\{
\AUC_\eta(\Lambda_\eta)-\AUC_\eta(\widehat s)
\right\}
\ge
c\sqrt{\frac d n}
\gtrsim
\frac{1}{\sqrt n}.
\]
\end{theorem}

Theorem~\ref{thm:AUROC-minimax} shows that the $n^{-1/2}$-type rate for AUROC
regret is minimax optimal, up to dimension-dependent factors. Combined with the
upper bounds above, this result shows that the raw estimator already attains
the optimal rate in the well-specified regime. If the synthetic distributional
error is of smaller order than the minimax rate, the augmented estimator can
also attain the minimax rate.

\begin{assumption}[Regularity and Separation Conditions for Balanced Accuracy]
\label{assumption:minimaxBA}
Assume Assumption~\ref{assumption:minimaxAUC}(a)--(b). In addition, assume the
following local balanced-accuracy oracle boundary disagreement condition. For
$h>0$ and $u,u'\in\mathbb S^{d-1}$, define
\[
\eta_{h,u}:=\eta_0+hu,
\qquad
\eta_{h,u'}:=\eta_0+hu',
\]
and define
\[
D_h^{+}(u,u')
:=
B_{\eta_{h,u}}\cap B_{\eta_{h,u'}}^{c}
=
\left\{
x:\Lambda_{\eta_{h,u}}(x)\ge \frac12,\ 
\Lambda_{\eta_{h,u'}}(x)<\frac12
\right\},
\]
and
\[
D_h^{-}(u,u')
:=
B_{\eta_{h,u}}^{c}\cap B_{\eta_{h,u'}}
=
\left\{
x:\Lambda_{\eta_{h,u}}(x)<\frac12,\ 
\Lambda_{\eta_{h,u'}}(x)\ge \frac12
\right\}.
\]
Assume there exist constants $\alpha\in(0,1)$, $c_{\mathrm{BA}}>0$, and
$h_0>0$ such that, for every $h\in(0,h_0]$ and every
$u,u'\in\mathbb S^{d-1}$ satisfying $\|u-u'\|_2\ge\alpha$, both directed lower
bounds hold:
\[
\int_{D_h^{+}(u,u')}
\min\left\{
\left|
\frac{\Lambda_{\eta_{h,u}}(x)}{1-\Lambda_{\eta_{h,u}}(x)}
-1
\right|,
\left|
\frac{\Lambda_{\eta_{h,u'}}(x)}{1-\Lambda_{\eta_{h,u'}}(x)}
-1
\right|
\right\}
\,dP_{0,\eta_0}(x)
\ge c_{\mathrm{BA}}h,
\]
and
\[
\int_{D_h^{-}(u,u')}
\min\left\{
\left|
\frac{\Lambda_{\eta_{h,u}}(x)}{1-\Lambda_{\eta_{h,u}}(x)}
-1
\right|,
\left|
\frac{\Lambda_{\eta_{h,u'}}(x)}{1-\Lambda_{\eta_{h,u'}}(x)}
-1
\right|
\right\}
\,dP_{0,\eta_0}(x)
\ge c_{\mathrm{BA}}h.
\]
\end{assumption}

Assumption~\ref{assumption:minimaxBA} is the analogue of
Assumption~\ref{assumption:minimaxAUC} for best-threshold balanced accuracy.
Instead of pairwise ranking disagreements, it focuses on disagreements near the
balanced-accuracy oracle boundary $\{\Lambda_\eta(x)=1/2\}$. The sets
$D_h^+(u,u')$ and $D_h^-(u,u')$ contain points whose balanced-accuracy oracle
classifications differ under two nearby parameters. The lower bounds require
these disagreement regions to have non-negligible weighted mass of order $h$.

\begin{theorem}[Balanced Accuracy Minimax Lower Bound]
\label{thm:best-BA-minimax}
Under Assumption~\ref{assumption:minimaxBA}, there exists a constant $c_d$ such
that, for any $d\ge c_d$, there exists a constant $c>0$ such that, for all
sufficiently large sample sizes,
\[
\inf_{\widehat s}
\sup_{\eta\in\Theta}
\mathbb E_\eta
\left\{
\BA_\eta^*(\Lambda_\eta)
-
\BA_\eta^*(\widehat s)
\right\}
\ge
c\sqrt{\frac d n}
\gtrsim
\frac1{\sqrt n}.
\]
\end{theorem}

Theorem~\ref{thm:best-BA-minimax} shows that best-threshold balanced accuracy
also has an unavoidable $n^{-1/2}$-type minimax lower bound, up to
dimension-dependent factors. The intuition is similar to the AUROC case, but
the difficulty now comes from uncertainty around the balanced-accuracy oracle
boundary rather than pairwise ranking. 

\section{Synthetic Augmentation under Misspecified Models}
\label{sec:misspecified}

The previous section shows that, under well-specification, synthetic minority
augmentation does not change the oracle population ranking target. Its possible
benefit is finite-sample variance reduction, offset by any bias from
$P_{\mathrm{syn}}\neq P_1$. We now turn to the more delicate and practically
important case of model misspecification. In this regime, the learned score is
not simply a noisy estimate of a likelihood-ratio-aligned oracle; rather, it is
the projection of a class-weighted population objective onto a restricted model
class.

Changing the effective class balance through augmentation can then alter not
only the calibration of the fitted score but also its induced ordering within
the restricted class. At the unrestricted population level, all class-weighted
Bayes scores $\eta_\alpha=g_\alpha\circ\Lambda$ share the same likelihood-ratio
ranking. Thus, changing $\alpha$ does not change the Bayes ordering itself.
Under misspecification, however, the restricted minimizer
$s_\alpha^*\in\argmin_{s\in\mathcal S}\cR_\alpha(s)$ need not preserve this
ordering, and different values of $\alpha$ can lead to very different
approximation errors in $\mathcal S$. We refer to this phenomenon as an
objective-induced ranking mismatch: the raw imbalanced objective may select a
restricted score whose ordering is misaligned with the oracle likelihood-ratio
ordering. When this mismatch is correctable by moving to a larger effective
minority weight, augmentation can improve AUROC, AUPRC, and best-threshold
performance by moving the training objective toward an effective prior whose
Bayes score is better approximated by the model class.

Recall from Section~\ref{sec:notation} that, for any effective minority weight
$\alpha\in(0,1)$, the unconstrained minimizer of $\cR_\alpha$ is
\[
\eta_\alpha(x)
=
g_\alpha(\Lambda(x)),
\qquad
g_\alpha(t)
=
\frac{\alpha t}{(1-\alpha)(1-t)+\alpha t}.
\]
Thus, the central issue in this section is not a change in the Bayes ordering, but misspecification: the difficulty of approximating $\eta_\alpha$ within $\mathcal S$ can depend strongly on $\alpha$.

When $\alpha\ll1$, the transformation $g_\alpha$ becomes sharply nonlinear.
Indeed,
\[
g_\alpha'(t)
=
\frac{\alpha(1-\alpha)}
{\{(1-\alpha)(1-t)+\alpha t\}^2},
\qquad
g_\alpha'(1)=\frac{1-\alpha}{\alpha},
\qquad
g_\alpha''(1)=\frac{2(1-\alpha)(1-2\alpha)}{\alpha^2}.
\]
Thus, for $\alpha\ll1$, the map compresses moderate values of $\Lambda(x)$
toward zero while developing a sharp transition near $t=1$. Hence, even when
$\Lambda$ is well approximated by $\mathcal S$, the composition
$g_\alpha\circ\Lambda=\eta_\alpha$ need not be. The next two examples make this
phenomenon concrete.

\begin{example}[Lipschitz Transformation Class]
\label{example:easier-to-approximate}
Let $K\in(1/2,1)$, and define
\[
\mathcal H_K
=
\left\{
h:[0,1]\to[0,1]:
|h(t)-h(t')|
\le
K|t-t'|,
\quad
\forall\, t,t'\in[0,1]
\right\}.
\]
Consider the restricted score class of likelihood-ratio scores after a
Lipschitz transformation:
\[
S_K
=
\left\{
x\mapsto h(\Lambda(x)):
h\in\mathcal H_K
\right\}.
\]
Since $K<1$, the identity map $t\mapsto t$ is not contained in $\mathcal H_K$,
and hence any non-constant balanced target $\Lambda$ is not exactly represented in $S_K$. We
can show that, whenever the imbalance is severe enough that
$\alpha<1-1/(2K)$,
\[
\inf_{h\in\mathcal H_K}
\sup_{t\in[0,1]}
|h(t)-g_\alpha(t)|
>
\inf_{h\in\mathcal H_K}
\sup_{t\in[0,1]}
|h(t)-t|.
\]
Moreover, let $\mu$ be a probability measure on $\mathcal X$. Suppose there exists a constant $r_0>0$ such that
$\Lambda(X)$, with $X\sim\mu$, admits a density $q$ satisfying
$q(t)\ge \underline q>0$ on $[1-r_0,1]$. Whenever
\[
\frac{450(1-K)^2}{\underline q}
<
\alpha
\le
\min\left\{r_0,\frac{1}{15K}\right\},
\]
we have
\[
\inf_{h\in\mathcal H_K}
\bE_\mu
\left[
\left(
h(\Lambda(X))-\Lambda(X)
\right)^2
\right]
<
\inf_{h\in\mathcal H_K}
\bE_\mu
\left[
\left(
h(\Lambda(X))-g_\alpha(\Lambda(X))
\right)^2
\right].
\]
Thus, under severe imbalance, the balanced likelihood-ratio score can be easier
to approximate than the corresponding imbalanced Bayes score.
\end{example}

\begin{example}[Norm-Constrained Shallow ReLU Network]
\label{example:relu-network}
Let $X=(X_1,\ldots,X_p)\in[0,1]^p$. Assume that the normalized
likelihood-ratio score has the generalized additive form
\[
\Lambda(x)
=
\frac1p\sum_{j=1}^p \lambda_j(x_j),
\]
where each $\lambda_j:[0,1]\to[0,1]$ is twice continuously differentiable and
satisfies $\lambda_j(1)=1$ for $j=1,\ldots,p$. Assume further that there exist
constants $B_1,B_2<\infty$, $0<r_0<1$, and $\lambda_->0$ such that, for every
$j=1,\ldots,p$,
\[
\|\lambda_j'\|_{L^\infty([0,1])}\le B_1,
\qquad
\|\lambda_j''\|_{L^\infty([0,1])}\le B_2,
\]
and
\[
\lambda_j'(t)\ge \lambda_-,
\qquad
t\in[1-r_0,1].
\]
Let $\sigma(u)=u_+=\max\{u,0\}$ be the ReLU activation function. For an integer
$m\ge p$ and a network-norm budget $A>0$, define the shallow ReLU class
\[
\mathcal N_{m,A}^{(p)}
=
\left\{
x\mapsto
\beta_0+\beta^\top x+\sum_{\ell=1}^m a_\ell
\sigma(w_\ell^\top x-t_\ell):
\|\beta\|_2+\sum_{\ell=1}^m |a_\ell|\,\|w_\ell\|_2\le A
\right\}.
\]
Here $m$ is the number of hidden ReLU units, and $A$ is a standard path-norm or
variation-norm budget. Assume the norm regularization is not too restrictive
and the number of hidden ReLU units is not too small relative to the feature
dimension:
\[
A\ge \frac{B_1}{\sqrt p}+B_2,
\qquad
m\geq p\max\left\{1,\sqrt{\frac{B_2}{2}}\right\}.
\]
Also assume the imbalance is severe:
\[
0<\alpha_{\rm raw}
<
\min\left\{
\lambda_- r_0,\,
\frac{2\lambda_-}{A\sqrt p}
\left(\frac14-\frac{B_2}{8(\left\lfloor m/p\right\rfloor+1)^2}\right)
\right\}.
\]
Then
\[
\inf_{s\in\mathcal N_{m,A}^{(p)}}
\|s-\eta_{1/2}\|_{L^\infty([0,1]^p)}
<
\inf_{s\in\mathcal N_{m,A}^{(p)}}
\|s-\eta_{\alpha_{\rm raw}}\|_{L^\infty([0,1]^p)}.
\]
Thus, for sufficiently severe imbalance, the balanced target is strictly easier
to approximate within the same fixed-width, norm-constrained shallow ReLU class.
The detailed proofs of Examples~\ref{example:easier-to-approximate}
and~\ref{example:relu-network} can be found in
Section~\ref{sec:propositions-examples-corollaries} of the supplementary
materials.
\end{example}

The examples show that the target induced by the training prior can be
substantially harder or easier to approximate within a misspecified class. To
quantify this approximation effect, define the $\alpha$-specific approximation
error
\[
\epsilon_\alpha
=
\inf_{s\in\mathcal S}
\left\{
\cR_\alpha(s)-\cR_\alpha(\eta_\alpha)
\right\}.
\]
Both terms use the same $\alpha$ because $\eta_\alpha$ is the unconstrained
minimizer of the same objective $\cR_\alpha$. Thus, $\epsilon_\alpha$ is the
misspecification error of $\mathcal S$ for the objective induced by training
with effective prior $\alpha$. In particular,
\[
\epsilon_{\rm raw}
:=
\epsilon_{\alpha_{\rm raw}},
\qquad
\epsilon_{\rm bal}
:=
\epsilon_{1/2}
=
\inf_{s\in\mathcal S}
\left\{
\cR_{1/2}(s)-\cR_{1/2}(\Lambda)
\right\},
\]
where $\eta_{1/2}=\Lambda$. The quantity $\epsilon_{\rm bal}$ corresponds to
balanced augmentation, $\tilde n=n_0-n_1$.

Adding $\tilde n$ synthetic minority samples sets the effective minority weight to
\[
\alpha(\tilde n)
=
\frac{n_1+\tilde n}{n_0+n_1+\tilde n}.
\]
Equivalently, targeting $\alpha\ge\alpha_{\rm raw}$ requires
\[
\tilde n
=
\frac{\alpha}{1-\alpha}n_0-n_1
\]
synthetic samples. The best achievable approximation error over augmentation
levels is
\[
\epsilon^*
=
\inf_{\alpha\in(\alpha_{\rm raw},1)}\epsilon_\alpha.
\]
When the infimum is attained, let $\alpha^*$ be a minimizer, with corresponding
augmentation size
\[
\tilde n^*
=
\frac{\alpha^*}{1-\alpha^*}n_0-n_1.
\]
In practice, $\alpha^*$, or equivalently $\tilde n^*$, can be selected by
data-driven methods such as $K$-fold cross-validation \citep{ma2026synthetic}.

This perspective separates two regimes. If $\epsilon_{\rm raw}$ is small, the
raw objective is already effectively well-specified for ranking. If instead
$\epsilon_{\rm raw}\gg\epsilon_{\rm bal}\ge\epsilon^*$, the raw objective forces
$\mathcal S$ to approximate the sharply transformed
$\eta_{\alpha_{\rm raw}}$, whereas augmentation toward $\alpha^*$ targets a
smoother or otherwise better-approximated $\eta_\alpha$. Augmentation then
improves discrimination not by changing the Bayes ranking, which is
prior-invariant, but by reducing the objective-induced ranking mismatch and
letting the restricted minimizer track that common ranking more faithfully.

We continue to consider a score class $\mathcal S$, where each score
$s\in\mathcal S$ is parameterized as $s_\theta$ with $\theta\in\Theta$. We
impose the following assumptions.

\begin{assumption}
\label{ass:misspecified-pop}
The following conditions hold.
\begin{itemize}
\item The score takes values in $[0,1]$. There exists a constant
$u_0\in[0,1]$ such that
\[
\ell(u_0,0)=\min_{u\in[0,1]}\ell(u,0),
\]
and there exists $m_0>0$ such that, for all $u\in[0,1]$,
\[
\ell(u,0)-\ell(u_0,0)
\ge
m_0(u-u_0)^2.
\]
Assume the constant score $s_0(x)\equiv u_0$ belongs to $\mathcal S$.

\item There exists $B_{\mathrm{LR}}<\infty$ such that
\[
L_{\eta^*}(x)\le B_{\mathrm{LR}}
\qquad
P_0\text{-almost surely}.
\]
Moreover, there exists $B_s<\infty$ such that, for every non-constant score
$s\in\mathcal S$, the distribution function
\[
F_{0,s}(t)=P_0(s(X)\le t)
\]
is continuous and $B_s$-Lipschitz:
\[
|F_{0,s}(t)-F_{0,s}(t')|
\le
B_s|t-t'|
\qquad
\text{for all }t,t'\in\mathbb R.
\]

\item The minority-risk range is finite:
\[
D_1
:=
\sup_{s,t\in\mathcal S}
\left|
\mathbb E_{P_1}\ell(s(X),1)
-
\mathbb E_{P_1}\ell(t(X),1)
\right|
<\infty.
\]

\item The loss is strictly proper and satisfies the quadratic calibration
condition
\[
C_\eta(u)-C_\eta(\eta)
\ge
m_\ell(u-\eta)^2,
\]
where
\[
C_\eta(u)=\eta\ell(u,1)+(1-\eta)\ell(u,0).
\]

\item The pairwise margin condition holds:
\[
\mathbb P\left(
|\Lambda(X_1)-\Lambda(X_0)|\le t
\right)
\le
C_{\rm pair}t
\]
for all sufficiently small $t>0$, where $X_1\sim P_1$ and $X_0\sim P_0$ are
independent.
\end{itemize}
\end{assumption}

Under these assumptions, we obtain the following improvement guarantees.

\begin{theorem}[Empirical AUROC and AUPRC Improvement under Misspecification]
\label{thm:misspecified-auc}
Under Assumptions~\ref{ass:misspecified-pop} and
\ref{assumption:ERM-regularity}(a)--(g), for any AUROC level
$A\in(1/2,\AUC(\Lambda_{\eta^*}))$, whenever the data imbalance is severe enough
that
\[
\frac{\alpha_{\rm raw}}{1-\alpha_{\rm raw}}D_1
<
m_0\left(\frac{A-\frac12}{B_s(1+\sqrt{B_{\mathrm{LR}}})}\right)^2,
\]
there exists a constant $C>0$ such that, with probability at least $1-\delta$,
\begin{align*}
\AUC(s_{\hat{\theta}_{\aug}})
- & 
\AUC(s_{\hat{\theta}_{\raw}})
\ge\;
\AUC(\Lambda_{\eta^*})-A
-
C
\left(\frac{\epsilon^*}{m_\ell}\right)^{1/3}
\\
&-
\frac{CL_s c}{\lambda}
\sqrt{\frac{B^2d\log(14d/\delta)}{n_0+n_1}}
-
CL_gL_s
\sqrt{\frac{1}{2\mu L_{\alpha}}}
\left(\alpha^*-(1-\alpha^*)\frac{n_1}{n_0}\right)
\epsilon_{\syn}.
\end{align*}

Similarly, for any AUPRC level
$A\in(\pi_1,\AUPRC(\Lambda_{\eta^*}))$, whenever the data imbalance is severe
enough that
\[
\frac{\alpha_{\mathrm{raw}}}{1-\alpha_{\mathrm{raw}}}D_1
<
m_0
\left(
\frac{A-\pi_1}{2
\sqrt{
\frac{\pi_0}{\pi_1}
B_{\mathrm{LR}} B_s (1+\sqrt{B_{\mathrm{LR}}})
}}
\right)^4,
\]
there exists a constant $C>0$ such that, with probability at least $1-\delta$,
\[
\begin{aligned}
\AUPRC(s_{\hat\theta_{\mathrm{aug}}})
-
\AUPRC(s_{\hat\theta_{\mathrm{raw}}})
\ge\;&
\AUPRC(\Lambda_{\eta^*})
-
A
-
C\left(\frac{\pi_0}{\pi_1}\right)^{1/2}
\left(
    \frac{\epsilon^*}{m_\ell}
\right)^{1/6}
\\
&-
\sqrt{\frac{CL_s c}{\lambda}\frac{\pi_0}{\pi_1}}
\left(
\frac{B^2d\log(14d/\delta)}{n_0+n_1}
\right)^{1/4}
\\
&-
\sqrt{CL_s\frac{\pi_0}{\pi_1}}
\left[
\frac{L_g^2}{\mu L_{\alpha}}
\left(
\left(\alpha^*-(1-\alpha^*)\frac{n_1}{n_0}\right)
\epsilon_{\mathrm{syn}}
\right)^2
\right]^{1/4}.
\end{aligned}
\]
\end{theorem}

\begin{theorem}[Empirical Best-Threshold BA and $\F_1$ Improvement under Misspecification]
\label{thm:best-BA-improvement}
Suppose Assumption~\ref{ass:misspecified-pop} holds, except that the pairwise
margin condition is replaced by the best-threshold anti-concentration condition:
there exist constants $C,t_0>0$ such that, for every relevant parameter
$\theta$, every threshold $\tau\in\mathbb R$, every $r\in[0,t_0]$, and each
$y\in\{0,1\}$,
\[
P_y\left(|s_\theta(X)-\tau|\le r\right)
\le
Cr.
\]
Also suppose Assumption~\ref{assumption:ERM-regularity}(b)--(g) holds. Then for
any balanced-accuracy level
\[
A\in\left(\frac12,\BA^*(\Lambda_{\eta^*})\right),
\]
whenever the data imbalance is sufficiently severe that
\[
\frac{\alpha_{\rm raw}}{1-\alpha_{\rm raw}}D_1
<
m_0\left(
\frac{A-\frac12}
{\sqrt{B_s(1+\sqrt{B_{\mathrm{LR}}})}}
\right)^4,
\]
there exist constants $C,c>0$ such that, with probability at least $1-\delta$,
\[
\begin{aligned}
\BA^*(s_{\widehat\theta_{\rm aug}})
-
\BA^*(s_{\widehat\theta_{\rm raw}})
\ge\;&
\BA^*(\Lambda_{\eta^*})-A
-
\frac{1-\alpha^*+\alpha^*B_{\mathrm{LR}}}
{2\alpha^*(1-\alpha^*)^{3/2}}
\sqrt{\frac{\epsilon^*}{m_\ell}}
\\
&-
\frac{CL_s c}{\lambda}
\sqrt{
\frac{B^2d\log(14d/\delta)}{n_0+n_1}
}
-
\frac{CL_sL_g}{\sqrt{2\mu L_\alpha}}
\left(\alpha^*-(1-\alpha^*)\frac{n_1}{n_0}\right)
\epsilon_{\syn}.
\end{aligned}
\]

Similarly, let $b_\pi:=2\pi_1/(1+\pi_1)$. For any best-threshold $\F_1$ level
\[
A\in \left(b_\pi,\F_1^*(\Lambda_{\eta^*})\right),
\]
whenever the data imbalance is sufficiently severe that
\[
\frac{\alpha_{\rm raw}}{1-\alpha_{\rm raw}}D_1
<
m_0
\left(
\frac{A-b_\pi}
{4\sqrt{B_s(1+\sqrt{B_{\mathrm{LR}}})}}
\right)^4,
\]
there exist constants $C,c>0$ such that, with probability at least $1-\delta$,
\[
\begin{aligned}
\F_1^*(s_{\widehat\theta_{\rm aug}})
-
\F_1^*(s_{\widehat\theta_{\rm raw}})
\ge\;&
\F_1^*(\Lambda_{\eta^*})-A
-
\frac{2(1-\alpha^*+\alpha^* B_{\mathrm{LR}})^2}
{\alpha^*(1-\alpha^*)^{3/2}}
\sqrt{\frac{\epsilon^*}{m_{\ell}}}
\\
&-
\frac{CL_s c}{\pi_1\lambda}
\sqrt{
\frac{B^2d\log(14d/\delta)}
{n_0+n_1}
}
-
\frac{4CL_sL_g}{\pi_1\sqrt{2\mu L_\alpha}}
\left(\alpha^*-(1-\alpha^*)\frac{n_1}{n_0}\right)
\epsilon_{\rm syn}.
\end{aligned}
\]
\end{theorem}

Theorems~\ref{thm:misspecified-auc} and~\ref{thm:best-BA-improvement} show that
synthetic augmentation can yield genuine improvements in ranking and
threshold-based performance under model misspecification. The intuition is as
follows. Fix a target performance level $A$. When the imbalance is sufficiently
severe, the raw objective becomes strongly majority-favoring, so its population
minimizer cannot pay enough majority-class risk cost to produce a score with
high discrimination. Consequently, the raw population score has AUROC, AUPRC,
balanced accuracy, or $\F_1$ score bounded by $A$. Augmentation instead moves
the effective prior from $\alpha_{\rm raw}$ to a larger $\alpha$, changing the
restricted target from the sharply transformed $\eta_{\alpha_{\rm raw}}$ to a
smoother or otherwise better-approximated $\eta_\alpha$. When $\alpha$ is
chosen so that $\epsilon_\alpha$ is small, optimally
$\epsilon_\alpha=\epsilon^*$, the augmented estimator stays close to a score
preserving the likelihood-ratio ordering. The improvement lower bounds reflect
exactly this tradeoff: the oracle performance gap is reduced only by the
approximation error, finite-sample estimation error, and synthetic
distributional error. As the sample size grows and the synthetic distribution
becomes accurate, the latter two vanish; if the approximation term is also
small enough, the bound remains positive, guaranteeing improvement over the raw
imbalanced estimator.

\begin{remark}[Model Expressiveness Explains Mixed AUROC Gains from Synthetic Augmentation]
The literature reviewed in the introduction reports both positive and negative
findings on whether synthetic augmentation improves AUROC. For example,
\cite{van2022harm} found no AUROC improvement for logistic regression. This is
consistent with the regime in which the score model is too restrictive to
represent an improved ranking: logistic regression has limited flexibility in
complex problems, so both $\epsilon_{\rm raw}$ and $\epsilon^*$ can be large,
and the improvement lower bound need not be positive. By contrast, studies
reporting gains typically use more expressive models such as neural networks or
boosting, matching the regime in which the model class is rich enough to exploit
the rebalanced objective while augmentation corrects the ranking errors induced
by the raw prior. Our theory thus explains these mixed findings: AUROC gains
from synthetic augmentation depend not only on the quality of the generated
data, but also on whether the model class can exploit the information
introduced by augmentation.
\end{remark}

\section{Simulation Studies}\label{sec:simulation}

We conduct simulation studies that illustrate the theoretical findings of Sections~\ref{sec:synthetic-augmentation} and~\ref{sec:misspecified}. The simulations are organized into three regimes: a well-specified setting in which the fitted score family matches the data-generating mechanism exactly (Section~\ref{sec:mle}), a restricted-model setting in which logistic regression is fitted to data-generating distributions with varying degrees of compatibility with a linear score, and a flexible-model setting in which a multilayer perceptron (MLP) is fitted to two of those same distributions. In each study, all reported metrics are evaluated on an independent test set of $n_{\rm test}$ observations drawn from the original population distribution $(P_0,P_1)$, and synthetic samples never enter the test set and are used only to modify the training objective. For a training set with $n_0$ majority and $n_1$ minority observations, we define the synthetic proportion
\[
    q \;=\; \frac{n_{\rm syn}}{n_0-n_1} \;\in\; [0,1],
\]
the fraction of the majority-minority size gap closed by adding $n_{\rm syn}$ synthetic minority samples to training. The grid of $q$ values is study-specific and is stated within each subsection, rather than fixed across all three studies. All results are averaged over $100$ independent replications. Unless noted otherwise, the lines represent the across-replication mean, and shaded bands show the mean $\pm$ one standard error. We compare raw training (no augmentation) with up to three augmentation strategies: an oracle benchmark that draws additional minority samples from the true $P_1$ serving as an ideal reference for distribution-matched augmentation, bootstrap resampling of the observed minority training data, and SMOTE \citep{chawla2002smote} which creates synthetic samples by linear interpolation between neighboring minority observations.

\noindent\textbf{Threshold tuning protocol.} When validation-tuned decision-level metrics are reported (Sections~\ref{sec:sim-well}, \ref{sec:sim-limited}, and~\ref{sec:sim-improve}), the training set is split into a fitting set and a validation set; oracle, bootstrap, and SMOTE samples are generated only from the minority observations in the fitting set, and each classifier is trained on the fitting set or its augmented version. AUROC and AUPRC are computed directly from test scores without threshold tuning, since they depend only on the score ordering. For balanced accuracy and $\F_1$ score, we use two separate thresholds maximizing the metrics on the validation set to estimate the best-threshold performance, and the metrics are evaluated at these thresholds on the independent test set. Section~\ref{sec:sim-limited} additionally reports balanced accuracy and $\F_1$ score at the fixed threshold $0.5$ alongside their validation-tuned threshold counterparts, so that the fixed-threshold values there serve as a diagnostic for threshold tuning rather than as the best-threshold metric studied in Section~\ref{sec:misspecified}.

\subsection{Well-Specified Case}\label{sec:sim-well}

We first consider the well-specified setting of Section~\ref{sec:synthetic-augmentation}, in which the fitted score family closely matches the data-generating mechanism. The predictor $X \in \mathbb{R}^{p}$ has $p = 10$ coordinates.
The first five coordinates of $X$ are Gaussian with a nonzero mean vector $0.5 \mathbf{1}_5$ and an AR(1) covariance structure, $\Sigma_{jk} = \rho^{|j-k|}$ with $\rho = 0.5$ (the same AR(1) coefficient used throughout this section), and the last five are multivariate Student-$t$ with $\nu = 5$ degrees of freedom, a nonzero location vector $0.5 \mathbf{1}_5$, and a comparable AR(1) scale matrix of the same form. The minority class prior is $\pi_1 = 0.05$, with $n_{\rm train} = 10{,}000$ training samples and $n_{\rm test} = 5{,}000$ test samples. We consider three well-specified models, each fitted with the estimator matching its generative mechanism:
\begin{itemize}
    \item Logistic: $X$ is drawn from the mixed marginal distribution above, and $Y \mid X = x$ is then generated as Bernoulli with success probability $s_{\eta^*}(x) = (1+\exp(-\eta^{*\top}x))^{-1}$. The resulting data are fitted by logistic regression.
    \item Probit: $X$ is drawn from the same marginal distribution, and $Y \mid X = x$ is then generated as Bernoulli with success probability $\Phi(\eta^{*\top}x)$, where $\Phi$ is the cumulative distribution function of a standard normal distribution. The resulting data are fitted by probit regression.
    \item Generative likelihood-ratio MLE (Section \ref{sec:mle}): a class-conditional likelihood-ratio MLE benchmark, distinct from the conditional MLE underlying the logistic and probit rows above. Here $Y$ is generated first with prior $\pi_1 = 0.05$, and $X \mid Y$ is then generated from class-conditional Gaussian and multivariate Student-$t$ components. For $Y=0$, both blocks have mean location vector $0.5 \mathbf{1}_5$. For $Y=1$, the Gaussian block has mean vector $0.85 \mathbf{1}_5$ and the Student-$t$ block has location vector $0.75 \mathbf{1}_5$. The classifier is the likelihood-ratio statistic obtained from class-conditional densities estimated by maximum likelihood, rather than a fitted linear index.
\end{itemize}
For the logistic and probit rows, the intercept component of $\eta^*$ is calibrated by Monte Carlo root-finding so that the marginal minority prevalence $\mathbb P(Y=1) \approx \pi_1 = 0.05$, matching the imbalance level used in the other studies.
We compare raw training with bootstrap and SMOTE augmentation as synthetic proportion $q$ varies between $0$ and $1$, reporting AUROC, AUPRC, and validation-tuned best-threshold balanced accuracy and $\F_1$ score (threshold tuning protocol above).

\begin{figure}[htbp]
    \centering
    \includegraphics[width=\linewidth]{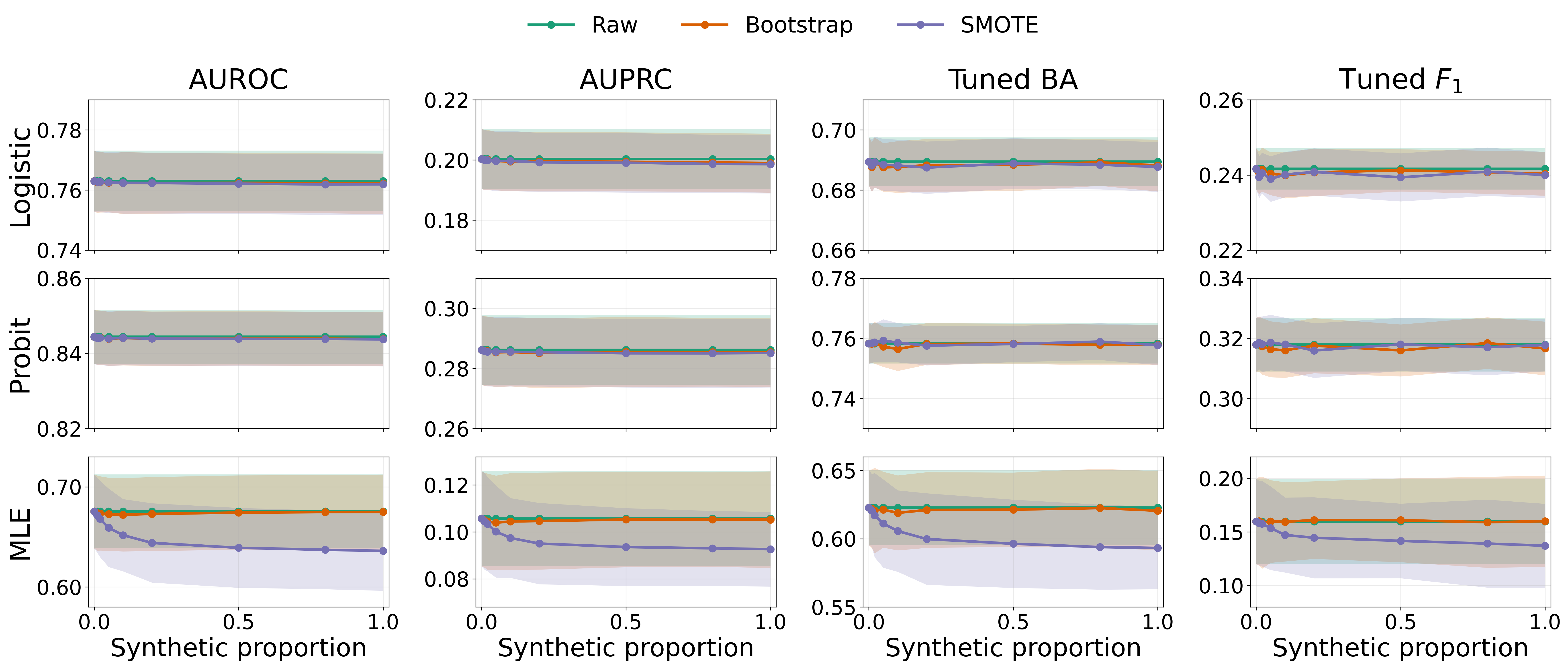}
    \caption{AUROC, AUPRC, validation-tuned best-threshold balanced accuracy and $\F_1$ score (columns), as a function of the synthetic proportion, for three well-specified models (rows): logistic, probit, and generative likelihood-ratio MLE.}
    \label{fig_well_specified_3by4}
\end{figure}

Figure~\ref{fig_well_specified_3by4} shows that, under this well-specified model setting, augmentation provides no meaningful benefit and may even be harmful in some cases. In the logistic and probit results, raw, bootstrap, and SMOTE are visually indistinguishable across the full range of synthetic proportions and across all four metrics, consistent with the prediction that augmentation does not alter the population-optimal ranking when the score family is well-specified. For the generative likelihood-ratio MLE, bootstrap remains comparable to raw, but SMOTE noticeably underperforms both raw training and bootstrap as the synthetic proportion increases, lowering AUROC, AUPRC, and both tuned metrics. Because SMOTE constructs synthetic points by linear interpolation rather than by sampling from the true class-conditional generative distributions, it introduces a synthetic distributional mismatch $P_{\rm syn} \neq P_1$ that adds bias without any compensating gain in ranking quality, exactly the failure mode anticipated in Remark \ref{remark-asymp-limitdata}. These results are consistent with the theoretical prediction that, under well-specified models, synthetic augmentation provides no systematic population-level improvement in score ordering, except for possible finite-sample variance effects and possible degradation from synthetic distributional mismatch. This degradation is better interpreted as evidence of synthetic distributional mismatch rather than as a consequence of model-class misspecification.

\subsection{Restricted Logistic Regression across Linear and Nonlinear Likelihood Ratios}\label{sec:sim-limited}

We next examine logistic regression as a restricted model class with linear score functions, fitted to four data-generating distributions chosen to span a range of true likelihood-ratio shapes. With shared covariance across classes, Gaussian and Gaussian AR(1) class-conditional distributions have a likelihood-ratio statistic that is exactly linear in $x$. Logistic regression is therefore effectively well specified for these two settings, and we include them as linear-likelihood-ratio controls. The $t_5$ setting has a true likelihood ratio that depends nonlinearly on $x$ through the quadratic form in the multivariate Student-$t$ density, so logistic regression is only mildly misspecified there. The fourth distribution, Mixture, serves as the main restricted-model misspecification example in this subsection. Its minority class is a mixture of a multivariate $t_5$ component and a Gaussian AR(1) component with distinct means and covariances, so its true likelihood ratio is genuinely nonlinear and not well approximated by any linear score.

For a misspecified model with a restricted score class, such as linear logistic regression, adding synthetic data does not alter the inherent linearity of the model class. Consequently, the approximation errors $\epsilon^*$ and $\epsilon_{\raw}$ defined in Section~\ref{sec:misspecified} may both be large. In this regime, the lower bounds in Section~\ref{sec:misspecified} need not be positive, and therefore no improvement guarantee can be established.

We use the same general configuration as above, with $p = 10$ predictors,
$n_{\rm train} = 10{,}000$, $n_{\rm test} = 5{,}000$, and $\pi_1 = 0.05$. The synthetic proportion is varied between $0$ and $1$. The four data-generating distributions are: Gaussian, with $P_0 = \mathcal{N}(\mathbf{0}, I_p)$ and $P_1 = \mathcal{N}(\mu_1 \mathbf{1}, I_p)$, $\mu_1 = 0.5$; $t_5$, with $P_0 = t_5(\mathbf{0}, I_p)$ and $P_1 = t_5(\mu_1 \mathbf{1}, I_p)$; Gaussian AR(1), with $P_0 = \mathcal{N}(\mathbf{0}, \Sigma)$ and $P_1 = \mathcal{N}(\mu_1 \mathbf{1}, \Sigma)$, where $\Sigma_{jk} = \rho^{|j-k|}$ with $\rho = 0.5$; and a mixture minority distribution,
\[
    P_1 = \pi_A \, P_A + (1-\pi_A)\, P_B, \qquad \pi_A = 0.5,
\]
where $P_A$ is multivariate $t_5$ with mean $\mu_A = (0.8,\ 0.6,\ -0.7,\ 0.5,\ 0.4,\ -0.5,\ 0.3,\ 0.4,\ -0.3,\ 0.2)^\top$ and diagonal scale $\Sigma_A = \mathrm{diag}(1.4,\ 1.2,\ 1.5,\ 1.0,\ 1.3,\ 1.1,\ 0.9,\ 1.2,\ 1.0,\ 0.8)$, and $P_B$ is Gaussian with mean $\mu_B = (-0.6,\ 0.7,\ 0.5,\ -0.8,\ 0.6,\ 0.4,\ -0.5,\ 0.3,\ 0.5,\ -0.4)^\top$ and AR(1)-type covariance $(\Sigma_B)_{jk} = 1.3 \times 0.6^{|j-k|}$. We compare raw training with oracle, bootstrap, and SMOTE augmentation.

\begin{figure}
    \centering
    \includegraphics[width=\linewidth]{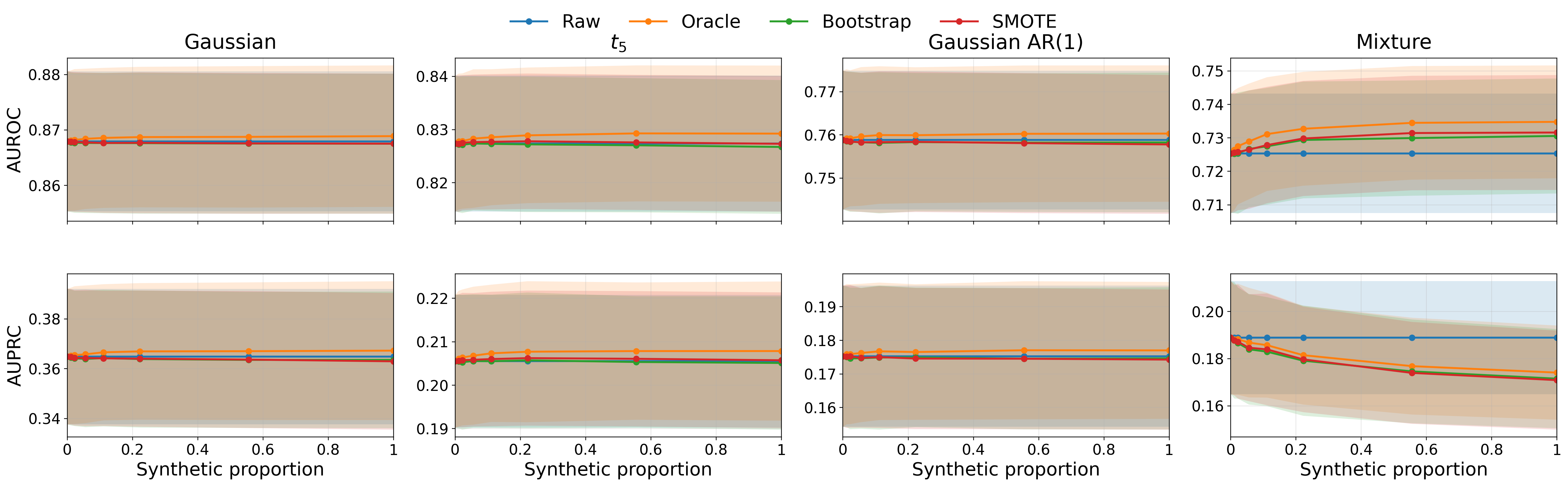}
    \caption{AUROC and AUPRC as functions of the synthetic proportion, for logistic regression applied to four data-generating distributions: Gaussian, $t_5$, Gaussian AR(1), and Mixture.}
    \label{fig:four_dist_auc}
\end{figure}

Figure~\ref{fig:four_dist_auc} confirms this distinction visually. For the two linear-likelihood-ratio controls, the AUROC and AUPRC curves for raw, oracle, bootstrap, and SMOTE are nearly indistinguishable across the full range of synthetic proportions, well within the shaded variability bands, and the mildly nonlinear $t_5$ case shows the same pattern. This is consistent with logistic regression already attaining close to the best ordering it can represent in these three settings, leaving augmentation essentially no room to help or hurt. However, in the Mixture setting, as the synthetic proportion grows, augmentation visibly improves AUROC while degrading AUPRC. This metric-specific, non-uniform pattern is consistent with the general tradeoff developed in Section~\ref{sec:misspecified}: when the model class is overly restricted, both the raw and augmented target score functions may be difficult to approximate, and no general improvement guarantee can be established.

Figure~\ref{fig_four_dist_ba_f1_tuned} complements the ranking-metric results by comparing validation-tuned best-threshold balanced accuracy and $\F_1$ score. For the validation-tuned best-threshold metrics, adding synthetic data provides little to no improvement across the four data-generating distributions. Performance at the fixed $0.5$ threshold is also reported for comparison. For the fixed threshold, synthetic augmentation can substantially improve performance because it effectively rebalances the training data relative to that fixed decision threshold. In addition, validation-threshold tuning of the raw score can achieve performance comparable to the best fixed-threshold synthetic-augmented classifier obtained by choosing an appropriate synthetic proportion. This comparison suggests that the large fixed-threshold gains from synthetic augmentation mainly reflect threshold recalibration rather than genuine improvements in score ordering.

\begin{figure}[htbp]
    \centering
    \includegraphics[width=\linewidth]{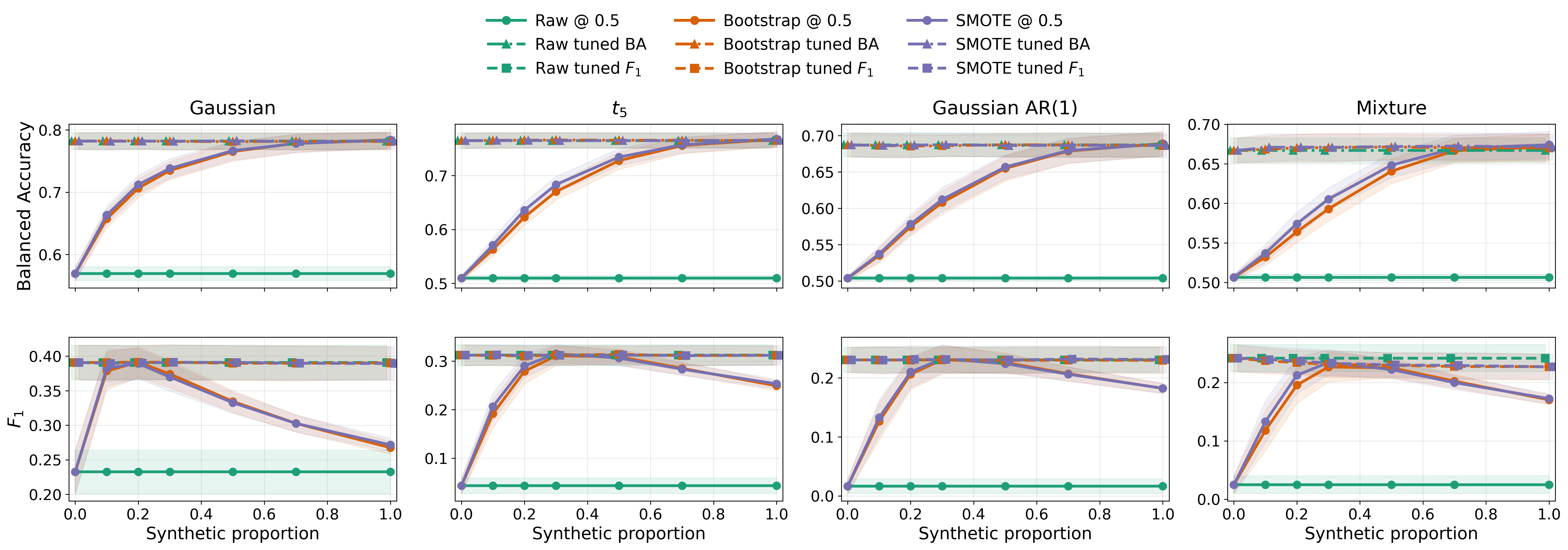}
    \caption{Balanced accuracy (top) and $\F_1$ score (bottom) at the fixed $0.5$ threshold and at validation-tuned thresholds, as a function of the synthetic proportion, for logistic regression applied to four data-generating distributions: Gaussian, $t_5$, Gaussian AR(1), and Mixture.}
    \label{fig_four_dist_ba_f1_tuned}
\end{figure}

\subsection{Misspecified Case with Complex Classification Model Class}\label{sec:sim-improve}

Finally, we consider a flexible model class, a multilayer perceptron (MLP) with two hidden layers of widths $(64, 32)$, ReLU activations, Adam optimization, early stopping on a validation split, and standardized inputs, applied to the $t_5$ distribution and Gaussian AR(1) distribution with $\rho = 0.5$, as defined above. Although the Gaussian AR(1) distribution has a linear likelihood-ratio structure, the finite-sample MLP trained under severe imbalance can still behave as an effectively restricted learner. Within this flexible MLP class, the raw imbalanced training objective can select a poor score ordering, even though the population MLP class could represent a better one. Augmentation changes the effective class balance seen during training and can correct this objective-induced ranking mismatch. We compare raw training with oracle, bootstrap, and SMOTE augmentation, reporting AUROC and AUPRC together with decision-level metrics evaluated both at a validation-tuned best-threshold and a fixed threshold of $0.5$. The synthetic proportion is varied between $0$ and $1$.

\begin{figure}[htbp]
    \centering
    \includegraphics[width=0.8\linewidth]{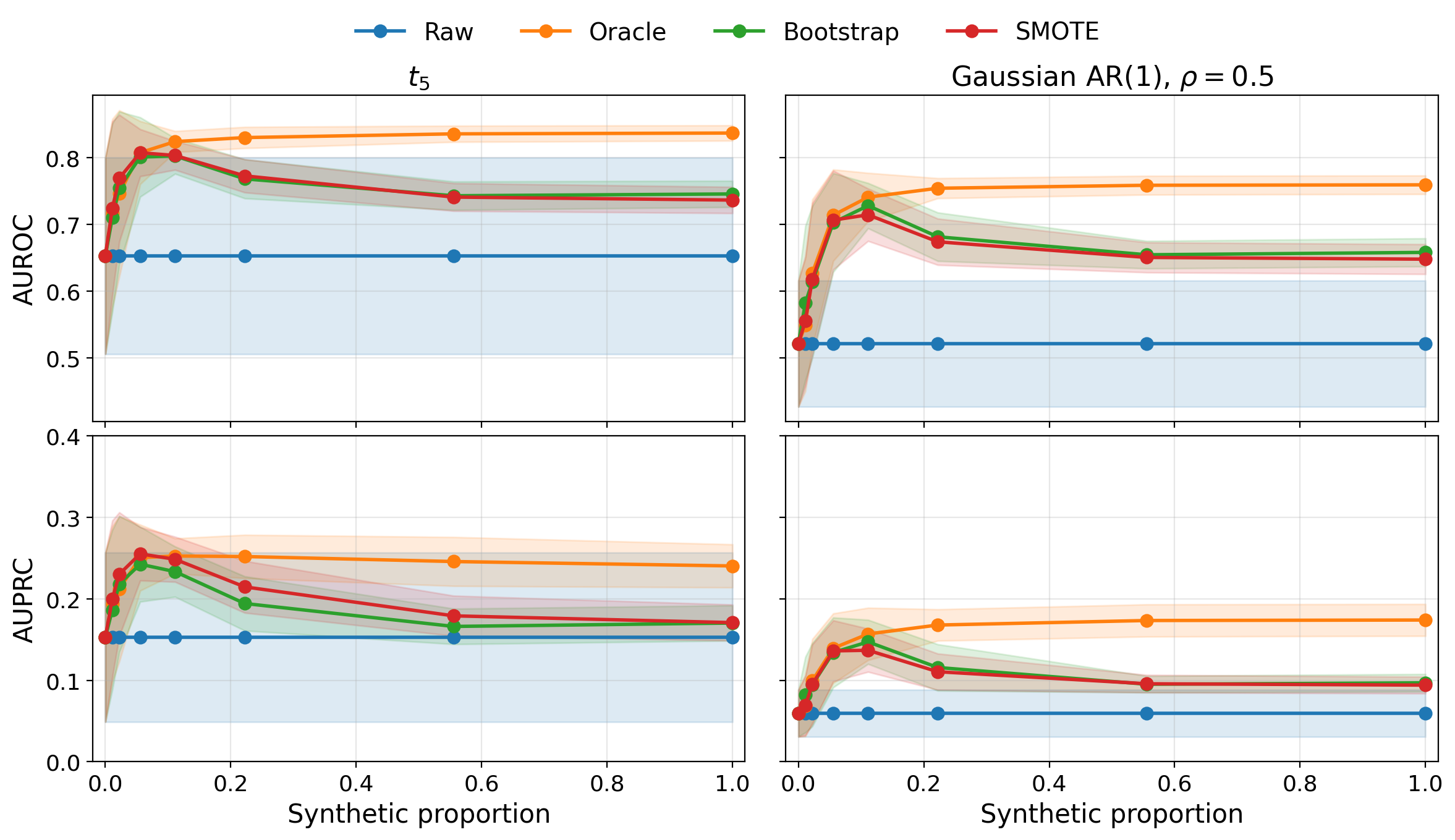}
    \caption{AUROC (top) and AUPRC (bottom) as a function of the synthetic proportion, for an MLP applied to the $t_5$ distribution (left) and Gaussian AR(1) distribution with $\rho = 0.5$ (right).}
    \label{fig_misspecified_improvement_auc}
\end{figure}

Figure~\ref{fig_misspecified_improvement_auc} addresses the score-ordering metrics AUROC and AUPRC. Oracle augmentation yields substantial and sustained improvements over raw training in AUROC and AUPRC for both distributions. Bootstrap and SMOTE also improve over raw, but nonmonotonically: their gains peak at a moderate synthetic proportion and decline as the proportion increases further. The decline is more pronounced for AUPRC, where the moderate-proportion gain largely fades at the largest synthetic proportions, leaving performance close to the raw baseline. This pattern is consistent with Theorem~\ref{thm:misspecified-auc}: moderate, high-quality augmentation corrects the objective-induced ranking mismatch, but as the proportion of imperfect synthetic data grows, the synthetic distributional error $\epsilon_{\rm syn}$ increasingly contaminates training and erodes the gain, whereas the oracle benchmark sampling from the true $P_1$ avoids this bias and maintains its advantage throughout.

\begin{figure}[htbp]
    \centering
    \includegraphics[width=0.8\linewidth]{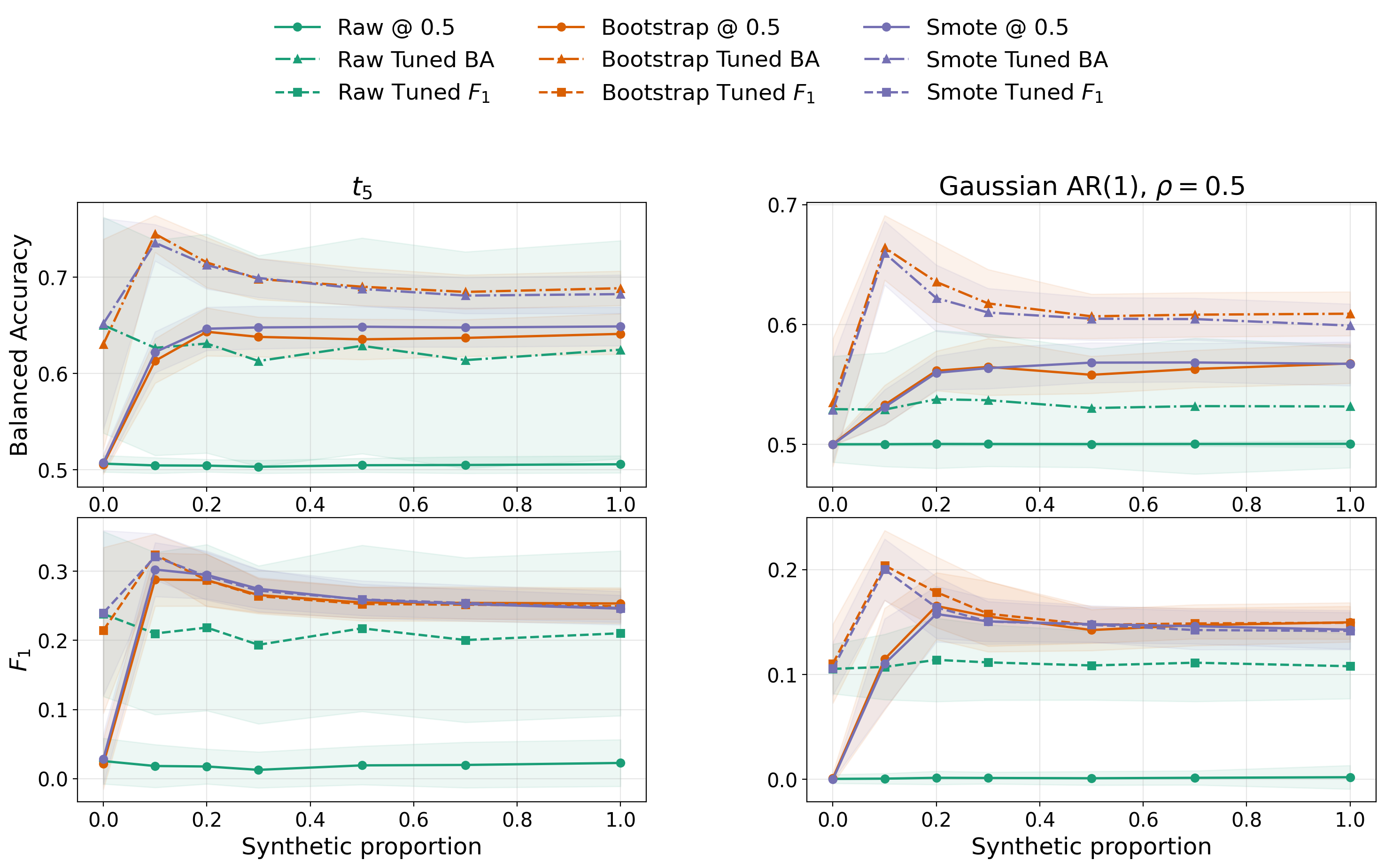}
    \caption{Balanced accuracy (top) and $\F_1$ score (bottom) at a fixed $0.5$ threshold and at a validation-tuned threshold, as a function of the synthetic proportion, for an MLP applied to the $t_5$ distribution (left) and the Gaussian AR(1) distribution with $\rho = 0.5$ (right). }
    \label{fig_misspecified_improvement_ba_f1}
\end{figure}

Figure~\ref{fig_misspecified_improvement_ba_f1} addresses the validation-tuned best-threshold metrics. It shows that augmentation can improve balanced accuracy and $\F_1$ score beyond what threshold tuning of the raw score alone achieves. At the fixed $0.5$ threshold, raw training performs poorly. Tuning the threshold of the raw score substantially improves its balanced accuracy and $\F_1$ score, and a visible gap to the augmented methods remains. Bootstrap and SMOTE generally outperform the tuned raw score over a broad range of synthetic proportions, with the clearest gains at moderate augmentation levels. These gains narrow somewhat at larger synthetic proportions but stay positive throughout, again reflecting the nonmonotone tradeoff between correcting the ranking mismatch at moderate augmentation and accumulating synthetic distributional error at high augmentation. This is consistent with Theorem~\ref{thm:best-BA-improvement}: when the effective model class is rich enough to represent an improved ordering and the imbalance is severe, augmentation corrects the objective-induced ranking mismatch itself, so that even after validation-based threshold tuning, the augmented score can yield better decision-level performance than the raw score, an improvement that threshold tuning of the raw score alone cannot fully replicate.

Together, the three studies trace out the regimes identified by the theory. Under well-specification, augmentation gives little or no systematic benefit and can hurt performance when $P_{\rm syn} \neq P_1$ (Section~\ref{sec:synthetic-augmentation}). Under a restricted linear-score model, augmentation has negligible effects on score ordering when the model is effectively well-specified, and has metric-specific, non-uniform effects when it is genuinely misspecified, while the large apparent gains at a fixed decision threshold are mainly threshold recalibration rather than improved ordering. Under a flexible learner that behaves as effectively restricted under severe imbalance, augmentation can improve both score ordering and validation-tuned decision metrics beyond what threshold tuning of the raw score alone achieves, with nonmonotone gains that are clearest at moderate augmentation levels and diminish as synthetic distributional error accumulates at larger synthetic proportions (Section~\ref{sec:misspecified}).

\section{Discussion}\label{sec:discussion}

This paper studies when synthetic minority augmentation can improve score-based imbalanced classification. Our results show that the effect of augmentation depends strongly on the relationship between the learned score, the likelihood-ratio ordering, and the synthetic minority distribution. Under well-specified score models, the raw estimator already targets the population-optimal likelihood-ratio ordering, so augmentation cannot provide a fundamental improvement for the score-based metrics considered here beyond possible finite-sample variance reduction. Under misspecification, however, high-quality synthetic data augmentation can play a different role: by changing the effective training objective, it can alter the learned score ordering and correct ranking errors that threshold tuning alone cannot repair.

Several limitations remain. First, our theory focuses on binary classification and score-based evaluation metrics such as AUROC, AUPRC, best-threshold balanced accuracy, and best-threshold \(\F_1\) score. Extending the framework to multiclass, multilabel, or structured-output classification would require new notions of oracle ordering and ranking mismatch. 

Second, our analysis treats the synthetic distribution \(P_{\rm syn}\) through distributional discrepancy measures such as Wasserstein distance. This abstraction is useful for general theory, but it does not fully capture the algorithm-specific behavior of modern generators such as diffusion models, GANs, or large language models. A sharper theory would connect generator training, minority sample size, and augmentation quality more directly.

Finally, the improvement guarantees under misspecification are sufficient conditions rather than necessary ones. The ranking-mismatch functionals identify regimes where augmentation can provably improve over the raw score or threshold tuning, but they may be difficult to estimate accurately in practice because the likelihood-ratio score is unknown. Developing data-driven diagnostics for detecting correctable ranking mismatch would make the theory more practically actionable. 

\bibliography{bib}
\bibliographystyle{abbrvnat}

\clearpage
%\appendix
\phantomsection
\pdfbookmark[1]{Supplementary Material}{supplement-title}

\begin{center}
{\Large\bfseries Supplementary Materials for ``When Does Synthetic Data Augmentation Improve Score-Based Imbalanced Classification?"}

\vspace{1em}

{\normalsize
Zhengchi Ma$^{1}$, Pengfei Lyu$^{2}$, and Anru R. Zhang$^{2,3}$\par}

\vspace{0.5em}

{\small
$^{1}$Department of Electrical \& Computer Engineering, Duke University\par
$^{2}$Department of Biostatistics \& Bioinformatics, Duke University\par
$^{3}$Department of Computer Science, Duke University\par}
\end{center}
\vspace{2em}

% switch to supplement style
\setcounter{section}{0}
\renewcommand{\thesection}{S\arabic{section}}
\setcounter{equation}{0}
\renewcommand{\theequation}{S\arabic{equation}}
\setcounter{theorem}{0}
\renewcommand{\thetheorem}{S\arabic{theorem}}
\setcounter{lemma}{0}
\renewcommand{\thelemma}{S\arabic{lemma}}
\setcounter{corollary}{0}
\renewcommand{\thecorollary}{S\arabic{corollary}}
\setcounter{proposition}{0}
\renewcommand{\theproposition}{S\arabic{proposition}}
\setcounter{definition}{0}
\renewcommand{\thedefinition}{S\arabic{definition}}
\setcounter{figure}{0}
\renewcommand{\thefigure}{S\arabic{figure}}

In the supplementary material, we provide the detailed proofs for the theoretical results in the paper.

\section{Proofs of Theorems}
\begin{proof}[Proof of Theorem \ref{thm:AUC-invariant}]
We prove the statements for AUROC and AUPRC separately.

\medskip
\noindent \textbf{AUROC:}

    First, we show that AUROC is invariant under strictly increasing transformations. If $h:\mathbb{R}\to\mathbb{R}$ is strictly increasing, then for any score $s$
    \begin{align*}
        \left\{s(x_1)>s(x_0)\right\} \Leftrightarrow\left\{h(s(x_1))>h(s(x_0))\right\},\quad \left\{s(x_1)=s(x_0)\right\}\Leftrightarrow \left\{h(s(x_1))=h(s(x_0))\right\}.
    \end{align*}
    Thus the events in the AUROC definition are identical under $s$ and $h\circ s$, and thus 
    \begin{equation*}
    \AUC(h\circ s)=\AUC(s).
    \end{equation*}
    That is, equivalently
    \begin{equation*}
    \AUC(g(\Lambda_{\eta^*}(X)))=\AUC(\Lambda_{\eta^*}(X)).
    \end{equation*}
    For ease of notation, let $\Delta_s:=s(X_1)-s(X_0)$. Then the AUROC can be written as
    \begin{align*}
        \AUC(s)=\mathbb{E}\left[1\{\Delta_s>0\}+\frac{1}{2}1\{\Delta_s=0\}\right].
    \end{align*}
    Then for two scores $s$ and $s'$, we have
    \begin{align}\label{eq:AUC-to-sign-diff}
        \nonumber\left|\AUC(s)-\AUC(s')\right|&\le \left|\mathbb{E}[1\{\Delta_s>0\}-1\{\Delta_{s'}>0\}]\right|+\frac12 \left|\mathbb{E}\left[1\{\Delta_s=0\}-1\{\Delta_{s'}=0\}\right]\right|\\
        \nonumber&\leq \mathbb{E}\left|1\{\Delta_s>0\}-1\{\Delta_{s'}>0\}\right|+\frac12 \mathbb{E}\left|1\{\Delta_s=0\}-1\{\Delta_{s'}=0\}\right|\\
         &\leq \mathbb P\left(\mathrm{Sign}(\Delta_s)\neq \mathrm{Sign}(\Delta_{s'})\right).
    \end{align}
    For the event, it holds that 
    \begin{align*}
        \left\{\mathrm{Sign}(\Delta_s)\neq \mathrm{Sign}(\Delta_{s'})\right\}\quad \Rightarrow \quad \left\{|\Delta_s|\leq\left|\Delta_s-\Delta_{s'}\right|\right\}.
    \end{align*}
    This holds because the magnitude of one of the two quantities must be no larger than their difference; otherwise, no sign flip can occur. Also,
    \begin{align*}
        \left|\Delta_s-\Delta_{s'}\right|&=\left|s(X_1)-s(X_0)-(s'(X_1)-s'(X_0))\right|\leq \left|s(X_1)-s'(X_1)\right|+\left|s(X_0)-s'(X_0)\right|.
    \end{align*}
    Thus we have the event
    \begin{align*}
        \left\{|\Delta_s|\leq\left|\Delta_s-\Delta_{s'}\right|\right\}\quad \Rightarrow \quad \left\{|\Delta_s|\leq\left|s(X_1)-s'(X_1)\right|+\left|s(X_0)-s'(X_0)\right|\right\}.
    \end{align*}
    Therefore by assumption, also condition on the training data such that $\delta_n$ is fixed, it holds that
    \begin{align}\label{eq:sign-diff-to-deltan}
        \nonumber\mathbb P\left(\mathrm{Sign}(\Delta_s)\neq \mathrm{Sign}(\Delta_{s'})\right)\leq \mathbb P\left(|\Delta_s|\leq\left|s(X_1)-s'(X_1)\right|+\left|s(X_0)-s'(X_0)\right|\right)\\
        \leq \mathbb P\left(|\Delta_s|\leq 2\delta_n\right)
        \leq 2C\delta_n,
    \end{align}
    where the last step holds by assumption when $s$ is taken as $g(\Lambda_{\eta^*}(\cdot))$ and $s'$ as $\hat{s}_n$. Here $$\delta_n:=\|\hat{s}-g\circ \Lambda_{\eta^*}\|_{L_\infty (P_0+P_1)}=O_P(r_n)=o_P(1).$$ Therefore $2\delta_n\leq t_0$ with probability tending to one.
    That is 
    \begin{align*}
        \left|\AUC(s)-\AUC(s')\right|\leq \mathbb P\left(\mathrm{Sign}(\Delta_s)\neq \mathrm{Sign}(\Delta_{s'})\right)\leq 2C\delta_n=O_P(r_n).
    \end{align*}
    Thus, by optimality and strictly increasing transformation, we finish the proof.

\medskip
\noindent \textbf{AUPRC:}

    We first show that AUPRC is invariant to strictly increasing transformations. For any strictly increasing function $h:\mathbb{R}\to\mathbb{R}$, for any threshold $u$, we have
    \begin{align*}
        \{h(s(x))\geq u\}\Leftrightarrow\{s(x)\geq h^{-1}(u)\}.
    \end{align*}
    Therefore the recall
    \begin{align*}
        \Rec_{h\circ s}(u)=\mathbb P(h(s(X_1))\geq u)=\mathbb P(s(X_1)\geq h^{-1}(u))=\Rec_s(h^{-1}(u)).
    \end{align*}
    Similarly for false positive rate and precision, we have
    \begin{align*}
        \mathrm{FPR}_{h\circ s}(u)=\mathrm{FPR}_s(h^{-1}(u)),\qquad \Prec_{h\circ s}(u)=\Prec_s(h^{-1}(u)).
    \end{align*}
    Because $h$ is strictly increasing, the map $u\to t:=h^{-1}(u)$ is thus a bijection. Therefore,
    \begin{align*}
        \AUPRC(h\circ s)=\int \Prec_{h\circ s}(u)d\Rec_{h\circ s}(u)=\int \Prec_{s}(t)d\Rec_{s}(t)=\AUPRC(s).
    \end{align*}
    Thus we have
    \begin{equation*}
    \AUPRC(g(\Lambda_{\eta^*}))=\AUPRC(\Lambda_{\eta^*}).
    \end{equation*}
    Next, we connect the AUPRC difference with score estimation. Define the positive survival transform,
    \[
    U_s(x):=\mathbb P(s(X_1)\geq s(x)\mid x)\in [0,1],
    \]
    and similarly we have $U_{s'}(x)$. Thus for $X_1'$, an independent copy of $X_1$, we have
    \begin{align*}
        U_s(X_1)=\mathbb P(s(X'_1)\geq s(X_1)\mid s(X_1))=1-F_1(s(X_1))=:\bar{F}_1(s(X_1)),
    \end{align*}
    where $F_1$ is the CDF of $s(X_1)$. We can then define the generalized inverse function 
    \begin{align*}
        \bar{F}_1^{-1}(u):=\inf \{t\in \mathbb{R}:\bar{F}_1(t)\leq u\}.
    \end{align*}
    By definition $\bar{F}_1$ and $\bar{F}_1^{-1}$ are non-increasing. 
    Thus
    \begin{align*}
        \mathbb P(\bar{F}_1(s(X_1))\leq u)=\mathbb P(s(X_1)\geq \bar{F}_1^{-1}(u))=\bar{F}_1(\bar{F}_1^{-1}(u))=u,
    \end{align*}
    where the last step is because of continuity. That is, we have
    \begin{align}\label{eq:U-uniform-Fbar}
        U_s(X_1)=\bar{F}_1(s(X_1))\sim\mathrm{Unif}[0,1].
    \end{align}
    For the transformed score, $U_s(x)$, we threshold at $r\in[0,1]$ using the rule: predict positive if and only if $U_s(x)\leq r$, which is equivalent to thresholding the score $s(x)$ at $\bar{F}_1^{-1}(r)$. Then we have the metrics
    \begin{align*}
        \widetilde{\Rec}_s(r):=\Rec_s(\bar{F}_1^{-1}(r))=\mathbb P(U_s(X_1)\leq r)=r,
    \end{align*}
     \begin{align*}
        \widetilde{\mathrm{FPR}}_s(r):=\FPR_s(\bar{F}_1^{-1}(r))=\mathbb P(U_s(X_0)\leq r),
    \end{align*}
    Thus the precision at each recall $r$ is
    \begin{align*}
        \widetilde{\Prec}_s(r):=\Prec_s(\bar{F}_1^{-1}(r))=\frac{\pi_1 r}{\pi_1 r+\pi_0\widetilde{\mathrm{FPR}}_s(r)}.
    \end{align*}
    Thus, we have expressed precision as a function of recall through the function $\bar{F}_1^{-1}$. Taking integral over recall, we have
    \begin{align}\label{eq:AUPRC-FPR-integral}
        \AUPRC(s)=\int_0^1 \frac{\pi_1 r}{\pi_1 r+\pi_0\widetilde{\mathrm{FPR}}_s(r)} dr.
    \end{align}
    The same arguments also apply to $s'$. Then for simplicity, for $r\in[0,1]$, define the function 
    \begin{align*}
        \phi_r(z):=\frac{\pi_1r}{\pi_1 r+\pi_0z},\quad z\in[0,1].
    \end{align*}
    Therefore,
    \begin{align*}
        |\AUPRC(s)-\AUPRC(s')|\leq \int_0^1 \left|\phi_r(\widetilde{\mathrm{FPR}}_s(r))-\phi_r(\widetilde{\mathrm{FPR}}_{s'}(r))\right|dr.
    \end{align*}
    For any $z$, and $z'$, the $\phi_r(\cdot)$ function satisfies
    \begin{align*}
        \left|\phi_r(z)-\phi_r(z')\right|=\left|\frac{\pi_1r}{\pi_1 r+\pi_0z}-\frac{\pi_1r}{\pi_1 r+\pi_0z'}\right|
        \leq \frac{\pi_0}{\pi_1}\frac{|z-z'|}{r}.
    \end{align*}
    Combining this with the trivial bound, we have
    \begin{align}\label{eq:phi-diff-upper-z}
         \left|\phi_r(z)-\phi_r(z')\right|\leq \min\left\{1,~\frac{\pi_0}{\pi_1}\frac{|z-z'|}{r}\right\}.
    \end{align}
    That is
    \begin{align*}
        |\AUPRC(s)-\AUPRC(s')|\leq \int_0^1 \min\left\{1,~\frac{\pi_0}{\pi_1}\frac{|\widetilde{\mathrm{FPR}}_s(r)-\widetilde{\mathrm{FPR}}_{s'}(r)|}{r}\right\}dr.
    \end{align*}
    Still for simplicity, denote $\Delta(r):=|\widetilde{\mathrm{FPR}}_s(r)-\widetilde{\mathrm{FPR}}_{s'}(r)|$. Then for any $\tau\in (0,1]$, we have the integral
    \begin{align*}
        \int_0^1 \min\left\{1,~\frac{\pi_0}{\pi_1}\frac{\Delta(r)}{r}\right\}dr\leq \int_0^\tau1dr+\int_\tau^1 \frac{\pi_0}{\pi_1}\frac{\Delta(r)}{r}dr\leq \tau +\frac{\pi_0}{\pi_1}\frac{1}{\tau}\int_0^1\Delta(r)dr.
    \end{align*}
    Then, when $1< \sqrt{\frac{\pi_0}{\pi_1}\int_0^1\Delta(r)dr}$, choose $\tau=1$, and we have
    \[
    \int_0^1 \min\left\{1,~\frac{\pi_0}{\pi_1}\frac{\Delta(r)}{r}\right\}dr\leq1.
    \]
    When $\sqrt{\frac{\pi_0}{\pi_1}\int_0^1\Delta(r)dr}\leq1$, choose $\tau=\sqrt{\frac{\pi_0}{\pi_1}\int_0^1\Delta(r)dr}$, and we have
    \[
    \int_0^1 \min\left\{1,~\frac{\pi_0}{\pi_1}\frac{\Delta(r)}{r}\right\}dr\leq 2\sqrt{\frac{\pi_0}{\pi_1}\int_0^1\Delta(r)dr}.
    \]
    Thus, combining above we have
    \begin{align}\label{eq:9}
        |\AUPRC(s)-\AUPRC(s')|\leq\min\left\{1,~2\sqrt{\frac{\pi_0}{\pi_1}\int_0^1\Delta(r)dr}\right\}.
    \end{align}
    Next, we only need to bound $\int_0^1\Delta(r)dr=\int_0^1|\widetilde{\mathrm{FPR}}_s(r)-\widetilde{\mathrm{FPR}}_{s'}(r)|dr$.
    Define 
    \[
    \mu_s=\mathrm{Law}(U_s(X_0)),\quad \mu_{s'}=\mathrm{Law}(U_{s'}(X_0)),
    \]
    Thus by definition and connection between Wasserstein distance and CDF (e.g. proposition 2.17 of \citep{Santambrogio2015OptimalTF}), we have
    \begin{align}\label{eq:6}
    W_1(\mu_s,\mu_{s'})=\int_0^1|\widetilde{\mathrm{FPR}}_s(r)-\widetilde{\mathrm{FPR}}_{s'}(r)|dr.
    \end{align}
    By the definition of Wasserstein distance, for any coupling $\Gamma$ of $(U_s(X_0),U_{s'}(X_0))$, we have
    \begin{align}\label{eq:7}
        W_1(\mu_s,\mu_{s'})\leq \bE_{\Gamma}\left|U_s(X_0)-U_{s'}(X_0)\right|.
    \end{align}
 Consider the natural coupling where $U_s(X_0)$ and $U_{s'}(X_0)$ are coupled through some $X_0$. Then for fixed $X_0$, 
 \begin{align*}
     U_s(X_0)=\bE\left[1\{s(X_1)\geq s(X_0)\}\mid X_0\right],\quad U_{s'}(X_0)=\bE\left[1\{s'(X_1)\geq s'(X_0)\}\mid X_0\right].
 \end{align*}
 Thus we have
    \begin{align*}
        \left|U_s(X_0)-U_{s'}(X_0)\right|\leq \bE \left[\left|1\{s(X_1)\geq s(X_0)\}-1\{s'(X_1)\geq s'(X_0)\}\right|\mid X_0\right]\\
        \le \mathbb P( 1\{s(X_1)\geq s(X_0)\}\neq 1\{s'(X_1)\geq s'(X_0)\} \mid X_0)
    \end{align*}
    Taking expectation over $X_0$ gives
    \begin{align}\label{eq:8}
        \bE\left|U_s(X_0)-U_{s'}(X_0)\right|
        \leq \mathbb  P( 1\{s(X_1)\geq s(X_0)\}\neq 1\{s'(X_1)\geq s'(X_0)\}).
    \end{align}
    Then, combining equations (\ref{eq:6}), (\ref{eq:7}) and (\ref{eq:8}), we have
    \begin{align}\label{eq:FPR-difference-diff-s}
        \nonumber \int_0^1|\widetilde{\mathrm{FPR}}_s(r)-\widetilde{\mathrm{FPR}}_{s'}(r)|dr\leq \mathbb P( 1\{s(X_1)\geq s(X_0)\}\neq 1\{s'(X_1)\geq s'(X_0)\})\\
        =\mathbb P(\mathrm{Sign}(s(X_1)-s(X_0))\neq\mathrm{Sign(s'(X_1)-s'(X_0))}).
    \end{align}
    Define $\Delta_s=s(X_1)-s(X_0)$. By an argument similar to the AUROC case, we have
    \begin{align*}
        &\left\{\mathrm{Sign}(\Delta_s)\neq \mathrm{Sign}(\Delta_{s'})\right\}\Rightarrow \left\{|\Delta_s|\leq\left|\Delta_s-\Delta_{s'}\right|\right\}\\&\quad\Rightarrow \left\{|\Delta_s|\leq\left|s(X_1)-s'(X_1)\right|+\left|s(X_0)-s'(X_0)\right|\right\} \Rightarrow \left\{|\Delta_s|\leq2\delta_n\right\}.
    \end{align*}
    Therefore by assumption, also condition on the training data such that $\delta_n$ is fixed, it holds that
    \begin{align}\label{eq:diff-sign-to-s-diff}
        \mathbb P\left(\mathrm{Sign}(\Delta_s)\neq \mathrm{Sign}(\Delta_{s'})\right)\leq \mathbb P\left(|\Delta_s|\leq2\delta_n\right)
        \leq 2C\delta_n,
    \end{align}
    where the last step holds by assumption when $s$ is taken as $g(\Lambda_{\eta^*}(\cdot))$ and $s'$ is taken as $\hat{s}_n$. Thus it holds that
    \begin{align*}
        \int_0^1|\widetilde{\mathrm{FPR}}_{\hat s_n}(r)-\widetilde{\mathrm{FPR}}_{g(\Lambda_{\eta^*})}(r)|dr\leq 2C\delta_n=O_P(r_n).
    \end{align*}
    Thus, by well-specification assumption, 
    combining equation (\ref{eq:9}) for the non-trivial bound, and combining the likelihood-ratio optimality, 
    \begin{align}\label{eq:AUPRC-to-parameter}
        \AUPRC(\Lambda_{\eta^*})-\AUPRC(\hat s_n)=|\AUPRC(\hat s_n)-\AUPRC(g(\Lambda_{\eta^*}))|=O_P\left(\sqrt{\frac{\pi_0 r_n}{\pi_1}}\right).
    \end{align}
\end{proof}

\begin{proof}[Proof of Theorem \ref{thm:doesnotimproveBA}]
We prove the results for balanced accuracy and $\F_1$ score separately.

\medskip
\noindent\textbf{Balanced Accuracy:}

    We first show that the best-threshold balanced accuracy is invariant under strictly increasing score transformations. Let $h$ be strictly increasing. Then the balanced accuracy satisfies
    \begin{align*}
        \BA(h\circ s,u)=\frac{1}{2}\left(\mathbb P(s(x_1)\geq h^{-1}(u))+\mathbb P(s(x_0)<h^{-1}(u))\right)=\BA(s,h^{-1}(u)).
    \end{align*}
    Taking supremum over all $u$ and noting that $u\mapsto h^{-1}(u)$ is a bijection, we have
    \begin{align*}
        \sup_u\BA(h\circ s,u)=\sup_\tau \BA(s,\tau).
    \end{align*}
    That is, we have for any strictly increasing transformation $g$,
    \begin{align*}
        \sup_\tau\BA(g(\Lambda_{\eta^*}),\tau)=\sup_\tau \BA(\Lambda_{\eta^*},\tau).
    \end{align*}
    Then, for two scores $s$ and $s'$, and for a fixed threshold $\tau$, 
    \begin{align}\label{eq:BA-decomp-absolute}
        \nonumber &\left|\BA(s,\tau)-\BA(s',\tau)\right|\\
        &\leq \frac{1}{2}\left|\mathbb P(s(X_1)\geq \tau)-\mathbb P(s'(X_1)\geq \tau)\right|+\frac{1}{2}\left|\mathbb P(s(X_0)< \tau)-\mathbb P(s'(X_0)< \tau)\right|.
    \end{align}
    We study the first term. It holds that
    \begin{align*}
        &\left|\mathbb P(s(X_1)\geq \tau)-\mathbb P(s'(X_1)\geq \tau)\right|=\left|\bE\left[1\{s(X_1)\geq \tau\}-1\{s'(X_1)\geq \tau\}\right]\right|\\
        &\leq \bE\left|1\{s(X_1)\geq \tau\}-1\{s'(X_1)\geq \tau\}\right|=\bE1\left\{1\{s(X_1)\geq \tau\}\neq 1\{s'(X_1)\geq \tau\}\right\}\\
        &=\mathbb P\left(1\{s(X_1)\geq \tau\}\neq 1\{s'(X_1)\geq \tau\}\right).
    \end{align*}
    We thus have
    \begin{align}\label{eq:BA-score-event}
       \left\{1\{s(X_1)\geq \tau\}\neq 1\{s'(X_1)\geq \tau\}\right\}\subseteq \left\{|s(X_1)-\tau|\leq |s(X_1)-s'(X_1)|\right\}.
    \end{align}
    Consider $$\delta_n:=\|\hat{s}-g\circ \Lambda_{\eta^*}\|_{L_\infty (P_0+P_1)}=O_P(r_n)=o_P(1).$$
Combining the above with the assumption, and conditioning on the training data so that $\delta_n$ is fixed, we have
    \begin{align}\label{eq:prob-to-parameter}
        \nonumber\left|\mathbb P(s(X_1)\geq \tau)-\mathbb P(s'(X_1)\geq \tau)\right|\leq \mathbb P \left(|s(X_1)-\tau|\leq |s(X_1)-s'(X_1)|\right)\\
        \leq \mathbb P \left(|s(X_1)-\tau|\leq \delta_n\right)\leq C\delta_n,
    \end{align}
    by taking $s$ as the increasingly transformed likelihood $g(\Lambda_{\eta^*})$ and $s'$ as $\hat{s}_n$. Also, $\delta_n\leq t_0$ with probability tending to $1$.
    The same argument for the negative class gives
    \begin{align*}
        \left|\mathbb P(s(X_0)< \tau)-\mathbb P(s'(X_0)< \tau)\right|\leq C\delta_n.
    \end{align*}
    Plugging back into equation (\ref{eq:BA-decomp-absolute}), and taking $s'$ as $\hat s_n$, we have
    \begin{align*}
        \left|\BA(\hat{s}_n,\tau)-\BA(g(\Lambda_{\eta^*}),\tau)\right|\leq  C\delta_n.
    \end{align*}
    Further, by the elementary inequality
    \(
        |\sup_\tau f(\tau)-\sup_\tau g(\tau)|\leq \sup_\tau |f(\tau)-g(\tau)|\),
    we have
    \begin{align*}
        \left|\sup_\tau\BA(\hat s_n,\tau)-\sup_\tau\BA(g(\Lambda_{\eta^*}),\tau)\right|\leq \sup_\tau\left|\BA(\hat s_n,\tau)-\BA(g(\Lambda_{\eta^*}),\tau)\right|\leq C\delta_n.
    \end{align*}
    Thus we have transformed the difference of balanced accuracy into the score distance. Finally, by the invariance in strictly increasing transformation and the well-specification assumption, combined with likelihood-ratio optimality, we finish the proof for BA.

\medskip
\noindent\textbf{$\F_1$ Score:}

    We still first show that the best-threshold $\F_1$ is invariant under strictly increasing transformations. Still let $h$ be strictly increasing, and for any threshold $u$,
    \begin{align*}
        \{h(s(x))\geq u\}\Leftrightarrow \{s(x)\geq h^{-1}(u)\}\qquad \{h(s(x))< u\}\Leftrightarrow \{s(x)< h^{-1}(u)\}.
    \end{align*}
    Hence by similar analysis on recall, precision, and FNR,
    \begin{align*}
        \sup_u \F_1(h\circ s, u)= \sup_\tau \F_1(s,\tau),\qquad \sup_\tau \F_1(g(\Lambda_{\eta^*}), \tau)= \sup_\tau \F_1(\Lambda_{\eta^*},\tau).
    \end{align*}
    Consider fixed scores $s$, $s'$ and fixed threshold $\tau$. For brevity, denote
    \begin{align*}
        R:=\Rec(s,\tau),\quad Q:=\FPR(s,\tau),\quad R':=\Rec(s',\tau), \quad Q':=\FPR(s',\tau).
    \end{align*}
    Then the $\F_1$ score can be denoted as
    \begin{align*}
        \F_1(s,\tau)=G(R,Q),\quad \text{where } G(r,q):=\frac{2\pi_1 r}{\pi_1(1+r)+\pi_0q}.
    \end{align*}
    Take derivative, we have
    \begin{align*}
        \left|\frac{\partial G}{\partial r}\right|=\left|\frac{2\pi_1(\pi_1+\pi_0q)}{(\pi_1(1+r)+\pi_0q)^2}\right|\leq \frac{2}{\pi_1}, \quad
        \left|\frac{\partial G}{\partial q}\right|=\left|-\frac{2\pi_0\pi_1r}{(\pi_1(1+r)+\pi_0q)^2}\right|\leq \frac{2}{\pi_1}.
    \end{align*}
    By the mean value theorem,
    \begin{align*}
        |\F_1(s,\tau)-\F_1(s',\tau)|=|G(R,Q)-G(R',Q')|\leq \frac{2}{\pi_1}\left(|R-R'|+|Q-Q'|\right).
    \end{align*}
    By similar procedure with equation (\ref{eq:prob-to-parameter}), condition on the training data, we have
    \begin{align}\label{eq:RRQQ-diff}
       \nonumber |R-R'|=\left|\mathbb P(s(X_1)\geq\tau)-\mathbb P(s'(X_1)\geq\tau)\right|\leq C\delta_n,\\
        |Q-Q'|=\left|\mathbb P(s(X_0)\geq\tau)-\mathbb P(s'(X_0)\geq\tau)\right|\leq C\delta_n,
    \end{align}
    where we take $s$ as $g(\Lambda_{\eta^*})$ and take $s'$ as $\hat{s}_n$. Also, $\delta_n\leq t_0$ with probability tending to $1$.
    Therefore
    \begin{align*}
        |\F_1(\hat s_n,\tau)-\F_1(g(\Lambda_{\eta^*}),\tau)|\leq \frac{4C\delta_n}{\pi_1}.
    \end{align*}
    Further, 
    we have
    \begin{align*}
        \left|\sup_\tau \F_1(\hat s_n,\tau)-\sup_\tau \F_1(g(\Lambda_{\eta^*}),\tau)\right|\leq \sup_\tau\left| \F_1(\hat s_n,\tau)- \F_1(g(\Lambda_{\eta^*}),\tau)\right|
        \leq \frac{4C\delta_n}{\pi_1}.
    \end{align*}
    Finally, by the well-specification assumption, and optimality of the likelihood ratio, we finish the proof.
\end{proof}

\begin{proof}[Proof of Theorem \ref{thm:ERM-well-specified}]
    In the proof of Theorem \ref{thm:AUC-invariant}, we have shown that AUROC is invariant under a strictly increasing transformation, that is 
    \begin{equation}\label{eq:increasingfunc-invariance}
    \AUC(h\circ s)=\AUC(s),
    \end{equation}
    where $h$ is strictly increasing.
    For any two pairs of weights $w=(w_0,w_1)$ and $w'=(w'_0,w'_1)$, by assumption, we have
    \(
    s_{\theta_w}(x)=g_{\theta_w}(\Lambda_{\eta^*}(x))\), \(s_{\theta_{w'}}(x)=g_{\theta_{w'}}(\Lambda_{\eta^*}(x))\) 
    with $g_{\theta_w}$ and $g_{\theta_{w'}}$ being strictly increasing functions. Also, because $g_{\theta_w}$ is increasing, it is invertible and $g_{\theta_w}^{-1}$ is also strictly increasing. Then the transformation $h:=g_{\theta_{w'}}\circ g_{\theta_w}^{-1}$ is strictly increasing. Then
    \begin{align*}
        s_{\theta_{w'}}(x)=g_{\theta_{w'}}(\Lambda_{\eta^*}(x))=(g_{\theta_{w'}}\circ g_{\theta_w}^{-1})(g_{\theta_w}(\Lambda_{\eta^*}(x)))=h(s_{\theta_{w}}(x)).
    \end{align*}
    Thus by equation (\ref{eq:increasingfunc-invariance}), we have
    \begin{equation}\label{eq:weight-invariance}
    \AUC(\theta_{w'})=\AUC(\theta_w),
    \end{equation}
    which means AUROC is invariant to weights in the population weighted objective. For ease of notation, let $\Delta_\theta:=s_\theta(X_1)-s_\theta(X_0)$. By Lipschitz score assumption,
    \begin{align*}
        \left|\Delta_\theta-\Delta_{\theta'}\right|&=\left|s_\theta(X_1)-s_\theta(X_0)-(s_{\theta'}(X_1)-s_{\theta'}(X_0))\right|\\
        &\leq \left|s_\theta(X_1)-s_{\theta'}(X_1)\right|+\left|s_{\theta}(X_0)-s_{\theta'}(X_0)\right|\leq 2L\|\theta-\theta'\|.
    \end{align*}
Then by similar steps as in equations (\ref{eq:AUC-to-sign-diff}) and (\ref{eq:sign-diff-to-deltan}), combined with assumption,
    \begin{align*}
        \left|\AUC(\theta)-\AUC(\theta')\right|\leq \mathbb P\left(|\Delta_\theta|\leq\left|\Delta_\theta-\Delta_{\theta'}\right|\right)\leq \mathbb P\left(|\Delta_\theta|\leq2L\|\theta-\theta'\|\right)\leq 2CL\|\theta-\theta'\|,
    \end{align*}
    whenever $2L\|\theta-\theta'\|\leq t_0$.
    Here we replace $C$ by $\max\{C,t_0^{-1}\}$ if necessary, then the bound holds for all $2L\left\|\theta-\theta'\right\|>0$.
    Therefore, consider the raw data estimator, condition on the training data, we have
    \begin{align*}
        \left|\AUC(\hat{\theta}_{\raw})-\AUC(\theta_{\raw}^*)\right|\leq 2CL\left\|\hat{\theta}_{\raw}-\theta^*_{\raw}\right\|,
    \end{align*}
    For parameter estimation, we have the following lemmas,
    \begin{lemma}[Empirical Estimation Error]\label{lemma:empiricalestimation}
    Under the assumptions of the theorem, for any $\delta>0$, and for sufficiently large sample sizes $n_0,n_1$, there exists a constant $c>0$ such that each of the following two statements holds separately with probability at least $1-\delta$ that
    \begin{align*}
        \left\|\hat{\theta}_{\raw}-\theta^*_{\raw}\right\|\leq \frac{c}{\lambda}\sqrt{\frac{B^2d\log(6d/\delta)}{n_0+n_1}}, \quad
        \left\|\theta^*_{\aug}-\hat{\theta}_{\aug}\right\|\leq \frac{c}{\lambda}\sqrt{\frac{B^2d\log(8d/\delta)}{n_0+n_1+\tilde{n}}}.    \end{align*}
    When the two bounds hold simultaneously, we have, with probability at least $1-\delta$, that
    \begin{align*}
        \left\|\hat{\theta}_{\raw}-\theta^*_{\raw}\right\|\leq \frac{c}{\lambda}\sqrt{\frac{B^2d\log(14d/\delta)}{n_0+n_1}}, \quad
        \left\|\theta^*_{\aug}-\hat{\theta}_{\aug}\right\|\leq \frac{c}{\lambda}\sqrt{\frac{B^2d\log(14d/\delta)}{n_0+n_1+\tilde{n}}}.    \end{align*}
    \end{lemma}
    Therefore, by Lemma \ref{lemma:empiricalestimation}, it holds that, with probability at least $1-\delta$,
    \begin{align}\label{eq:raw-auc-closeness}
        \left|\AUC(\hat{\theta}_{\raw})-\AUC(\theta_{\raw}^*)\right|\leq \frac{CLc}{\lambda}\sqrt{\frac{B^2d\log(6d/\delta)}{n_0+n_1}}.
    \end{align}
    Also, by the strictly increasing transformation invariance,  we finish proving the statement for the raw data estimator. 

    Similarly for the synthetic augmented estimator,
    \begin{align*}
        \left|\AUC(\hat{\theta}_{\aug})-\AUC(\theta_{\aug}^*)\right|\leq 2CL\left\|\hat{\theta}_{\aug}-\theta^*_{\aug}\right\|.
    \end{align*}
    In addition, define $\rho:=\frac{n_0}{n_0+n_1+\tilde{n}}$, and define the risk
\begin{align*}
    \cR_\rho(\theta)=\frac{n_0}{n_0+n_1+\tilde{n}}\mathbb{E}_{P_0}\ell(\theta;x,0)+\frac{n_1+\tilde{n}}{n_0+n_1+\tilde{n}}\mathbb{E}_{P_1}\ell(\theta;x,1),
\end{align*}
and let $\theta_\rho^*$ denote its minimizer. We then have the following lemma relating the synthetic augmented parameter and the corresponding rebalanced parameter.
\begin{lemma}[Synthetic Error]\label{lemma:sytheticerror}
        Under the assumptions of the theorem, the synthetic parameter error satisfies
        \begin{align*}
            \left\|\theta^*_\rho-\theta^*_{\aug}\right\|\leq \frac{L_g}{\sqrt{2\mu L_{\rho}}}\frac{\tilde{n}}{n_0+n_1+\tilde{n}}\epsilon_{\mathrm{syn}}.
        \end{align*}
    \end{lemma}
Thus, by Lemma \ref{lemma:sytheticerror} and the parameter estimation for AUROC, we have
\begin{align}\label{eq:aug-auc-closenesss}
       \nonumber \left|\AUC(\hat{\theta}_{\aug})-\AUC(\theta_{\rho}^*)\right|&\leq \left|\AUC(\hat{\theta}_{\aug})-\AUC(\theta_{\aug}^*)\right|+\left|\AUC(\theta_{\aug}^*)-\AUC(\theta_{\rho}^*)\right|\\
        \nonumber&\leq
        2CL\left\|\hat{\theta}_{\aug}-\theta^*_{\aug}\right\|+2CL\left\|\theta_{\aug}^*-\theta^*_{\rho}\right\|\\
        &\leq \frac{CLc}{\lambda}\sqrt{\frac{B^2d\log(8d/\delta)}{n_0+n_1+\tilde{n}}}+CL_gL\sqrt{\frac{2}{\mu L_{\rho}}}\frac{\tilde{n}}{n_0+n_1+\tilde{n}}\epsilon_{\syn},
    \end{align}
    with probability at least $1-\delta$.
Finally by the invariance to strictly increasing function, we finish the proof.
\end{proof}

\begin{proof}[Proof of Theorem \ref{thm:AUPRC-converge-ERM}]
    By previous analysis, AUPRC is invariant under strictly increasing transformations of the score, and it is invariant to weights. Equation (\ref{eq:9}) gives that
    \begin{align*}
        |\AUPRC(s)-\AUPRC(s')|\leq\min\left\{1,~2\sqrt{\frac{\pi_0}{\pi_1}\int_0^1\Delta(r)dr}\right\}.
    \end{align*}
    By Lipschitz score, we have
    \begin{align*}
        \left\{\mathrm{Sign}(\Delta_\theta)\neq \mathrm{Sign}(\Delta_{\theta'})\right\}\Rightarrow \left\{|\Delta_\theta|\leq\left|\Delta_\theta-\Delta_{\theta'}\right|\right\}\Rightarrow \left\{|\Delta_\theta|\leq2L\|\theta-\theta'\|\right\}.
    \end{align*}
    Then by assumption and similar step with equations (\ref{eq:FPR-difference-diff-s}) and (\ref{eq:diff-sign-to-s-diff}), it holds that
    \begin{align*}
        \int_0^1|\widetilde{\mathrm{FPR}}_s(r)-\widetilde{\mathrm{FPR}}_{s'}(r)|dr\leq2CL\|\theta-\theta'\|,
    \end{align*}
    where we replace $C$ by $\max\{C,t_0^{-1}\}$ if necessary, then the bound holds for all $2L\left\|\theta-\theta'\right\|>0$.
   Thus,
    \begin{align*}
        |\AUPRC(s_\theta)-\AUPRC(s_{\theta'})|\leq2\sqrt{\frac{\pi_0}{\pi_1}2CL\|\theta-\theta'\|}.
    \end{align*}
    We have thus connected the AUPRC difference to the parameter error. Then, similar to the analysis for AUROC, it holds that for the raw data estimator,
    \begin{align*}
        |\AUPRC(\hat\theta_\raw)-\AUPRC(\theta_\raw^*)|\leq2\sqrt{\frac{\pi_0}{\pi_1}2CL\|\hat\theta_\raw-\theta_\raw^*\|}.
    \end{align*}
    and for the synthetic augmented estimator
\begin{align*}
        &|\AUPRC(\hat\theta_\aug)-\AUPRC(\theta_\rho^*)|\\
        &\leq |\AUPRC(\hat\theta_\aug)-\AUPRC(\theta_\aug^*)|+|\AUPRC(\theta_\aug^*)-\AUPRC(\theta_\rho^*)|\\
        &\leq2\sqrt{\frac{\pi_0}{\pi_1}2CL\|\hat\theta_\aug-\theta_\aug^*\|}+2\sqrt{\frac{\pi_0}{\pi_1}2CL\|\theta^*_\aug-\theta_\rho^*\|}.
    \end{align*}
    Finally, combining lemmas \ref{lemma:empiricalestimation} and \ref{lemma:sytheticerror}, along with the AUPRC invariance to increasing function, we have the results for both raw and augmented data estimators.
\end{proof}

\begin{proof}[Proof of Theorem \ref{thm:parametric-ERM-best-BA}]
    By the proof of Theorem \ref{thm:doesnotimproveBA}, we have the invariance to strictly increasing function of the best threshold balanced accuracy
    \begin{align*}
        \sup_u\BA(h\circ s,u)=\sup_\tau \BA(s,\tau).
    \end{align*}
    We proceed with the proof of Theorem \ref{thm:doesnotimproveBA} starting from equation (\ref{eq:BA-score-event}). Combining the Lipschitz score assumption, we have
    \begin{align*}
       \left\{1\{s_{\theta}(X_1)\geq \tau\}\neq 1\{s_{\theta'}(X_1)\geq \tau\}\right\}\subseteq \left\{|s_\theta(X_1)-\tau|\leq |s_\theta(X_1)-s_{\theta'}(X_1)|\right\}\\
       \subseteq \left\{|s_\theta(X_1)-\tau|\leq L\|\theta-\theta'\|\right\}.
    \end{align*}
    That is
    \begin{align*}
        \left|\mathbb  P(s_{\theta}(X_1)\geq \tau)-\mathbb  P(s_{\theta'}(X_1)\geq \tau)\right|\leq \mathbb  P \left(|s_\theta(X_1)-\tau|\leq L\|\theta-\theta'\|\right)\leq CL\|\theta-\theta'\|.
    \end{align*}
    Here we replace $C$ by $\max\{C,t_0^{-1}\}$ if necessary, then the bound holds for all $L\left\|\theta-\theta'\right\|>0$.
    The same argument for the negative class gives
    \begin{align*}
        \left|\mathbb  P(s_{\theta}(X_0)< \tau)-\mathbb  P(s_{\theta'}(X_0)< \tau)\right|\leq CL\|\theta-\theta'\|.
    \end{align*}
    Plugging back into equation (\ref{eq:BA-decomp-absolute}), we have
    \begin{align*}
        \left|\BA(s_\theta,\tau)-\BA(s_{\theta'},\tau)\right|\leq CL\|\theta-\theta'\|.
    \end{align*}
    Further, taking supremum gives
    \begin{align*}
        \left|\sup_\tau\BA(s_\theta,\tau)-\sup_\tau\BA(s_{\theta'},\tau)\right|\leq \sup_\tau\left|\BA(s_\theta,\tau)-\BA(s_{\theta'},\tau)\right|\leq CL\|\theta-\theta'\|.
    \end{align*}
    Thus we have transformed the difference of balanced accuracy into the parameter distance. The parameter distance for $\hat{\theta}_\aug$ and $\hat{\theta}_\raw$ has been well studied in Lemmas \ref{lemma:empiricalestimation} and \ref{lemma:sytheticerror}. Therefore, conditioning on the training data, combining the above and by similar procedure with the analysis of AUROC and AUPRC, we have the results for both raw and augmented estimators.
\end{proof}

\begin{proof}[Proof of Theorem \ref{thm:parametric-ERM-F1}]
    Proceeding with equation (\ref{eq:RRQQ-diff}) in the proof of Theorem \ref{thm:doesnotimproveBA}, along with the Lipschitz score condition, it holds that
    \begin{align*}
        |R-R'|=\left|\mathbb  P_1(s_\theta\geq\tau)-\mathbb  P_1(s_{\theta'}\geq\tau)\right|\leq CL\|\theta-\theta'\|,\\
        |Q-Q'|=\left|\mathbb  P_0(s_\theta\geq\tau)-\mathbb  P_0(s_{\theta'}\geq\tau)\right|\leq CL\|\theta-\theta'\|.
    \end{align*}
    As before, we enlarge $C$ when necessary so the bounds always hold.
    That is we have
    \begin{align*}
        |\F_1(s_\theta,\tau)-\F_1(s_{\theta'},\tau)|\leq \frac{4CL}{\pi_1}\|\theta-\theta'\|.
    \end{align*}
    Further, taking supremum,
    \begin{align*}
        \left|\sup_\tau \F_1(s_\theta,\tau)-\sup_\tau \F_1(s_{\theta'},\tau)\right|\leq \sup_\tau\left| \F_1(s_\theta,\tau)- \F_1(s_{\theta'},\tau)\right|\leq \frac{4CL}{\pi_1}\|\theta-\theta'\|.
    \end{align*}
    Thus we have transformed the difference of $\F_1$ score into the parameter distance. The parameter distance for $\hat{\theta}_\aug$ and $\hat{\theta}_\raw$ has been well studied in Lemmas \ref{lemma:empiricalestimation} and \ref{lemma:sytheticerror}. Combining the above with invariance under strictly increasing transformations,  we have the results for both raw and augmented estimators.
\end{proof}

\begin{proof}[Proof of Theorem \ref{thm:AUROC-minimax}]
We first construct a Euclidean packing.
For constant $\alpha\in (0,1)$ in Assumption \ref{assumption:minimaxAUC}, define
\(
r_\alpha:=\frac{1+\alpha}{2}\), 
then
\(
\alpha<r_\alpha<1\).
Consider the Euclidean ball
\[
r_\alpha B_2^{d-1}
=
\{x\in\mathbb R^{d-1}:\|x\|_2\le r_\alpha\}.
\]
By Corollary 4.2.11 in \cite{vershynin2018high}, the covering number of the Euclidean unit ball satisfies
\[
\mathcal N(B_2^{d-1},\epsilon)
\ge
\left(\frac1\epsilon\right)^{d-1}.
\]
By scaling,
\[
\mathcal N(r_\alpha B_2^{d-1},\alpha)
=
\mathcal N\left(B_2^{d-1},\frac{\alpha}{r_\alpha}\right)
\ge
\left(\frac{r_\alpha}{\alpha}\right)^{d-1}.
\]
Since
\(
\frac{r_\alpha}{\alpha}
=
\frac{1+\alpha}{2\alpha}
>1\),
this covering number is exponential in $d$.
Then let
\[
\{x_1,\ldots,x_M\}
\subset r_\alpha B_2^{d-1}
\]
be a maximal $\alpha$-separated set. Since every maximal $\alpha$-separated set is an $\alpha$-net, its cardinality is at least the $\alpha$-covering number. Hence
\[
M
\ge
\mathcal N(r_\alpha B_2^{d-1},\alpha)
\ge
\left(\frac{r_\alpha}{\alpha}\right)^{d-1}
=
\left(\frac{1+\alpha}{2\alpha}\right)^{d-1}.
\]
Next, embed $r_\alpha B_2^{d-1}$ into the upper hemisphere of $\mathbb S^{d-1}$ by
\[
\Phi(x)
=
\left(x,\sqrt{1-\|x\|_2^2}\right).
\]
This map is well-defined because $r_\alpha<1$. Define
\(
u_i:=\Phi(x_i)\), \(i=1,\ldots,M\).
Then
\(
u_i\in \mathbb S^{d-1}\).
For any $i\ne j$,
\[
\|u_i-u_j\|_2^2
=
\|x_i-x_j\|_2^2
+
\left(
\sqrt{1-\|x_i\|_2^2}
-
\sqrt{1-\|x_j\|_2^2}
\right)^2
\ge
\|x_i-x_j\|_2^2.
\]
Since the $x_i$'s are $\alpha$-separated,
\(
\|x_i-x_j\|_2\ge \alpha\).
Therefore,
\[
\|u_i-u_j\|_2\ge \alpha,
\qquad i\ne j.
\]
Moreover, since $\alpha\in(0,1)$ is fixed, for $d\geq 2$,
\[
\log M
\ge
(d-1)\log\left(\frac{1+\alpha}{2\alpha}\right)\geq c_1 d.
\]
where one may take
\(
c_1
=
\frac12
\log\left(\frac{1+\alpha}{2\alpha}\right)\).
Now define the local parameter points
\[
\eta_i
:=
\eta_0+h u_i,
\qquad i=1,\ldots,M.
\]
If $h\le r_0$, then,
\(
\eta_i\in B_2(\eta_0,r_0)\subseteq\Theta\).
Furthermore,
\(
\|\eta_i-\eta_j\|_2
=
h\|u_i-u_j\|_2
\ge
\alpha h\), for \(i\ne j\).
Therefore,
\(\{\eta_1,\ldots,\eta_M\}\)
is a Euclidean packing of the local parameter space with separation $\alpha h$, and its cardinality satisfies
\(
\log M\ge c_1 d\).

Next, we give a pairwise AUC excess representation.
For fixed parameter \(\eta\), write
\[
Q_\eta=P_{0,\eta}\otimes P_{0,\eta}.
\]
We first rewrite the AUC excess in a pairwise form. For any measurable score \(s\),
\[
\operatorname{AUC}_{\eta}(s)
=
\iint
H(s(x)-s(z))\,dP_{1,\eta}(x)\,dP_{0,\eta}(z),\quad \text{where } H(t):=\mathbf 1\{t>0\}+\frac12\mathbf 1\{t=0\}.
\]
Since
\(
dP_{1,\eta}(x)
=
L_{\eta}(x)dP_{0,\eta}(x)\),
and since \(Q_{\eta}=P_{0,\eta}\otimes P_{0,\eta}\), we may write
\[
\operatorname{AUC}_{\eta}(s)
=
\iint
L_{\eta}(x)
H(s(x)-s(z))\,dQ_{\eta}(x,z).
\]
Using the symmetry of \(Q_{\eta}\) in \((x,z)\), the same quantity also equals
\[
\iint
L_{\eta}(z)
H(s(z)-s(x))\,dQ_{\eta}(x,z).
\]
Averaging the two displays gives the symmetrized representation
\[
\operatorname{AUC}_{\eta}(s)
=
\frac12
\iint
\left[
L_{\eta}(x)H(s(x)-s(z))
+
L_{\eta}(z)H(s(z)-s(x))
\right]
\,dQ_{\eta}(x,z).
\]
For a score \(s\), define the pairwise oracle-disagreement loss
\[
L_{\eta,s}(x,z)
:=
\begin{cases}
0,
& \text{whenever } L_\eta(x)=L_\eta(z) ,\\[1mm]
0,
& \text{if } s \text{ ranks } x,z \text{ in the same strict order as } L_{\eta},\\[1mm]
1/2,
& \text{if } s(x)=s(z),\\[1mm]
1,
& \text{if } s \text{ ranks } x,z \text{ in the opposite strict order from } L_{\eta}.
\end{cases}
\]
For the pairwise likelihood-ratio contrast
\(
\Delta_{\eta}(x,z)
=
L_{\eta}(x)-L_{\eta}(z)\), we claim that
\[
\operatorname{AUC}_{\eta}(\Lambda_{\eta})
-
\operatorname{AUC}_{\eta}(s)=\operatorname{AUC}_{\eta}(L_{\eta})
-
\operatorname{AUC}_{\eta}(s)
=
\frac12
\iint
|\Delta_{\eta}(x,z)|
L_{\eta,s}(x,z)
\,dQ_{\eta}(x,z).
\]
To verify this identity, fix a pair \((x,z)\) and write
\(
a:=L_{\eta}(x)\) and \(
b:=L_{\eta}(z)\).
The symmetrized pairwise contribution of \(s\) is
\[
C_s(x,z)
:=
\frac12
\left[
aH(s(x)-s(z))
+
bH(s(z)-s(x))
\right],
\]
while the corresponding oracle contribution is
\[
C_{L}(x,z)
:=
\frac12
\left[
aH(a-b)+bH(b-a)
\right].
\]
If \(a>b\), then the oracle ranks \(x\) above \(z\), and
\(
C_{L}(x,z)=\frac a2\).
In this case, if \(s(x)>s(z)\), then \(C_s(x,z)=a/2\), so the excess is zero. If
\(s(x)=s(z)\), then
\[
C_s(x,z)=\frac12\left(\frac a2+\frac b2\right)=\frac{a+b}{4},
\]
and hence
\[
C_{L}(x,z)-C_s(x,z)
=
\frac a2-\frac{a+b}{4}
=
\frac{a-b}{4}
=
\frac12 |a-b|\cdot \frac12.
\]
If \(s(x)<s(z)\), then \(C_s(x,z)=b/2\), and hence
\[
C_{L}(x,z)-C_s(x,z)
=
\frac a2-\frac b2
=
\frac{a-b}{2}
=
\frac12 |a-b|\cdot 1.
\]
Thus, when \(a>b\),
\[
C_{L}(x,z)-C_s(x,z)
=
\frac12 |\Delta_{\eta}(x,z)|
L_{\eta,s}(x,z).
\]
The case \(a<b\) is symmetric. If \(a=b\), then \(|\Delta_{\eta}(x,z)|=0\), and the pair contributes no excess regardless of how \(s\) ranks \(x\) and \(z\). Therefore, for every pair \((x,z)\),
\[
C_{L}(x,z)-C_s(x,z)
=
\frac12
|\Delta_{\eta}(x,z)|
L_{\eta,s}(x,z).
\]
Integrating this pointwise identity with respect to \(Q_{\eta}\) yields
\[
\operatorname{AUC}_{\eta}(\Lambda_{\eta})
-
\operatorname{AUC}_{\eta}(s)
=\operatorname{AUC}_{\eta}(L_{\eta})
-
\operatorname{AUC}_{\eta}(s)=
\frac12
\iint
|\Delta_{\eta}(x,z)|
L_{\eta,s}(x,z)
\,dQ_{\eta}(x,z).
\]
This is the desired pairwise representation of the AUC excess risk.
Then the AUC excess has the exact representation
\(
\mathcal E_\eta(s)
=
\frac12
\int
|\Delta_\eta(x,z)|
L_{\eta,s}(x,z)
\,dQ_\eta(x,z)\).
Using \(Q_\eta\ge c_Q Q_0\), we get
\[
\mathcal E_\eta(s)
\ge
\frac{c_Q}{2}
\int
|\Delta_\eta(x,z)|
L_{\eta,s}(x,z)
\,dQ_0(x,z).
\]

Then we transfer the Euclidean packing into AUC-space separation.
Fix \(i\ne j\). Let
\(
\eta_i=\eta_0+hu_i\), and 
\(\eta_j=\eta_0+hu_j\).
By previous construction,
\(
\|u_i-u_j\|_2\ge \alpha\).
Define
\[
\mathcal R_{ij,h}
:=
\mathcal R_h(u_i,u_j)
=
\left\{
(x,z):
\Delta_{\eta_i}(x,z)
\Delta_{\eta_j}(x,z)<0
\right\}.
\]
That is, on \(\mathcal R_{ij,h}\), the two oracle likelihood-ratio rankings are
opposite.
For any score \(s\), one score cannot agree with both opposite rankings.
Therefore,
\[
L_{\eta_i,s}(x,z)
+
L_{\eta_j,s}(x,z)
\ge 1
\qquad
\text{on } \mathcal R_{ij,h}.
\]
Using the AUC excess representation,
\[
\mathcal E_{\eta_i}(s)
+
\mathcal E_{\eta_j}(s)
\ge
\frac{c_Q}{2}
\int_{\mathcal R_{ij,h}}
\left[
|\Delta_{\eta_i}|L_{\eta_i,s}
+
|\Delta_{\eta_j}|L_{\eta_j,s}
\right]
\,dQ_0.
\]
For nonnegative \(a,b,\ell_1,\ell_2\),
\(
a\ell_1+b\ell_2
\ge
\min\{a,b\}(\ell_1+\ell_2)\).
Therefore, on \(\mathcal R_{ij,h}\),
\[
|\Delta_{\eta_i}|L_{\eta_i,s}
+
|\Delta_{\eta_j}|L_{\eta_j,s}
\ge
\min\{
|\Delta_{\eta_i}|,
|\Delta_{\eta_j}|
\}.
\]
Hence
\[
\mathcal E_{\eta_i}(s)
+
\mathcal E_{\eta_j}(s)
\ge
\frac{c_Q}{2}
\int_{\mathcal R_{ij,h}}
\min\{
|\Delta_{\eta_i}|,
|\Delta_{\eta_j}|
\}
\,dQ_0.
\]
By assumption,
\(
\int_{\mathcal R_{ij,h}}
\min\{
|\Delta_{\eta_i}|,
|\Delta_{\eta_j}|
\}
\,dQ_0
\ge
c_{\rm rk}h\).
Therefore,
\[
\mathcal E_{\eta_i}(s)
+
\mathcal E_{\eta_j}(s)
\ge
\frac{c_Qc_{\rm rk}}{2}h\quad\Rightarrow \quad \max\{
\mathcal E_{\eta_i}(s),
\mathcal E_{\eta_j}(s)
\}
\ge
\frac{c_Qc_{\rm rk}}{4}h.
\]
Define
\(
\epsilon_h
:=
\frac{c_Qc_{\rm rk}}{8}h\).
Then for all \(i\ne j\),
\[
\inf_s
\max\{
\mathcal E_{\eta_i}(s),
\mathcal E_{\eta_j}(s)
\}
\ge
2\epsilon_h.
\]
So the AUC-space packing has been derived from the Euclidean packing.

Next, we prove the KL divergence control.
Let parameters \(\eta,\eta'\in B_2(\eta_0,r_0)\). Fix
\(y\in\{0,1\}\). Set
\[
\delta:=\eta-\eta',
\qquad
\eta_t:=\eta'+t\delta.
\]
Thus \(\eta_0=\eta'\) and \(\eta_1=\eta\). For fixed \(x\), define the one-dimensional function
\(
g_x(t):=p_{y,\eta_t}(x)\).
By differentiability of \(p_{y,\eta}(x)\) with respect to \(\eta\), the chain rule gives
\[
g_x'(t)
=
\nabla_\eta p_{y,\eta_t}(x)^\top
\frac{d\eta_t}{dt}
=
\nabla_\eta p_{y,\eta_t}(x)^\top \delta.
\]
Therefore, by the fundamental theorem of calculus,
\[
p_{y,\eta}(x)-p_{y,\eta'}(x)
=
g_x(1)-g_x(0)
=
\int_0^1 g_x'(t)\,dt.
\]
Substituting the expression for \(g_x'(t)\), we obtain
\[
p_{y,\eta}(x)-p_{y,\eta'}(x)
=
\int_0^1
\delta^\top \nabla_\eta p_{y,\eta_t}(x)
\,dt.
\]
Since
\(
\nabla_\eta p_{y,\eta_t}(x)
=
p_{y,\eta_t}(x)S_y(x;\eta_t)\),
we have
\[
p_{y,\eta}(x)-p_{y,\eta'}(x)
=
\int_0^1
\delta^\top S_y(x;\eta_t)
p_{y,\eta_t}(x)
\,dt.
\]
By Cauchy--Schwarz,
\[
\left(p_{y,\eta}(x)-p_{y,\eta'}(x)\right)^2
\le
\|\delta\|_2^2
\int_0^1
\|S_y(x;\eta_t)\|_2^2
p_{y,\eta_t}(x)^2
\,dt.
\]
Recall the definition of $\chi^2$ divergence,
\(
\chi^2(P_{y,\eta},P_{y,\eta'})
=
\int
\frac{
(p_{y,\eta}-p_{y,\eta'})^2
}{
p_{y,\eta'}
}
\,d\nu
\).
Therefore
\[
\chi^2(P_{y,\eta},P_{y,\eta'})
\le
\|\delta\|_2^2
\int_0^1
\int
\|S_y(x;\eta_t)\|_2^2
\frac{p_{y,\eta_t}(x)^2}{p_{y,\eta'}(x)}
\,d\nu(x)\,dt.
\]
By local likelihood-ratio comparability that
\(
\frac{p_{y,\eta_t}(x)}
{p_{y,\eta'}(x)}
\le B_{\rm lr}\),
we have
\(
\frac{p_{y,\eta_t}(x)^2}
{p_{y,\eta'}(x)}
\le
B_{\rm lr}p_{y,\eta_t}(x)\).
Thus
\[
\chi^2(P_{y,\eta},P_{y,\eta'})
\le
B_{\rm lr}\|\delta\|_2^2
\int_0^1
\mathbb E_{y,\eta_t}
\|S_y(X;\eta_t)\|_2^2
\,dt.
\]
By assumption,
\[
\chi^2(P_{y,\eta},P_{y,\eta'})
\le
B_{\rm lr}I_y
\|\eta-\eta'\|_2^2.
\]
Since
\(
D_{\rm KL}(P\|Q)\le \chi^2(P,Q)\),
we get
\[
D_{\rm KL}(P_{y,\eta}\|P_{y,\eta'})
\le
B_{\rm lr}I_y
\|\eta-\eta'\|_2^2.
\]
Further, for the class-stratified sample,
\(
\mathbb P_\eta^{(n)}
=
P_{0,\eta}^{\otimes n_0}
\otimes
P_{1,\eta}^{\otimes n_1}\).
Therefore,
\[
\begin{aligned}
D_{\rm KL}
\left(
\mathbb P_\eta^{(n)}
\middle\|
\mathbb P_{\eta'}^{(n)}
\right)
&=
n_0D_{\rm KL}(P_{0,\eta}\|P_{0,\eta'})
+
n_1D_{\rm KL}(P_{1,\eta}\|P_{1,\eta'})
\\
&\le
B_{\rm lr}(n_0I_0+n_1I_1)
\|\eta-\eta'\|_2^2\le
B_{\rm lr}I_{\max}n
\|\eta-\eta'\|_2^2.
\end{aligned}
\]
For packing points,
\(
\|\eta_i-\eta_j\|_2
=
h\|u_i-u_j\|_2
\le 2h\).
Therefore,
\begin{align}\label{eq:KL-upper-minimax}
D_{\rm KL}
\left(
\mathbb P_{\eta_i}^{(n)}
\middle\|
\mathbb P_{\eta_j}^{(n)}
\right)
\le
4B_{\rm lr}I_{\max}nh^2.
\end{align}

Let
\(
h_n
:=
a\sqrt{\frac d n}\),
where \(a>0\) is sufficiently small. Since
\(
\log M\ge c_1d\),
choose \(a\) such that
\[
4B_{\rm lr}I_{\max}nh_n^2
=
4B_{\rm lr}I_{\max}a^2d
\le
\beta\log M
\]
for some fixed \(\beta\in(0,1)\). It is enough to take
\(
a^2
\le
\frac{\beta c_1}{4B_{\rm lr}I_{\max}}\).
For sufficiently large \(n\), we also have
\(
h_n\le r_0\), and 
\(
h_n\le h_0\).
Thus all assumptions apply. Therefore, the AUC packing radius we choose is
\[
\epsilon_n
=
\epsilon_{h_n}
=
\frac{c_Qc_{\rm rk}}{8}h_n
=
c\sqrt{\frac d n}.
\]

To apply Fano's inequality,
let
\(
V\sim \operatorname{Unif}\{1,\ldots,M\}\).
Conditional on \(V=i\), draw the training data from
\(
\mathbb P_{\eta_i}^{(n)}\).
Let \(\hat s\) be any score estimator. Define the decoder
\[
\hat V
\in
\arg\min_{1\le i\le M}
\mathcal E_{\eta_i}(\hat s).
\]
Suppose the true index is \(i\), and consider the event
\(
\{\mathcal E_{\eta_i}(\hat s)<\epsilon_n\}\).
For every \(j\ne i\), the AUC-space separation gives
\(
\max\{
\mathcal E_{\eta_i}(\hat s),
\mathcal E_{\eta_j}(\hat s)
\}
\ge
2\epsilon_n\).
Therefore,
\(
\mathcal E_{\eta_j}(\hat s)\ge 2\epsilon_n\)
for every \(j\ne i\), so the decoder must select \(i\). Hence
\(
\{\hat V\ne i\}
\subseteq
\left\{
\mathcal E_{\eta_i}(\hat s)\ge \epsilon_n
\right\}\).
Averaging over \(i\),
\[
\mathbb P(\hat V\ne V)=\sum_{i=1}^M \mathbb P(V=i)\mathbb P(\hat V\ne V\mid V=i)
\le
\frac1M
\sum_{i=1}^M
\mathbb P_{\eta_i}
\left(
\mathcal E_{\eta_i}(\hat s)\ge \epsilon_n
\right).
\]
Therefore,
\[
\sup_{\eta\in\Theta}
\mathbb P_\eta
\left(
\mathcal E_\eta(\hat s)\ge \epsilon_n
\right)\ge \frac1M
\sum_{i=1}^M
\mathbb P_{\eta_i}
\left(
\mathcal E_{\eta_i}(\hat s)\ge \epsilon_n
\right)
\ge
\mathbb P(\hat V\ne V).
\]
By Fano's inequality (e.g. equation (15.31) in \cite{wainwright2019high}),
\[
\inf_{\hat V}
\mathbb P(\hat V\ne V)
\ge
1-
\frac{I(V;D_n)+\log 2}{\log M}.
\]
For bounding the mutual information,
let
\(
P_i:=\mathbb P_{\eta_i}^{(n)}\), for \( i=1,\ldots,M\).
Conditional on \(V=i\), the training sample \(D_n\) has distribution \(P_i\). Hence the marginal distribution of \(D_n\) is the mixture
\(
\overline P
:=
\frac1M\sum_{j=1}^M P_j\).
By the definition of mutual information,
\(
I(V;D_n)
=
\frac1M\sum_{i=1}^M
D_{\mathrm{KL}}(P_i\|\overline P)\).
Let \(p_i\) and \(\overline p\) denote the corresponding densities with respect to a common dominating measure. Since
\(
\overline p
=
\frac1M\sum_{j=1}^M p_j\),
we have
\[
D_{\mathrm{KL}}(P_i\|\overline P)
=
\int
p_i
\log
\frac{p_i}{\overline p}
\,d\nu
=
\int
p_i
\log
\frac{p_i}{M^{-1}\sum_{j=1}^M p_j}
\,d\nu .
\]
By concavity of the logarithm,
\[
\log\left(\frac1M\sum_{j=1}^M p_j\right)
\ge
\frac1M\sum_{j=1}^M \log p_j .
\]
Therefore,
\[
\log
\frac{p_i}{M^{-1}\sum_{j=1}^M p_j}
\le
\frac1M\sum_{j=1}^M
\log\frac{p_i}{p_j}.
\]
Integrating both sides with respect to \(p_i\,d\nu\), we obtain
\[
D_{\mathrm{KL}}(P_i\|\overline P)
\le
\frac1M\sum_{j=1}^M
D_{\mathrm{KL}}(P_i\|P_j).
\]
Averaging this inequality over \(i=1,\ldots,M\) gives
\[
I(V;D_n)
=
\frac1M\sum_{i=1}^M
D_{\mathrm{KL}}(P_i\|\overline P)
\le
\frac1{M^2}
\sum_{i,j=1}^M
D_{\mathrm{KL}}(P_i\|P_j)=\frac1{M^2}
\sum_{i,j=1}^M
D_{\mathrm{KL}}
\left(
\mathbb P_{\eta_i}^{(n)}
\middle\|
\mathbb P_{\eta_j}^{(n)}
\right).
\]
By the KL control and the choice of \(h_n\),
\begin{align}\label{eq:mutual-info-bound}
I(V;D_n)
\le
\beta\log M.
\end{align}
Thus
\[
\inf_{\hat V}
\mathbb P(\hat V\ne V)
\ge
1-\beta-\frac{\log 2}{\log M}.
\]
We can choose a small $\beta$ such that $\beta\leq 1/4$ and recall that the dimension $d$ is larger than a constant such that we can choose $d\geq 2\log2/c_1$. Then after absorbing fixed-dimensional constants, 
\[
1-\beta-\frac{\log 2}{\log M}\geq 1-\beta-\frac{\log 2}{c_1d}
\ge c_0
\]
for some constant \(c_0>0\). Therefore,
\[
\inf_{\hat s}
\sup_{\eta\in\Theta}
\mathbb P_\eta
\left(
\mathcal E_\eta(\hat s)\ge \epsilon_n
\right)
\ge c_0.
\]
Since
\(
\epsilon_n
=
c\sqrt{\frac d n}\),
we have
\[
\inf_{\hat s}
\sup_{\eta\in\Theta}
\mathbb P_\eta
\left(
\mathcal E_\eta(\hat s)
\ge
c\sqrt{\frac d n}
\right)
\ge c_0.
\]

Finally we convert probability to expectation.
For any nonnegative random variable \(Z\), by Markov's inequality,
\(
\mathbb E Z\ge t\mathbb P(Z\ge t)
\).
Apply this with
\(
Z=\mathcal E_\eta(\hat s)\), and \(
t=\epsilon_n
\), we have
\[
\inf_{\hat s}
\sup_{\eta\in\Theta}
\mathbb E_\eta
\mathcal E_\eta(\hat s)
\ge
c_0\epsilon_n\ge
c\sqrt{\frac d n}.
\]
Finally, we finish the proof by noting the definition of $\mathcal E_\eta(\hat s)$ and the optimality of the likelihood-ratio score.
\end{proof}

\begin{proof}[Proof of Theorem \ref{thm:best-BA-minimax}]
Define the best-threshold balanced-accuracy excess risk
\[
\mathcal E_{\eta}^{\mathrm{BA}}(s)
:=
\sup_{\tau}BA_{\eta}(\Lambda_{\eta},\tau)
-
\sup_{\tau}BA_{\eta}(s,\tau).
\]
By the population optimality of the likelihood-ratio threshold rule for balanced accuracy (equation (\ref{eq:best-BA-event})), proved in Lemma \ref{thm:likelihood-ratio-optimality},
\(
\sup_{\tau}BA_{\eta}(\Lambda_{\eta},\tau)
=
BA_{\eta}(B_{\eta})
\),
where
\(
B_{\eta}=\{x:L_{\eta}(x)\ge 1\}\).
Also, by optimality of the likelihood ratio,
\[
\mathcal E_{\eta}^{\mathrm{BA}}(s)\ge 0, \quad \text{and } \left|
\sup_{\tau}BA_{\eta}(s,\tau)
-
\sup_{\tau}BA_{\eta}(\Lambda_{\eta},\tau)
\right|
=
\mathcal E_{\eta}^{\mathrm{BA}}(s).
\]

We first derive the balanced-accuracy excess representation. For any measurable set
\(S\subseteq\mathcal X\), write
\(
BA_{\eta}(S)
=
\frac12 \mathbb  P_{1,\eta}(S)
+
\frac12 \mathbb  P_{0,\eta}(S^c)\).
Since \(dP_{1,\eta}=L_{\eta}\,dP_{0,\eta}\), we have
\[
BA_{\eta}(S)
=
\frac12
+
\frac12
\int_S
(L_{\eta}(x)-1)
\,dP_{0,\eta}(x),
\]
with the maximizer being
\(
B_{\eta}=\{x:L_{\eta}(x)\ge 1\}\).
Hence, for any measurable set \(S\),
\[
BA_{\eta}(B_{\eta})-BA_{\eta}(S)
=
\frac12
\int_{B_{\eta}\triangle S}
|L_{\eta}(x)-1|
\,dP_{0,\eta}(x).
\]
Therefore, for a score \(s\), writing
\(
S_{\tau}(s):=\{x:s(x)\ge \tau\}\),
we have
\[
\mathcal E_{\eta}^{\mathrm{BA}}(s)
=
\inf_{\tau}
\frac12
\int_{B_{\eta}\triangle S_{\tau}(s)}
|L_{\eta}(x)-1|
\,dP_{0,\eta}(x).
\]

Next, according to the proof of Theorem \ref{thm:AUROC-minimax}, we can construct a Euclidean packing. Namely, for \(\alpha\in(0,1)\) in assumption, there exist points
\(
u_1,\dots,u_M\in \mathbb{S}^{d-1}
\)
such that
\(
\|u_i-u_j\|_2\ge \alpha\), for \( i\ne j
\),
and
\(
\log M\ge c_1d
\)
for some constant \(c_1>0\). For a radius \(h>0\), define
\(
\eta_i:=\eta_0+hu_i\),
\(i=1,\dots,M\).
For \(h\le r_0\), all \(\eta_i\in\Theta\).
We now prove that this parameter packing induces separation in best-threshold balanced accuracy. Fix \(i\ne j\). Write
\[
B_i:=B_{\eta_i},
\qquad
B_j:=B_{\eta_j},
\qquad
L_i:=L_{\eta_i},
\qquad
L_j:=L_{\eta_j}.
\]
Define the directed disagreement regions
\(
D_{ij}^{+}:=B_i\cap B_j^c\),
\(D_{ij}^{-}:=B_i^c\cap B_j\).
By assumption,
\[
\int_{D_{ij}^{+}}
\min\{|L_i-1|,|L_j-1|\}
\,dP_{0,\eta_0}
\ge c_{\mathrm{BA}}h,\quad 
\int_{D_{ij}^{-}}
\min\{|L_i-1|,|L_j-1|\}
\,dP_{0,\eta_0}
\ge c_{\mathrm{BA}}h.
\]

Let \(s\) be any measurable score. Because best-threshold balanced accuracy allows parameter-specific thresholds, take arbitrary thresholds \(\tau_i,\tau_j\) and define
\[
S_i:=S_{\tau_i}(s)=\{x:s(x)\ge \tau_i\},
\qquad
S_j:=S_{\tau_j}(s)=\{x:s(x)\ge \tau_j\}.
\]
Since these are upper level sets of the same score \(s\), they are nested. Thus either
\(
S_i\subseteq S_j
\)
or
\(
S_j\subseteq S_i\).
First suppose \(S_i\subseteq S_j\). On the region
\(
D_{ij}^{+}=B_i\cap B_j^c\),
the oracle decision under \(\eta_i\) is positive, while the oracle decision under \(\eta_j\) is negative. Because \(S_i\subseteq S_j\), the two threshold sets \(S_i,S_j\) cannot match these opposite decisions simultaneously. Indeed, for every \(x\in D_{ij}^{+}\),
\[
\mathbf 1\{x\in B_i\triangle S_i\}
+
\mathbf 1\{x\in B_j\triangle S_j\}
\ge 1.
\]
For the integral, we have
\[
\begin{aligned}
&\frac12 \int_{B_i\triangle S_i} |L_i-1|\,dP_{0,\eta_i}
+
\frac12 \int_{B_j\triangle S_j} |L_j-1|\,dP_{0,\eta_j}
\\
&\qquad =
\frac12 \int |L_i-1|
\mathbf 1\{x\in B_i\triangle S_i\}
\,dP_{0,\eta_i}
+
\frac12 \int |L_j-1|
\mathbf 1\{x\in B_j\triangle S_j\}
\,dP_{0,\eta_j}.
\end{aligned}
\]
The two integrands are nonnegative. Hence restricting both integrals to the
smaller set \(D_{ij}^{+}\) can only decrease their values, and therefore
\[
\begin{aligned}
&\frac12 \int_{B_i\triangle S_i} |L_i-1|\,dP_{0,\eta_i}
+
\frac12 \int_{B_j\triangle S_j} |L_j-1|\,dP_{0,\eta_j}
\\
&\qquad \ge
\frac12 \int_{D_{ij}^{+}}
|L_i-1|
\mathbf 1\{x\in B_i\triangle S_i\}
\,dP_{0,\eta_i} +
\frac12 \int_{D_{ij}^{+}}
|L_j-1|
\mathbf 1\{x\in B_j\triangle S_j\}
\,dP_{0,\eta_j}.
\end{aligned}
\]
We now use the local likelihood-ratio comparability condition. Since
\(\eta_i\) and \(\eta_j\) lie in the local neighborhood of
\(\eta_0\), assumption implies
\(
\frac{dP_{0,\eta_i}}{dP_{0,\eta_0}}(x)\ge B_{\mathrm{lr}}^{-1}\),
\(\frac{dP_{0,\eta_j}}{dP_{0,\eta_0}}(x)\ge B_{\mathrm{lr}}^{-1}\),
or equivalently,
\(
dP_{0,\eta_i}\ge B_{\mathrm{lr}}^{-1}dP_{0,\eta_0}\),
\(dP_{0,\eta_j}\ge B_{\mathrm{lr}}^{-1}dP_{0,\eta_0}\).
Therefore, 
\[
\begin{aligned}
&\frac12 \int_{D_{ij}^{+}}
|L_i-1|
\mathbf 1\{x\in B_i\triangle S_i\}
\,dP_{0,\eta_i}\ge
\frac{1}{2B_{\mathrm{lr}}}
\int_{D_{ij}^{+}}
|L_i-1|
\mathbf 1\{x\in B_i\triangle S_i\}
\,dP_{0,\eta_0},
\end{aligned}
\]
and similarly
\[
\begin{aligned}
&\frac12 \int_{D_{ij}^{+}}
|L_j-1|
\mathbf 1\{x\in B_j\triangle S_j\}
\,dP_{0,\eta_j}\ge
\frac{1}{2B_{\mathrm{lr}}}
\int_{D_{ij}^{+}}
|L_j-1|
\mathbf 1\{x\in B_j\triangle S_j\}
\,dP_{0,\eta_0}.
\end{aligned}
\]
Combining the two bounds yields
\[
\begin{aligned}
&\frac12 \int_{B_i\triangle S_i} |L_i-1|\,dP_{0,\eta_i}
+
\frac12 \int_{B_j\triangle S_j} |L_j-1|\,dP_{0,\eta_j}
\\
&\qquad\ge
\frac{1}{2B_{\mathrm{lr}}}
\int_{D_{ij}^{+}}
\left[
|L_i-1|\mathbf 1\{x\in B_i\triangle S_i\}
+
|L_j-1|\mathbf 1\{x\in B_j\triangle S_j\}
\right]
\,dP_{0,\eta_0}.
\end{aligned}
\]
Since at least one of the two disagreement indicators is equal to one on \(D_{ij}^{+}\), by assumption,
\[
\begin{aligned}
&\frac12\int_{B_i\triangle S_i}|L_i-1|\,dP_{0,\eta_i}
+
\frac12\int_{B_j\triangle S_j}|L_j-1|\,dP_{0,\eta_j}
\\
&\qquad\ge
\frac{1}{2B_{\mathrm{lr}}}
\int_{D_{ij}^{+}}
\min\{|L_i-1|,|L_j-1|\}
\,dP_{0,\eta_0}\ge
\frac{c_{\mathrm{BA}}}{2B_{\mathrm{lr}}}h.
\end{aligned}
\]

The other case is symmetric. If \(S_j\subseteq S_i\), then on
\(
D_{ij}^{-}=B_i^c\cap B_j\),
the oracle decision under \(\eta_i\) is negative, while the oracle decision under \(\eta_j\) is positive. Again, the nested sets \(S_i,S_j\) cannot match both oracle decisions simultaneously, so the same argument gives
\[
\frac12\int_{B_i\triangle S_i}|L_i-1|\,dP_{0,\eta_i}
+
\frac12\int_{B_j\triangle S_j}|L_j-1|\,dP_{0,\eta_j}
\ge
\frac{c_{\mathrm{BA}}}{2B_{\mathrm{lr}}}h.
\]
Because \(\tau_i,\tau_j\) were arbitrary, taking the infimum over thresholds gives
\[
\mathcal E_{\eta_i}^{\mathrm{BA}}(s)
+
\mathcal E_{\eta_j}^{\mathrm{BA}}(s)
\ge
\frac{c_{\mathrm{BA}}}{2B_{\mathrm{lr}}}h\quad \Rightarrow \quad \max\left\{
\mathcal E_{\eta_i}^{\mathrm{BA}}(s),
\mathcal E_{\eta_j}^{\mathrm{BA}}(s)
\right\}
\ge
\frac{c_{\mathrm{BA}}}{4B_{\mathrm{lr}}}h.
\]
Define
\(
\epsilon_h
:=
\frac{c_{\mathrm{BA}}}{8B_{\mathrm{lr}}}h\).
Then for every \(i\ne j\),
\[
\inf_s
\max\left\{
\mathcal E_{\eta_i}^{\mathrm{BA}}(s),
\mathcal E_{\eta_j}^{\mathrm{BA}}(s)
\right\}
\ge
2\epsilon_h.
\]

Now we use the same KL and Fano argument as in the proof of Theorem \ref{thm:AUROC-minimax}. By the KL-control part of Theorem \ref{thm:AUROC-minimax} in equation (\ref{eq:KL-upper-minimax}), for every \(i,j\),
\[
D_{\mathrm{KL}}
\left(
P_{\eta_i}^{(n)}
\,\middle\|\,
P_{\eta_j}^{(n)}
\right)
\le
4B_{\mathrm{lr}}I_{\max}nh^2.
\]
Choose
\(
h_n:=a\sqrt{\frac dn}\),
where \(a>0\) is sufficiently small so that
\[
4B_{\mathrm{lr}}I_{\max}nh_n^2
\le
\beta\log M
\]
for some fixed \(\beta\in(0,1)\). Since \(\log M\ge c_1d\), it is enough to choose
\(
a^2\le \frac{\beta c_1}{4B_{\mathrm{lr}}I_{\max}}\).
For all sufficiently large \(n\), we also have
\(
h_n\le r_0\),
\(h_n\le h_0
\).
Thus all local assumptions apply.
Set
\(
\epsilon_n
:=
\epsilon_{h_n}
=
\frac{c_{\mathrm{BA}}}{8B_{\mathrm{lr}}}h_n
=
c\sqrt{\frac dn}\).
Let \(V\sim \mathrm{Unif}\{1,\dots,M\}\). Conditional on \(V=i\), draw the training data from
\(
P_{\eta_i}^{(n)}\).
Given any score estimator \(\hat s\), define the decoder
\[
\hat V
\in
\arg\min_{1\le i\le M}
\mathcal E_{\eta_i}^{\mathrm{BA}}(\hat s).
\]
Suppose \(V=i\) and consider when
\(
\mathcal E_{\eta_i}^{\mathrm{BA}}(\hat s)<\epsilon_n\).
For every \(j\ne i\), the balanced-accuracy separation gives
\(
\max\left\{
\mathcal E_{\eta_i}^{\mathrm{BA}}(\hat s),
\mathcal E_{\eta_j}^{\mathrm{BA}}(\hat s)
\right\}
\ge
2\epsilon_n\).
Therefore,
\[
\mathcal E_{\eta_j}^{\mathrm{BA}}(\hat s)
\ge
2\epsilon_n
>
\mathcal E_{\eta_i}^{\mathrm{BA}}(\hat s),
\]
so the decoder must select \(i\). Hence
\(
\{\hat V\ne V\}
\subseteq
\left\{
\mathcal E_{\eta_V}^{\mathrm{BA}}(\hat s)\ge \epsilon_n
\right\}\).
Averaging over \(V\),
\[
\sup_{\eta\in\Theta}
\mathbb  P_{\eta}
\left(
\mathcal E_{\eta}^{\mathrm{BA}}(\hat s)\ge \epsilon_n
\right)
\ge
\mathbb  P(\hat V\ne V).
\]

By the mutual-information bound in equation (\ref{eq:mutual-info-bound}),
\(
I(V;D_n)
\le
\beta\log M\).
Therefore, by Fano's inequality (e.g. equation (15.31) in \cite{wainwright2019high}),
\[
\inf_{\hat V}P(\hat V\ne V)
\ge
1-\frac{I(V;D_n)+\log 2}{\log M}
\ge
1-\beta-\frac{\log 2}{\log M}.
\]
Since \(\log M\ge c_1d\), for sufficiently large \(d\) and sufficiently small fixed \(\beta\), there exists a constant \(c_0>0\) such that
\(
1-\beta-\frac{\log 2}{\log M}
\ge c_0\).
Thus
\[
\inf_{\hat s}
\sup_{\eta\in\Theta}
\mathbb  P_{\eta}
\left(
\mathcal E_{\eta}^{\mathrm{BA}}(\hat s)\ge \epsilon_n
\right)
\ge c_0.
\]

Finally, converting this probability lower bound into an expectation lower bound, for any nonnegative random variable \(Z\), by Markov inequality,
\(
\mathbb E Z\ge t\mathbb  P(Z\ge t)\).
Applying this with
\(
Z=\mathcal E_{\eta}^{\mathrm{BA}}(\hat s)\) and
\(
t=\epsilon_n\),
we obtain
\[
\inf_{\hat s}
\sup_{\eta\in\Theta}
\mathbb E_{\eta}
\mathcal E_{\eta}^{\mathrm{BA}}(\hat s)
\ge
c_0\epsilon_n
\ge
c\sqrt{\frac dn}.
\]
Combining with the definition of $\mathcal E_{\eta}^{\mathrm{BA}}(\hat s)$ and the optimality of the likelihood ratio score, we finish the proof.
\end{proof}

\begin{proof}[Proof of Theorem \ref{thm:misspecified-auc}] We prove the results for AUROC and AUPRC separately.

\medskip
\noindent\textbf{Proof for AUROC:}

We first show that high AUROC requires majority score variation. For any score \(s\in\mathcal S\), define
\(
V_0(s)=\mathbb E_{P_0}\{s(X)-u_0\}^2\).
We first prove that
\[
\operatorname{AUC}(s)
\le
\frac12+B_s(1+\sqrt{B_L})\sqrt{V_0(s)},
\]
If \(s\) is constant, then all points are tied and thus
\(
\operatorname{AUC}(s_0)=\frac12
\). Thus the inequality holds.
Now consider any non-constant score \(s\). Let \(F_{0,s}\) and \(F_{1,s}\) denote the distributions of \(s(X)\) under \(X\sim P_0\) and \(X\sim P_1\), respectively. Let
\[
U_0=s(X_0)\sim F_{0,s},
\qquad
U_1=s(X_1)\sim F_{1,s}.
\]
If \(F_{0,s}\) is continuous, then ties occur with probability zero. Hence
\(
\operatorname{AUC}(s)
=
\mathbb P\{U_1>U_0\}\).
Conditioning on \(U_1\), we obtain
\[
\mathbb P\{U_1>U_0\mid U_1\}
=
\mathbb P\{U_0<U_1\mid U_1\}
=
F_{0,s}(U_1),
\]
where the last equality uses the continuity of \(F_{0,s}\). Therefore,
\[
\operatorname{AUC}(s)
=
\mathbb E\{F_{0,s}(U_1)\}.
\]
Similarly, since \(U_0\) has CDF $F_{0,s}$, the probability integral transform implies that
\(
F_{0,s}(U_0)\sim \operatorname{Unif}(0,1)\).
Consequently,
\(
\mathbb E\{F_{0,s}(U_0)\}=\frac12\).
Therefore, since \(F_{0,s}\) is \(B_s\)-Lipschitz, the Kantorovich-Rubinstein inequality gives
\begin{align*}
\operatorname{AUC}(s)-\frac12
=
\mathbb E_{F_{1,s}}F_{0,s}(U)
-
\mathbb E_{F_{0,s}}F_{0,s}(U)\le
B_s W_1(F_{1,s},F_{0,s})\\\le
B_s(W_1(F_{1,s},\delta_{u_0})
+
W_1(F_{0,s},\delta_{u_0})),
\end{align*}
where \(\delta_{u_0}\) is the point mass at \(u_0\). 
Recall the definition of Wasserstein distance,  for probability measures \(\mu\) and \(\nu\) on \(\mathbb R\),
\(
W_1(\mu,\nu)
=
\inf_{\gamma\in\Pi(\mu,\nu)}
\int |u-v|\,d\gamma(u,v)\),
where \(\Pi(\mu,\nu)\) denotes the set of couplings of \(\mu\) and \(\nu\). If \(\nu=\delta_{u_0}\) is the point mass at \(u_0\), then any coupling between \(\mu\) and \(\delta_{u_0}\) must have second coordinate equal to \(u_0\) almost surely. Hence
\[
W_1(\mu,\delta_{u_0})
=
\int |u-u_0|\,d\mu(u).
\]
Applying this identity with \(\mu=F_{1,s}\), the distribution of \(s(X)\) under \(P_1\), gives
\[
W_1(F_{1,s},\delta_{u_0})
=
\int |u-u_0|\,dF_{1,s}(u)
=
\mathbb E_{P_1}|s(X)-u_0|,\quad W_1(F_{0,s},\delta_{u_0})
=
\mathbb E_{P_0}|s(X)-u_0|.
\]
Therefore, by the triangle inequality for \(W_1\),
\[
W_1(F_{1,s},F_{0,s})
\le
W_1(F_{1,s},\delta_{u_0})
+
W_1(F_{0,s},\delta_{u_0})
\le
\mathbb E_{P_1}|s(X)-u_0|
+
\mathbb E_{P_0}|s(X)-u_0|.
\]
By Cauchy-Schwarz,
\[
\mathbb E_{P_0}|s(X)-u_0|
\le
\sqrt{V_0(s)}.
\]
Since \(dP_1=L\,dP_0\) and \(L\le B_L\),
\[
\mathbb E_{P_1}\{s(X)-u_0\}^2
=
\mathbb E_{P_0}\left[
L(X)\{s(X)-u_0\}^2
\right]
\le
B_L V_0(s).
\]
Therefore,
\[
\mathbb E_{P_1}|s(X)-u_0|
\le
\sqrt{B_LV_0(s)}.
\]
Combining the previous bounds,
\begin{equation}\label{eq:wasserstein-two-dist}
W_1(F_{1,s},F_{0,s})
\le
(1+\sqrt{B_L})\sqrt{V_0(s)}.
\end{equation}
Thus
\[
\operatorname{AUC}(s)
\le
\frac12+B_s(1+\sqrt{B_L})\sqrt{V_0(s)}.
\]
Consequently, for any specified AUROC level $A$, if
\(
\operatorname{AUC}(s)\ge A\),
then
\(
A-\frac12
\le
B_s(1+\sqrt{B_L})\sqrt{V_0(s)}
\).
Hence
\[
V_0(s)
\ge
\left(\frac{A-\frac12}{B_s(1+\sqrt{B_L})}\right)^2
=:
\kappa_A.
\]

Next, we show that high AUROC requires a majority risk cost. By assumption,
\[
\ell(s(x),0)-\ell(u_0,0)
\ge
m_0\{s(x)-u_0\}^2.
\]
Taking expectation under \(P_0\),
\begin{equation}\label{eq:R0-R0-const-geqV0}
\cR_0(s)-\cR_0(s_0)
\ge
m_0V_0(s).
\end{equation}
Since \(s_0(x)\equiv u_0\) minimizes the pointwise majority-class loss, it also minimizes \(\cR_0\) over \(\mathcal S\). 
Then, for any AUROC level $A\in\left(\frac{1}{2},\AUC(\Lambda_{\eta^*})\right)$, by previous analysis, if
\(
\operatorname{AUC}(s)\ge A\),
it holds that
\(
V_0(s)\ge \kappa_A\).
Thus
\[
\operatorname{AUC}(s)\ge A
\quad\Longrightarrow\quad
\cR_0(s)-\cR_0(s_0)\ge m_0\kappa_A.
\]
So any score with AUROC at least \(A\) must pay a majority-class risk cost of at least \(m_0\kappa_A\).

Next, we show that the raw imbalanced population minimizer cannot pay this majority risk cost. Since \(s_{\rm raw}^*\) minimizes \(\cR_{\alpha_{\rm raw}}\), that is
\(
\cR_{\alpha_{\rm raw}}(s_{\rm raw}^*)
\le
\cR_{\alpha_{\rm raw}}(s_0)\).
Expanding and rearranging,
\[
(1-\alpha_{\rm raw})
\{\cR_0(s_{\rm raw}^*)-\cR_0(s_0)\}
\le
\alpha_{\rm raw}
\{\cR_1(s_0)-\cR_1(s_{\rm raw}^*)\}.
\]
By the definition of \(D_1\),
\(
\cR_1(s_0)-\cR_1(s_{\rm raw}^*)\le D_1\).
Therefore,
\begin{equation}\label{eq:R-0-alpha-raw-D-1}
\cR_0(s_{\rm raw}^*)-\cR_0(s_0)
\le
\frac{\alpha_{\rm raw}}{1-\alpha_{\rm raw}}D_1.
\end{equation}
By the severe-imbalance assumption,
\(
\frac{\alpha_{\rm raw}}{1-\alpha_{\rm raw}}D_1
<
m_0\kappa_A\), we have
\(
\cR_0(s_{\rm raw}^*)-\cR_0^*
<
m_0\kappa_A\).
If, contrary to the desired conclusion,
\(
\operatorname{AUC}(s_{\rm raw}^*)\ge A\),
then previous analysis would imply
\(
\cR_0(s_{\rm raw}^*)-\cR_0^*
\ge
m_0\kappa_A\),
which contradicts the previous strict inequality.
Therefore,
\[
\operatorname{AUC}(s_{\rm raw}^*)<A.
\]

Let
\(
M^*=(1-\alpha^*)P_0+\alpha^* P_1\) 
denote the augmented feature distribution. Under the augmented population, the conditional probability of class \(1\) given \(X=x\) is
\[
\eta_{\rm \alpha^*}(x)
=
\frac{\alpha^* p_1(x)}{(1-\alpha^*)p_0(x)+\alpha^* p_1(x)}
=
g_{\alpha^*}(\Lambda_{\eta^*}(x)),
\]
where $g_{\alpha^*}(\cdot)$ is a strictly increasing function.
For a fixed conditional class probability \(\eta\), write
\[
C_\eta(u)=\eta \ell(u,1)+(1-\eta)\ell(u,0).
\]
By strict propriety and quadratic calibration, for all \(u,\eta\in[0,1]\),
\(
C_\eta(u)-C_\eta(\eta)
\ge
m_\ell (u-\eta)^2\).
Applying this inequality pointwise with \(\eta=\eta_{\alpha^*}(x)\) and \(u=s(x)\), we obtain
\[
C_{\eta_{\alpha^*}(x)}(s(x))-C_{\eta_{\alpha^*}(x)}(\eta_{\alpha^*}(x))
\ge
m_\ell\{s(x)-\eta_{\alpha^*}(x)\}^2.
\]
Taking expectation with respect to \(X\sim M^*\) gives
\[
\mathbb E_{M^*}\!\left[
C_{\eta_{\alpha^*}(X)}(s(X))-C_{\eta_{\alpha^*}(X)}(\eta_{\alpha^*}(X))
\right]
\ge
m_\ell \mathbb E_{M^*}\{s(X)-\eta_{\alpha^*}(X)\}^2.
\]
By direct calculation, we can write in expectation that
\(
R_{\alpha^*}(s)=\mathbb E_{M^*} C_{\eta_{\alpha^*}(X)}(s(X))\), and 
\(R_{\alpha^*}(\eta_{\alpha^*})=\mathbb E_{M^*} C_{\eta_{\alpha^*}(X)}(\eta_{\alpha^*}(X))\).
We then conclude that
\[
R_{\alpha^*}(s)-R_{\alpha^*}(\eta_{\alpha^*})
\ge
m_\ell \mathbb E_{M^*}\{s(X)-\eta_{\alpha^*}(X)\}^2.
\]
Since \(s_{\alpha^*}\) minimizes \(R_{\alpha^*}\) over \(\mathcal S\), by assumption
\(
R_{\alpha^*}(s_{\alpha^*})-R_{\alpha^*}(\eta_{\alpha^*})
\le
\epsilon^*\).
Therefore,
\begin{equation}\label{eq:sbal-Lambda-diff}
\mathbb E_{M^*}\{s_{\alpha^*}(X)-\eta_{\alpha^*}(X)\}^2
\le
\frac{\epsilon^*}{m_\ell}.
\end{equation}
Define
\(
e(x)=s_{\alpha^*}(x)-\eta_{\alpha^*}(x)\).
Recall the definition of AUROC,
\[
\operatorname{AUC}(s)
=
\mathbb E\left[
\mathbf 1\{s(X_1)>s(X_0)\}
+
\frac12\mathbf 1\{s(X_1)=s(X_0)\}
\right].
\]
Hence
\[
\begin{aligned}
\operatorname{AUC}(\eta_{\alpha^*})-\operatorname{AUC}(s_{\alpha^*})\le
\mathbb P\left(
\eta_{\alpha^*}(X_1)\ge \eta_{\alpha^*}(X_0),\,
s_{\alpha^*}(X_1)\le s_{\alpha^*}(X_0)
\right).
\end{aligned}
\]
On the event
\(
\left\{
\eta_{\alpha^*}(X_1)\ge \eta_{\alpha^*}(X_0),\,
s_{\alpha^*}(X_1)\le s_{\alpha^*}(X_0)
\right\}\),
we have
\[
\eta_{\alpha^*}(X_1)+e(X_1)
=
s_{\alpha^*}(X_1)
\le
s_{\alpha^*}(X_0)
=
\eta_{\alpha^*}(X_0)+e(X_0).
\]
Therefore,
\[
\eta_{\alpha^*}(X_1)-\eta_{\alpha^*}(X_0)
\le
e(X_0)-e(X_1)
\le
|e(X_1)|+|e(X_0)|.
\]
Consequently, for any \(t>0\),
\[
\begin{aligned}
&\left\{
\eta_{\alpha^*}(X_1)\ge \eta_{\alpha^*}(X_0),\,
s_{\alpha^*}(X_1)\le s_{\alpha^*}(X_0)
\right\} \\
&\subseteq
\left\{
|\eta_{\alpha^*}(X_1)-\eta_{\alpha^*}(X_0)|\le t
\right\}
\cup
\left\{
|e(X_1)|+|e(X_0)|>t
\right\}.
\end{aligned}
\]
Taking probabilities and applying the union bound gives
\begin{equation}\label{eq:AUC-two-prob-upper}
\begin{aligned}
\operatorname{AUC}(\eta_{\alpha^*})-\operatorname{AUC}(s_{\alpha^*})
\le\;&
\mathbb P\left(
|\eta_{\alpha^*}(X_1)-\eta_{\alpha^*}(X_0)|\le t
\right) +
\mathbb P\left(
|e(X_1)|+|e(X_0)|>t
\right).
\end{aligned}
\end{equation}
By the pairwise margin condition,
\(
\mathbb P\left(
|\Lambda_{\eta^*}(X_1)-\Lambda_{\eta^*}(X_0)|\le t
\right)
\le
C_{\rm pair}t\).
Since $g_{\alpha^*}$ is strictly increasing and has derivative bounded below by some constant $c_{\alpha^*}>0$, it holds that
\begin{align*}
    |\eta_{\alpha^*}(X_1)-\eta_{\alpha^*}(X_0)|\geq c_{\alpha^*}|\Lambda_{\eta^*}(X_1)-\Lambda_{\eta^*}(X_0)|.
\end{align*}
Thus
\[
\mathbb P\left(
|\eta_{\alpha^*}(X_1)-\eta_{\alpha^*}(X_0)|\le t
\right)
\le
Ct.
\]
By Markov's inequality,
\[
\mathbb P\left(
|e(X_1)|+|e(X_0)|>t
\right)
\le
\frac{\mathbb E(|e(X_1)|+|e(X_0)|)^2}{t^2}\leq \frac{2\mathbb E_{P_1}e(X)^2
+
2\mathbb E_{P_0}e(X)^2}{t^2}.
\]
Since
\(
\mathbb E_{M^*}e(X)^2
=
(1-\alpha^*)\mathbb E_{P_0}e(X)^2
+
\alpha^*\mathbb E_{P_1}e(X)^2\),
we get
\[
2\mathbb E_{P_1}e(X)^2
+
2\mathbb E_{P_0}e(X)^2
\leq
\max\left\{\frac{2}{\alpha^*},\frac{2}{1-\alpha^*}\right\}\mathbb E_{M^*}e(X)^2
\le
\frac{2}{\alpha^*(1-\alpha^*)}\frac{\epsilon^*}{m_\ell},
\]
where $C_{\alpha^*}$ is constant determined by $\alpha^*$.
Therefore,
\[
\operatorname{AUC}(\eta_{\alpha^*})-\operatorname{AUC}(s_{\alpha^*})
\le
Ct
+
\frac{2}{\alpha^*(1-\alpha^*)}\frac{\epsilon^*}{m_\ell t^2}.
\]
Optimizing over \(t\) yields
\begin{equation}\label{eq:AUC-epbal-upper}
\operatorname{AUC}(\eta_{\alpha^*})-\operatorname{AUC}(s_{\alpha^*})
\le
C
\left(\frac{\epsilon^*}{m_\ell}\right)^{1/3}.
\end{equation}
Combining that $g_{\alpha^*}$ is strictly increasing, and
combining this with previous analysis at the AUROC level $A$, we have
\[
\operatorname{AUC}(s_{\rm raw}^*)<A,\quad \operatorname{AUC}(s_{\alpha^*})
\ge
\operatorname{AUC}(\Lambda_{\eta^*})
-
C
\left(\frac{\epsilon^*}{m_\ell}\right)^{1/3}.
\]
That is, 
\[
\operatorname{AUC}(s_{\alpha^*})-\operatorname{AUC}(s_{\rm raw}^*)
\ge
\operatorname{AUC}(\Lambda_{\eta^*})-A
-
C
\left(\frac{\epsilon^*}{m_\ell}\right)^{1/3}.
\]
We thus have derived the result at the population level.

Next, for finite-sample estimators, by essentially the same procedure with equations (\ref{eq:raw-auc-closeness}) and (\ref{eq:aug-auc-closenesss}), but with the simultaneous statement in Lemma \ref{lemma:empiricalestimation}, the following inequalities hold simultaneously with probability at least $1-\delta$ that
\begin{align*}
    \left|\AUC(s_{\hat{\theta}_{\raw}})-\AUC(s_{\raw}^*)\right|\leq \frac{CLc}{\lambda}\sqrt{\frac{B^2d\log(14d/\delta)}{n_0+n_1}},
\end{align*}
and 
\begin{align*}\left|\AUC(s_{\hat{\theta}_{\aug}})-\AUC(s_{\alpha^*})\right|\leq \frac{CLc}{\lambda}\sqrt{\frac{B^2d\log(14d/\delta)}{n_0+n_1+\tilde{n}}}+CL_gL\sqrt{\frac{2}{\mu L_{\rho}}}\frac{\tilde{n}}{n_0+n_1+\tilde{n}}\epsilon_{\syn},
    \end{align*}
The synthetic size is taken as $\tilde n=\frac{\alpha^*}{1-\alpha^*}n_0-n_1$. If this is not an integer, we can simply add an $O(1/n_0)$ term, which is small. This gives
\begin{align*}\left|\AUC(s_{\hat{\theta}_{\aug}})-\AUC(s_{\alpha^*})\right|\leq \frac{CLc}{\lambda}\sqrt{\frac{B^2d\log(14d/\delta)}{n_0/(1-\alpha^*)}}+CL_gL\sqrt{\frac{1}{2\mu L_{\rho}}}\left(\alpha^*-(1-\alpha^*)\frac{n_1}{n_0}\right)\epsilon_{\syn},
    \end{align*}
Combining and absorbing constants into $c$ gives the final bound for AUROC.

\medskip
\noindent\textbf{Proof for AUPRC:}

We first prove that high AUPRC forces nontrivial variation under $P_0$. For constant scores, we adopt the standard no-ranking convention under which $\AUPRC(s)=\pi_1$. Since the AUPRC part considers $A>\pi_1$, constant scores are naturally excluded from the subsequent high-AUPRC argument.
Write
\[
    G_{0,s}(t)=P_0(s(X)\ge t),
    \qquad
    G_{1,s}(t)=P_1(s(X)\ge t).
\]
For a recall level $r\in[0,1]$, let $t_r$ be a threshold satisfying
\(
    G_{1,s}(t_r)=r\).
Since \(F_{0,s}\) is \(B_s\)-Lipschitz, it follows that \(F_{1,s}\) is \(B_LB_s\)-Lipschitz and hence continuous. Therefore the recall parametrization is well defined.
Define the false-positive rate at recall $r$ by
\(
    q_s(r)=G_{0,s}(t_r)\).
Then the AUPRC of $s$ can be written as
\[
    \operatorname{AUPRC}(s)
    =
    \int_0^1
    \frac{\pi_1 r}
         {\pi_1 r+\pi_0 q_s(r)}
    \,dr .
\]
If $F_{0,s}=F_{1,s}$, then $q_s(r)=r$, and hence the precision is equal to
$\pi_1$ for every $r>0$. Thus $\pi_1$ is the no-ranking information baseline for AUPRC.
For $r\in[0,1]$, 
\[
    \operatorname{AUPRC}(s)-\pi_1
    =
    \int_0^1
    \{\phi_r(q_s(r))-\phi_r(r)\}
    \,dr,
\]
where $    \phi_r(z)
    =
    \frac{\pi_1 r}{\pi_1 r+\pi_0 z}$.
By equation (\ref{eq:phi-diff-upper-z}), for $z,z'\in[0,1]$,
\(
    |\phi_r(z)-\phi_r(z')|
    \le
    \min\left\{
        1,
        \frac{\pi_0}{\pi_1}
        \frac{|z-z'|}{r}
    \right\}\).
Using the same standard splitting argument for the singularity at $r=0$ as in equation (\ref{eq:9}), we get
\[
    |\operatorname{AUPRC}(s)-\pi_1|
    \le
    2
    \sqrt{
        \frac{\pi_0}{\pi_1}
        \int_0^1 |q_s(r)-r|\,dr}.
\]

We now bound the integrated false-positive-rate discrepancy. 
Since
$q_s(r)=G_{0,s}(t_r)$ and $r=G_{1,s}(t_r)$, by a change-of-variable statement,
\[
    \int_0^1 |q_s(r)-r|\,dr
    =
    \int |G_{0,s}(t)-G_{1,s}(t)|\,dF_{1,s}(t).
\]
Since
$L_{\eta^*}\le B_L$ $P_0$-a.s., the score distribution under $P_1$ is
dominated by the score distribution under $P_0$ such that
\(
    dF_{1,s} \le B_L\,dF_{0,s}\).
Moreover, since $F_{0,s}$ is $B_s$-Lipschitz,
\(
    dF_{0,s}(t)\le B_s\,dt\).
Therefore
\(
    dF_{1,s}(t)\le B_LB_s\,dt\).
It then follows that
\[
\begin{aligned}
    \int_0^1 |q_s(r)-r|\,dr
    &\le
    B_LB_s
    \int |F_{1,s}(t)-F_{0,s}(t)|\,dt  =
    B_LB_s\, W_1(F_{1,s},F_{0,s}) .
\end{aligned}
\]
Here the last equality uses the one-dimensional representation of the
Wasserstein-$1$ distance.
Combining equation (\ref{eq:wasserstein-two-dist}),
\[
    \int_0^1 |q_s(r)-r|\,dr
    \le
    B_LB_s(1+\sqrt{B_L})\sqrt{V_0(s)}.
\]
Plugging this into the AUPRC stability bound yields
\[
    \operatorname{AUPRC}(s)
    \le
    \pi_1
    +
    2
    \left\{
        \frac{\pi_0}{\pi_1}
        B_LB_s(1+\sqrt{B_L})
    \right\}^{1/2}
    V_0(s)^{1/4}.
\]
Therefore, if $\operatorname{AUPRC}(s)\ge A>\pi_1$, then
\[
    V_0(s)
    \ge
    \left(
        \frac{A-\pi_1}{2
    \left\{
        \frac{\pi_0}{\pi_1}
        B_LB_s(1+\sqrt{B_L})
    \right\}^{1/2}}
    \right)^4
    =:
    \kappa_A^{\mathrm{PR}} .
\]
Next, by equation (\ref{eq:R0-R0-const-geqV0}),
\(
    \cR_0(s)-\cR_0(s_0)
\ge
m_0V_0(s)\),
where $s_0(x)\equiv u_0$. Therefore,
\[
\operatorname{AUPRC}(s)\ge A
\quad\Longrightarrow\quad
\cR_0(s)-\cR_0^*\ge m_0\kappa_A^{\mathrm{PR}}.
\]

Now consider the raw population minimizer $s_{\mathrm{raw}}^*$.
We have
\(
    \cR_{\alpha_{\mathrm{raw}}}(s_{\mathrm{raw}}^*)
    \le
    \cR_{\alpha_{\mathrm{raw}}}(s_0)\).
Equivalently,
\[
\begin{aligned}
(1-\alpha_{\mathrm{raw}})\cR_0(s_{\mathrm{raw}}^*)
+
\alpha_{\mathrm{raw}}\cR_1(s_{\mathrm{raw}}^*)\le
(1-\alpha_{\mathrm{raw}})\cR_0(s_0)
+
\alpha_{\mathrm{raw}}\cR_1(s_0).
\end{aligned}
\]
Rearranging and by the definition of $D_1$,
\[
    \cR_0(s_{\mathrm{raw}}^*)-\cR_0^*
    \le
    \frac{\alpha_{\mathrm{raw}}}{1-\alpha_{\mathrm{raw}}}D_1.
\]
Under the stated severe-imbalance condition,
\(
    \frac{\alpha_{\mathrm{raw}}}{1-\alpha_{\mathrm{raw}}}D_1
    <
    m_0\kappa_A^{\mathrm{PR}}\).
Thus
\(
    \cR_0(s_{\mathrm{raw}}^*)-\cR_0^*
    <
    m_0\kappa_A^{\mathrm{PR}}\).
If $\operatorname{AUPRC}(s_{\mathrm{raw}}^*)\ge A$, then the previous
analysis would force
\(
    \cR_0(s_{\mathrm{raw}}^*)-\cR_0^*
    \ge
    m_0\kappa_A^{\mathrm{PR}}\),
which is a contradiction. Hence
\[
    \operatorname{AUPRC}(s_{\mathrm{raw}}^*)<A .
\]

We next lower bound the AUPRC of the augmented population minimizer. By the AUPRC
stability inequality in equation (\ref{eq:9}),
\[
\begin{aligned}
\operatorname{AUPRC}(\eta_{\alpha^*})
-
\operatorname{AUPRC}(s_{\alpha^*})\le
2
\left\{
    \frac{\pi_0}{\pi_1}
    \int_0^1
    \left|
        \widetilde{\operatorname{FPR}}_{\eta_{\alpha^*}}(r)
        -
        \widetilde{\operatorname{FPR}}_{s_{\alpha^*}}(r)
    \right|
    dr
\right\}^{1/2}.
\end{aligned}
\]
By equation (\ref{eq:FPR-difference-diff-s}), the integrated FPR difference is bounded by the probability of a pairwise
ranking disagreement:
\[
\begin{aligned}
\int_0^1
\left|
    \widetilde{\operatorname{FPR}}_{\eta_{\alpha^*}}(r)
    -
    \widetilde{\operatorname{FPR}}_{s_{\alpha^*}}(r)
\right|
dr\le
\mathbb P\left(
    \operatorname{sign}\{\eta_{\alpha^*}(X_1)-\eta_{\alpha^*}(X_0)\}
    \neq
    \operatorname{sign}
    \{s_{\alpha^*}(X_1)-s_{\alpha^*}(X_0)\}
\right).
\end{aligned}
\]
By a standard set argument, for any $t>0$,
\[
\begin{aligned}
&
\mathbb P\left(
    \operatorname{sign}\{\eta_{\alpha^*}(X_1)-\eta_{\alpha^*}(X_0)\}
    \neq
    \operatorname{sign}
    \{s_{\alpha^*}(X_1)-s_{\alpha^*}(X_0)\}
\right)
\\
&\qquad\le
\mathbb P\bigl(
    |\eta_{\alpha^*}(X_1)-\eta_{\alpha^*}(X_0)|\le t
\bigr)
+
\mathbb P\bigl(
    |e(X_1)|+|e(X_0)|>t
\bigr).
\end{aligned}
\]
By the same analysis used between equations (\ref{eq:AUC-two-prob-upper}) and (\ref{eq:AUC-epbal-upper}), it then holds that
\[
    \int_0^1
    \left|
        \widetilde{\operatorname{FPR}}_{\eta_{\alpha^*}}(r)
        -
        \widetilde{\operatorname{FPR}}_{s_{\alpha^*}}(r)
    \right|
    dr
    \le
    C
    \left(
        \frac{\epsilon^*}{m_\ell}
    \right)^{1/3}.
\]
Consequently, combining the AUPRC invariance for strictly increasing functions,
\[
    \operatorname{AUPRC}(s_{\alpha^*})
    \ge
    \operatorname{AUPRC}(\Lambda_{\eta^*})
    -
    C
    \left(\frac{\pi_0}{\pi_1}\right)^{1/2}\left(
        \frac{\epsilon^*}{m_\ell}
    \right)^{1/6}.
\]

Combining this inequality with
$\operatorname{AUPRC}(s_{\mathrm{raw}}^*)<A$, we obtain the population
improvement bound
\[
\begin{aligned}
\operatorname{AUPRC}(s_{\alpha^*})
-
\operatorname{AUPRC}(s_{\mathrm{raw}}^*)\ge
\operatorname{AUPRC}(\Lambda_{\eta^*})
-
A
-
C\left(\frac{\pi_0}{\pi_1}\right)^{1/2}
\left(
    \frac{\epsilon^*}{m_\ell}
\right)^{1/6}.
\end{aligned}
\]

Finally, it remains to pass from population minimizers to empirical estimators. By the same procedure with equations (\ref{eq:rawLambdadiffAUPRC}) and (\ref{eq:augLambdadiffAUPRC}), but with the simultaneous statement in Lemma \ref{lemma:empiricalestimation}, the following inequalities hold
simultaneously with probability at least $1-\delta$:
\[
\begin{aligned}
&
\left|
    \operatorname{AUPRC}(s_{\widehat\theta_{\mathrm{raw}}})
    -
    \operatorname{AUPRC}(s_{\mathrm{raw}}^*)
\right|\le 
\sqrt{\frac{CLc}{\lambda}\frac{\pi_0}{\pi_1}}\left(\frac{B^2d\log(14d/\delta)}{n_0+n_1}\right)^{1/4},
\end{aligned}
\]
and
\[
\begin{aligned}
&\left|
    \operatorname{AUPRC}(s_{\widehat\theta_{\mathrm{aug}}})
    -
    \operatorname{AUPRC}(s_{\alpha^*})
\right|
\\
&\leq \sqrt{\frac{CLc}{\lambda}\frac{\pi_0}{\pi_1}}\left(\frac{B^2d\log(14d/\delta)}{n_0+n_1+\tilde{n}}\right)^{1/4}+2\sqrt{2CL\frac{\pi_0}{\pi_1}}\left(\frac{L_g^2}{2\mu L_{\rho}}\left(\frac{\tilde{n}}{n_0+n_1+\tilde{n}}\epsilon_{\mathrm{syn}}\right)^2\right)^{1/4}.
\end{aligned}
\]
Taking $\widetilde n=\frac{\alpha^*}{1-\alpha^*}n_0-n_1$, and absorbing constants, we finish the proof.
\end{proof}

\begin{proof}[Proof of Theorem \ref{thm:best-BA-improvement}]
We prove the theorem for two metrics separately.

\medskip
\noindent\textbf{Proof for BA:}

First, we show that a high best-threshold balanced accuracy requires nontrivial majority
score variation. For any score \(s\in\mathcal S\), define
\(
V_0(s)
:=
\mathbb E_{P_0}\{s(X)-u_0\}^2 \).
Let \(F_{0,s}\) and \(F_{1,s}\) denote the distributions of \(s(X)\) under
\(X\sim P_0\) and \(X\sim P_1\), respectively. For a threshold \(\tau\), write
\(
B_s(\tau)=\{x:s(x)\ge \tau\}\).
Then
\[
\begin{aligned}
\mathrm{BA}(s,\tau)-\frac12
=
\frac12\left\{
P_1(B_s(\tau))+P_0(B_s(\tau)^c)-1
\right\}
=
\frac12\left\{
P_1(s(X)\ge \tau)-P_0(s(X)\ge \tau)
\right\}.
\end{aligned}
\]
Therefore,
\begin{equation}\label{eq:supBA-12-upper}
\sup_{\tau}\BA(s,\tau)-\frac12
\le
\frac12
\sup_{\tau\in\mathbb R}
\left|
F_{1,s}([\tau,\infty))-F_{0,s}([\tau,\infty))
\right|.
\end{equation}

Next, let
\(U_1\sim F_{1,s}\) and \(U_0\sim F_{0,s}\). For any coupling of \((U_1,U_0)\) and any
\(h>0\),
\[
\mathbf 1\{U_1\ge \tau\}-\mathbf 1\{U_0\ge \tau\}
\le
\mathbf 1\{|U_1-U_0|>h\}
+
\mathbf 1\{U_0\in[\tau-h,\tau)\}.
\]
Taking expectations, using Markov inequality and the \(B_s\)-Lipschitz property of \(F_{0,s}\), and optimizing over couplings gives
\[
F_{1,s}([\tau,\infty))-F_{0,s}([\tau,\infty))
\le
\frac{\mathbb E|U_1-U_0|}{h}
+
B_s h\le
\frac{W_1(F_{1,s},F_{0,s})}{h}+B_s h.
\]
The inequality in the other direction can be derived similarly.
Therefore taking absolute value and taking supremum gives
\[
\sup_{\tau\in\mathbb R}
\left|
F_{1,s}([\tau,\infty))-F_{0,s}([\tau,\infty))
\right|
\le
\frac{W_1(F_{1,s},F_{0,s})}{h}+B_s h.
\]
Choosing \(h=\{W_1(F_{1,s},F_{0,s})/B_s\}^{1/2}\), we get
\begin{equation}\label{eq:F10s-wass}
\sup_{\tau\in\mathbb R}
\left|
F_{1,s}([\tau,\infty))-F_{0,s}([\tau,\infty))
\right|
\le
2\sqrt{B_sW_1(F_{1,s},F_{0,s})}.
\end{equation}
Combining with equation (\ref{eq:supBA-12-upper}), we have
\[
\sup_{\tau}\BA(s,\tau)-\frac12
\le
\sqrt{B_sW_1(F_{1,s},F_{0,s})}.
\]
Combining with equation (\ref{eq:wasserstein-two-dist}) gives
\[
\sup_{\tau}\BA(s,\tau)
\le
\frac12
+
\sqrt{B_s(1+\sqrt{B_L})}\,
V_0(s)^{1/4}.
\]
Consequently, 
\[
\sup_{\tau}\BA(s,\tau)\ge A\quad\Rightarrow\quad V_0(s)
\ge
\left(
\frac{A-\frac12}
{\sqrt{B_s(1+\sqrt{B_L})}}
\right)^4=:\kappa_A^{\rm BA}.
\]

Next, we show that high best-threshold balanced accuracy requires a majority-risk cost.
Combining the analysis with
equation (\ref{eq:R0-R0-const-geqV0}),
\[
\sup_\tau\BA(s,\tau)\ge A
\quad\Longrightarrow\quad
\cR_0(s)-\cR_0^*
\ge
m_0\kappa_A^{\rm BA}.
\]

We now show that the raw imbalanced population minimizer cannot pay this majority-risk
cost under the stated imbalance condition. 
By the severe imbalance condition,
\(
\frac{\alpha_{\rm raw}}{1-\alpha_{\rm raw}}D_1
<
m_0\kappa_A^{\rm BA}\), combined with equation (\ref{eq:R-0-alpha-raw-D-1}),
\[
\cR_0(s^*_{\rm raw})-\cR_0^*
<
m_0\kappa_A^{\rm BA}.
\]
If, contrary to the desired conclusion,
\(
\sup_\tau\BA(s^*_{\rm raw},\tau)\ge A\),
then the previous analysis would give
\(
\cR_0(s^*_{\rm raw})-\cR_0^*
\ge
m_0\kappa_A^{\rm BA}\),
which is a contradiction. Hence
\[
\sup_\tau\BA(s^*_{\rm raw},\tau)<A.
\]

We next control the augmented population score.
Let
\[
B^*
=
\{x:\Lambda_{\eta^*}(x)\ge 1/2\}
=
\{x:L_{\eta^*}(x)\ge 1\}=\{x:\eta_{\alpha^*}(x)\ge \alpha^*\}.
\]
For any measurable set \(B\subseteq\mathcal X\),
\(
\mathrm{BA}(B)
=
\frac12\{P_1(B)+P_0(B^c)\}\).
Since \(dP_1=L_{\eta^*}\,dP_0\),
\[
\mathrm{BA}(B)
=
\frac12
\left\{
1+\int_B(L_{\eta^*}(x)-1)\,dP_0(x)
\right\}.
\]
Thus \(B^*\) maximizes balanced accuracy, and by strictly increasing transformation,
\[
\sup_\tau\BA(\eta_{\alpha^*},\tau)=\sup_\tau\BA(\Lambda_{\eta^*},\tau)=\mathrm{BA}(B^*).
\]
Since
\(
\sup_\tau\BA(s_{\alpha^*},\tau)
\ge
\mathrm{BA}\left(\{x:s_{\alpha^*}(x)\ge \alpha^*\}\right)\),
we have
\[
\begin{aligned}
\sup_\tau\BA(\Lambda_{\eta^*},\tau)-\sup_\tau\BA(s_{\alpha^*},\tau)
&\le
\mathrm{BA}(B^*)
-
\mathrm{BA}\left(\{x:s_{\alpha^*}(x)\ge \alpha^*\}\right)
\\
&=
\frac12
\int_{\{s_{\alpha^*}\ge \alpha^*\}\triangle B^*}
|L(x)-1|\,dP_0(x).
\end{aligned}
\]
On the disagreement set
\(
\{s_{\alpha^*}\ge \alpha^*\}\triangle \{\eta_{\alpha^*}\ge \alpha^*\}\),
we must have
\(
|\eta_{\alpha^*}(x)-\alpha^*|
\le
|s_{\alpha^*}(x)-\eta_{\alpha^*}(x)|\).
Moreover, inverse the relationship, we have
\[
L_{\eta^*}=\frac{1-\alpha^*}{\alpha^*}\frac{\eta_{\alpha^*}}{1-\eta_{\alpha^*}},\qquad |L_{\eta^*}-1|=\frac{|\eta_{\alpha^*}-\alpha^*|}{\alpha^*(1-\eta_{\alpha^*})}.
\]
Since $L_{\eta^*}\leq B_L$,
\(
\frac{1}{1-\eta_{\alpha^*}}=1+\frac{\alpha^*}{1-\alpha^*}L_{\eta^*}\leq 1+\frac{\alpha^*}{1-\alpha^*}B_L\).
Thus
\[
|L_{\eta^*}-1|\leq \left(\frac{1}{\alpha^*}+\frac{B_L}{1-\alpha^*}\right)|\eta_{\alpha^*}-\alpha^*|\leq \left(\frac{1}{\alpha^*}+\frac{B_L}{1-\alpha^*}\right)|s_{\alpha^*}(x)-\eta_{\alpha^*}(x)|.
\]
Combining the preceding displays with equation (\ref{eq:sbal-Lambda-diff}) gives
\[
\begin{aligned}
&\sup_\tau\BA(\Lambda_{\eta^*},\tau)-\sup_\tau\BA(s_{\alpha^*},\tau)
\le
\frac{1}{2}\left(\frac{1}{\alpha^*}+\frac{B_L}{1-\alpha^*}\right)\bE_{P_0}|s_{\alpha^*}(X)-\eta_{\alpha^*}(X)|
\\
&\quad \le
\frac{1}{2}\left(\frac{1}{\alpha^*}+\frac{B_L}{1-\alpha^*}\right)
\left[
\mathbb E_{P_0}\{s_{\alpha^*}(X)-\eta_{\alpha^*}(X)\}^2
\right]^{1/2}
\le
\frac{1-\alpha^*+\alpha^*B_L}{2\alpha^*(1-\alpha^*)^{3/2}}
\sqrt{\frac{\epsilon^*}{m_\ell}}.
\end{aligned}
\]
Hence
\[
\sup_\tau\BA(s_{\alpha^*},\tau)
\ge
\sup_\tau\BA(\Lambda_{\eta^*},\tau)
-
\frac{1-\alpha^*+\alpha^*B_L}{2\alpha^*(1-\alpha^*)^{3/2}}
\sqrt{\frac{\epsilon^*}{m_\ell}}.
\]
Combining this with \(\sup_{\tau}\BA(s^*_{\rm raw},\tau)<A\), we obtain the population-level gap
\[
\sup_\tau\BA(s_{\alpha^*},\tau)
-
\sup_{\tau}\BA(s^*_{\rm raw},\tau)
\ge
\sup_{\tau}\BA(\Lambda_{\eta^*},\tau)-A
-
\frac{1-\alpha^*+\alpha^*B_L}{2\alpha^*(1-\alpha^*)^{3/2}}
\sqrt{\frac{\epsilon^*}{m_\ell}}.
\]

It remains to pass from population minimizers to empirical estimators. By the same procedure with equations (\ref{eq:raw-BA-empirical}) and (\ref{eq:aug-BA-empirical}), but with the simultaneous statement in Lemma \ref{lemma:empiricalestimation}, we have the following holds simultaneously with probability at least $1-\delta$ that
\[
\left|\sup_\tau \BA(s_{\hat{\theta}_\raw},\tau)-\sup_\tau \BA(s_{\rm raw}^*,\tau)\right|\leq \frac{CLc}{\lambda}\sqrt{\frac{B^2d\log(14d/\delta)}{n_0+n_1}},
\]
and
\[
\left|\sup_\tau \BA(s_{\hat{\theta}_\aug},\tau)-\sup_\tau \BA(s_{\alpha^*},\tau)\right|\leq \frac{CLc}{\lambda}\sqrt{\frac{B^2d\log(14d/\delta)}{n_0+n_1+\tilde{n}}}+\frac{CLL_g}{\sqrt{2\mu L_\rho}}\frac{\tilde{n}}{n_0+n_1+\tilde{n}}\epsilon_{\syn}.
\]
Taking $\tilde n=\frac{\alpha^*}{1-\alpha^*}n_0-n_1$ and absorbing constants gives the final result for BA.

\medskip
\noindent\textbf{Proof for $\F_1$ score:}

We first show that a high best-threshold $\F_1$ score implies nontrivial majority-score variation.
For a fixed threshold $\tau$, define
\(
R_\tau:=P_1(s(X)\ge \tau)\), and
\(Q_\tau:=P_0(s(X)\ge \tau)\).
Then
\[
\F_1(s,\tau)
=
\frac{2\pi_1R_\tau}
{\pi_1(1+R_\tau)+\pi_0Q_\tau}.
\]
The constant-score best-threshold $\F_1$ baseline is
\(
b_\pi=\frac{2\pi_1}{1+\pi_1}\).
By direct calculation,
\[
\begin{aligned}
\F_1(s,\tau)-b_\pi
=
\frac{2\pi_1R_\tau}
{\pi_1(1+R_\tau)+\pi_0Q_\tau}
-
\frac{2\pi_1}{1+\pi_1}
=
\frac{
2\pi_1\{R_\tau-\pi_1-\pi_0Q_\tau\}
}
{
\{\pi_1(1+R_\tau)+\pi_0Q_\tau\}(1+\pi_1)
}.
\end{aligned}
\]
Since
\(
R_\tau-\pi_1-\pi_0Q_\tau
\le R_\tau-Q_\tau\)
and 
\(
\pi_1(1+R_\tau)+\pi_0Q_\tau\ge \pi_1\),
we have
\[
\F_1(s,\tau)-b_\pi
\le
2\left|R_\tau-Q_\tau\right|.
\]
Taking the supremum over $\tau$ gives
\[
\sup_\tau \F_1(s,\tau)-b_\pi
\le
2\sup_{\tau}
\left|
P_1(s(X)\ge \tau)-P_0(s(X)\ge \tau)
\right|.
\]
By equation (\ref{eq:F10s-wass}), it holds that
\[
\sup_{\tau}
\left|
P_1(s(X)\ge \tau)-P_0(s(X)\ge \tau)
\right|
\le
2\sqrt{B_s W_1(F_{1,s},F_{0,s})}.
\]
Moreover, combining with equation (\ref{eq:wasserstein-two-dist}),
\[
\sup_\tau \F_1(s,\tau)
\le
b_\pi
+
4\sqrt{B_s(1+\sqrt{B_L})}\,V_0(s)^{1/4}.
\]
Therefore, whenever $\sup_\tau \F_1(s,\tau)\ge A>b_\pi$, we have
\[
V_0(s)
\ge
\left(
\frac{A-b_\pi}
{4\sqrt{B_s(1+\sqrt{B_L})}}
\right)^4
=:\kappa_A^{\F_1}.
\]
Hence by equation (\ref{eq:R0-R0-const-geqV0}),
\[
\sup_\tau \F_1(s,\tau)\ge A
\quad\Longrightarrow\quad
R_0(s)-R_0(s_0)
\ge
m_0\kappa_A^{\F_1}.
\]

Next we show that the raw population minimizer cannot achieve $\F_1$ level $A$.
By the severe imbalance condition,
\(
\frac{\alpha_{\rm raw}}{1-\alpha_{\rm raw}}D_1
<
m_0\kappa_A^{\F_1}\), combined with equation (\ref{eq:R-0-alpha-raw-D-1}),
\[
R_0(s_{\rm raw}^*)-R_0(s_0)
<
m_0\kappa_A^{\F_1}.
\]
If $\sup_\tau \F_1(s_{\rm raw}^*,\tau)\ge A$, previous analysis would imply
\(
R_0(s_{\rm raw}^*)-R_0(s_0)
\ge
m_0\kappa_A^{\F_1}\),
which is a contradiction. Therefore
\[
\sup_\tau \F_1(s_{\rm raw}^*,\tau)<A.
\]

Next we show that the augmented target can approximate the oracle $\F_1$ score up to approximation error term. 
There exists a likelihood-ratio threshold set
\[
B_F^*
=
\{x:\Lambda_{\eta^*}(x)\ge t_F\}=\{x:\eta_{\alpha^*}(x)\ge t_F^*\},
\]
where $t_F^*=g_{\alpha^*}(t_F)$ and we select $t_F$ such that
\(
\F_1(B_F^*)=\sup_{\tau}\F_1(\Lambda,\tau)\).
For a measurable set $B$, write
\(
r_B:=P_1(B)\),
\(
q_B:=P_0(B)\),
so that
\[
\F_1(B)
=
\frac{2\pi_1r_B}
{\pi_1(1+r_B)+\pi_0q_B}.
\]
Equivalently, writing
\(
N(B):=2\pi_1P_1(B)\),
\(
D(B):=\pi_1(1+P_1(B))+\pi_0P_0(B)\),
we have
\(
\F_1(B)=\frac{N(B)}{D(B)}\).
By Lemma \ref{thm:likelihood-ratio-optimality}, for every measurable set $B$,
\(
\F_1(B)\le \sup_{\tau}\F_1(\Lambda,\tau)=\F_1(B_F^*)
\).
Since $D(B)>0$, this is equivalent to
\[
N(B)
-
\sup_{\tau}\F_1(\Lambda,\tau)D(B)
\le 0.
\]
Moreover, equality holds at $B=B_F^*$:
\(
N(B_F^*)
-
\sup_{\tau}\F_1(\Lambda,\tau)D(B_F^*)
=0\).
Therefore, $B_F^*$ is a maximizer, over measurable sets $B$, of the functional
\[
B\mapsto
N(B)
-
\sup_{\tau}\F_1(\Lambda,\tau)D(B).
\]
Now use $dP_1=L_{\eta^*}\,dP_0$. Then
\[
N(B)
=
\int_B 2\pi_1L_{\eta^*}(x)\,dP_0(x),\quad D(B)
=
\pi_1
+
\int_B\{\pi_1L_{\eta^*}(x)+\pi_0\}\,dP_0(x).
\]
Hence
\[
\begin{aligned}
&N(B)
-
\sup_{\tau}\F_1(\Lambda,\tau)D(B)
\\
&\quad =
\int_B
\left[
2\pi_1L_{\eta^*}(x)
-
\sup_{\tau}\F_1(\Lambda,\tau)
\{\pi_1L_{\eta^*}(x)+\pi_0\}
\right]dP_0(x)
-
\pi_1\sup_{\tau}\F_1(\Lambda,\tau).
\end{aligned}
\]
The last term does not depend on $B$. Thus maximizing the display above over
$B$ is equivalent to maximizing
\(
\int_B \psi(x)\,dP_0(x)\),
where
\[
\psi(x)
:=
2\pi_1L_{\eta^*}(x)
-
\sup_{\tau}\F_1(\Lambda,\tau)
\{\pi_1L_{\eta^*}(x)+\pi_0\}.
\]
For any integrable function $\psi$, the set maximizing
$\int_B\psi(x)\,dP_0(x)$ over all measurable $B$ is simply
\(
\{x:\psi(x)\ge 0\}\)
up to $P_0$-null sets. 
Therefore the oracle best-$\F_1$ decision region can be written, up to null sets, as
\(
B_F^*=\{x:\psi(x)\ge 0\}.
\)
The $\F_1$ difference can be written as
\[
\sup_{\tau}\F_1(\Lambda,\tau)-\F_1(B)
=
\frac{
\sup_{\tau}\F_1(\Lambda,\tau)D(B)-N(B)
}{D(B)}.
\]
The previous analysis gives
\[
\sup_{\tau}\F_1(\Lambda,\tau)D(B)-N(B)
=
\pi_1\sup_{\tau}\F_1(\Lambda,\tau)
-
\int_B\psi(x)\,dP_0(x).
\]
Since $B_F^*=\{x:\psi(x)\ge 0\}$ attains the oracle value, we have
\[
0
=
\sup_{\tau}\F_1(\Lambda,\tau)D(B_F^*)-N(B_F^*)
=
\pi_1\sup_{\tau}\F_1(\Lambda,\tau)
-
\int_{B_F^*}\psi(x)\,dP_0(x).
\]
Thus
\[
\begin{aligned}
\sup_{\tau}\F_1(\Lambda,\tau)D(B)-N(B)
=
\int_{B_F^*}\psi(x)\,dP_0(x)
-
\int_B\psi(x)\,dP_0(x)
=
\int_{B_F^*\triangle B}|\psi(x)|\,dP_0(x),
\end{aligned}
\]
where the last equality uses $\psi\ge 0$ on $B_F^*$ and $\psi<0$ on
$(B_F^*)^c$. Then it follows that
\[
\sup_{\tau}\F_1(\Lambda,\tau)-\F_1(B)
=
\frac{
\int_{B_F^*\triangle B}|\psi(x)|\,dP_0(x)
}{
\pi_1(1+r_B)+\pi_0q_B
}\le
\frac{1}{\pi_1}
\int_{B_F^*\triangle B}|\psi(x)|\,dP_0(x).
\]
Now take
\(
B_{\rm aug}:=\{x:s_{\alpha^*}(x)\ge t_F^*\}\).
On $B_F^*\triangle B_{\rm aug}$, we have
\(
|\eta_{\alpha^*}(x)-t_F^*|
\le
|s_{\alpha^*}(x)-\eta_{\alpha^*}(x)|\).
We now bound $|\psi(x)|$. By definition, it can be written as
\[
\psi(x)
=
\pi_1\{2-\sup_{\tau}\F_1(\Lambda,\tau)\}L_{\eta^*}(x)
-
\pi_0\sup_{\tau}\F_1(\Lambda,\tau).
\]
Since the oracle best-$\F_1$ set is
\[
B_F^*=\{x:\Lambda_{\eta^*}(x)\ge t_F\}=\{x:\eta_{\alpha^*}(x)\ge t_F^*\}
      =\{x:L_{\eta^*}(x)\ge \ell_F\},
\]
where
\[
\ell_F=\frac{t_F}{1-t_F}=\frac{1-\alpha^*}{\alpha^*}\frac{t_F^*}{1-t_F^*},
\]
the boundary condition $\psi(x)=0$ at $L_{\eta^*}(x)=\ell_F$ gives
\[
\pi_1\{2-\sup_{\tau}\F_1(\Lambda,\tau)\}\ell_F
=
\pi_0\sup_{\tau}\F_1(\Lambda,\tau).
\]
Therefore,
\[
\psi(x)
=
\pi_1\{2-\sup_{\tau}\F_1(\Lambda,\tau)\}
\{L_{\eta^*}(x)-\ell_F\}.
\]
Consequently,
\[
|\psi(x)|
\le
2\pi_1 |L_{\eta^*}(x)-\ell_F|,
\]
because $\sup_{\tau}\F_1(\Lambda,\tau)\in[0,1]$. Plugging in $L_{\eta^*}$ and $\ell_F$ with respect to $\eta_{\alpha^*}$ and $t_F^*$ gives
\begin{align*}
    |L_{\eta^*}(x)-\ell_F|=\left|\frac{1-\alpha^*}{\alpha^*}\frac{\eta_{\alpha^*}}{1-\eta_{\alpha^*}}-\frac{1-\alpha^*}{\alpha^*}\frac{t_F^*}{1-t_F^*}\right|=\frac{1-\alpha^*}{\alpha^*}\frac{|\eta_{\alpha^*}-t_F^*|}{(1-\eta_{\alpha^*})(1-t_F^*)}.
\end{align*}
Since $L_{\eta^*}\leq B_L$,
\(
\frac{1}{1-\eta_{\alpha^*}}=1+\frac{\alpha^*}{1-\alpha^*}L_{\eta^*}\leq 1+\frac{\alpha^*}{1-\alpha^*}B_L\).
Also, the oracle $\F_1$ threshold satisfies $\ell_F\leq B_L$. Indeed, if $\ell_F>B_L$, then $\{L_{\eta^*}\geq \ell_F\}$ is empty, so the resulting $\F_1$ is $0$, contradicting the fact that predicting everything positive already gives a strictly positive $\F_1$ score. Therefore,
\(
\frac{1}{1-t_F^*}=1+\frac{\alpha^*}{1-\alpha^*}\ell_F\leq 1+\frac{\alpha^*}{1-\alpha^*}B_L\).
Thus,
\[
|L_{\eta^*}(x)-\ell_F|\leq \frac{1-\alpha^*}{\alpha^*}\left(1+\frac{\alpha^*}{1-\alpha^*}B_L\right)^2|\eta_{\alpha^*}-t_F^*|=\frac{(1-\alpha^*+\alpha^* B_L)^2}{\alpha^*(1-\alpha^*)}|\eta_{\alpha^*}-t_F^*|.
\]
Combining the preceding displays yields
\[
|\psi(x)|
\le
2\pi_1\frac{(1-\alpha^*+\alpha^* B_L)^2}{\alpha^*(1-\alpha^*)}|\eta_{\alpha^*}-t_F^*|.
\]
Combining the preceding inequalities gives
\[
\begin{aligned}
\sup_\tau \F_1(\Lambda,\tau)-\F_1(B_{\rm aug})
&\le
\frac{2(1-\alpha^*+\alpha^* B_L)^2}{\alpha^*(1-\alpha^*)}
\bE_{P_0}|s_{\alpha^*}(X)-\eta_{\alpha^*}(X)|
\\
&\le
\frac{2(1-\alpha^*+\alpha^* B_L)^2}{\alpha^*(1-\alpha^*)}
\left[
\bE_{P_0}\{s_{\alpha^*}(X)-\eta_{\alpha^*}(X)\}^2
\right]^{1/2}.
\end{aligned}
\]
By $M^*=(1-\alpha^*)P_0+\alpha^* P_1$,
\[
\bE_{P_0}\{s_{\alpha^*}(X)-\eta_{\alpha^*}(X)\}^2\leq \frac{1}{1-\alpha^*}\bE_{M^*}\{s_{\alpha^*}(X)-\eta_{\alpha^*}(X)\}^2,
\]
Thus, by equation (\ref{eq:sbal-Lambda-diff}),
\[
\sup_\tau \F_1(\Lambda,\tau)-\F_1(B_{\rm aug})\le \frac{2(1-\alpha^*+\alpha^* B_L)^2}{\alpha^*(1-\alpha^*)^{3/2}}
\sqrt{\frac{\epsilon^*}{m_{\ell}}}.
\]
Since $\sup_\tau \F_1(s_{\alpha^*},\tau)\ge \F_1(B_{\rm aug})$, we obtain
\[
\sup_\tau \F_1(s_{\alpha^*},\tau)
\ge
\sup_\tau \F_1(\Lambda,\tau)
-
\frac{2(1-\alpha^*+\alpha^* B_L)^2}{\alpha^*(1-\alpha^*)^{3/2}}
\sqrt{\frac{\epsilon^*}{m_{\ell}}}.
\]
Together with $\sup_\tau \F_1(s_{\rm raw}^*,\tau)<A$, this gives the population-level gap
\[
\sup_\tau \F_1(s_{\alpha^*},\tau)-\sup_\tau \F_1(s_{\rm raw}^*,\tau)
\ge
\sup_\tau \F_1(\Lambda,\tau)-A
-
\frac{2(1-\alpha^*+\alpha^* B_L)^2}{\alpha^*(1-\alpha^*)^{3/2}}
\sqrt{\frac{\epsilon^*}{m_{\ell}}}.
\]

Finally, we pass from population minimizers to empirical estimators.
By the same finite-sample argument as in equations (\ref{eq:F-1-to-oracle-raw}) and (\ref{eq:F-1-to-oracle-aug}), but with the simultaneous statement in Lemma \ref{lemma:empiricalestimation}, the following inequalities hold
simultaneously with probability at least $1-\delta$:
\[
\left|
\sup_\tau \F_1(s_{\widehat\theta_{\rm raw}},\tau)
-
\sup_\tau \F_1(s_{\rm raw}^*,\tau)
\right|
\le
\frac{CLc}{\pi_1\lambda}
\sqrt{
\frac{B^2d\log(14d/\delta)}
{n_0+n_1}
},
\]
and
\[
\left|
\sup_\tau \F_1(s_{\widehat\theta_{\rm aug}},\tau)
-
\sup_\tau \F_1(s_{\alpha^*},\tau)
\right|
\le
\frac{CLc}{\pi_1\lambda}
\sqrt{
\frac{B^2d\log(14d/\delta)}
{n_0+n_1+\widetilde n}
}
+
\frac{4CLL_g}{\pi_1\sqrt{2\mu L_\rho}}
\frac{\widetilde n}{n_0+n_1+\widetilde n}
\epsilon_{\rm syn}.
\]
Taking $\widetilde n=\frac{\alpha^*}{1-\alpha^*}n_0-n_1$, and absorbing constants, we finish the proof.
\end{proof}

\section{Proofs of Examples and Corollaries}\label{sec:propositions-examples-corollaries}

\begin{proof}[Proof of Example \ref{example:easier-to-approximate}]
Define the uniform approximation errors
\[
a_{\rm bal}(K)
=
\inf_{h\in\mathcal H_K}
\sup_{t\in[0,1]}
|h(t)-t|, \quad a_\alpha(K)
=
\inf_{h\in\mathcal H_K}
\sup_{t\in[0,1]}
|h(t)-g_\alpha(t)|.
\]
We first study the uniform approximation error of the balanced target.
For any \(h\in\mathcal H_K\), define
\(
e(t)=h(t)-t\).
Since \(h\) is \(K\)-Lipschitz,
\(
h(1)-h(0)\le K\).
Therefore,
\[
e(0)-e(1)
=
h(0)-\bigl(h(1)-1\bigr)
=
1+h(0)-h(1)
\ge
1-K.
\]
Hence
\[
|e(0)|+|e(1)|
\ge
1-K,
\]
which implies
\[
\sup_{t\in[0,1]}|h(t)-t|
\ge
\frac{1-K}{2}.
\]
Taking the infimum over \(h\in\mathcal H_K\), we obtain
\[
a_{\rm bal}(K)
\ge
\frac{1-K}{2}.
\]
This lower bound is sharp. Indeed, define
\(
h_K(t)
=
Kt+\frac{1-K}{2}\).
Then \(h_K\in\mathcal H_K\), and \(h_K:[0,1]\to[0,1]\). Moreover,
\[
h_K(t)-t
=
(1-K)
\left(
\frac12-t
\right),
\]
so
\[
\sup_{t\in[0,1]}
|h_K(t)-t|
=
\frac{1-K}{2}.
\]
Therefore,
\(
a_{\rm bal}(K)
=
\frac{1-K}{2}\).
Because \(K<1\), this quantity is strictly positive. Hence the balanced target is not exactly representable uniformly as a transformation on $[0,1]$. In particular, if $\Lambda$ is nonconstant, then $x\mapsto \Lambda(x)$ is not exactly represented in $\mathcal S_K$.

Next we study the imbalanced target. A direct calculation gives
\(
g_\alpha(1)-g_\alpha(1-\alpha)
=
\frac12\).
Thus \(g_\alpha\) changes by \(1/2\) over the interval
\([1-\alpha,1]\), whose length is \(\alpha\).
Let \(h\in\mathcal H_K\), and define
\(
e_\alpha(t)=h(t)-g_\alpha(t)\).
Since \(h\) is \(K\)-Lipschitz,
\(
|h(1)-h(1-\alpha)|
\le
K\alpha\).
Therefore,
\[
\begin{aligned}
\left|
e_\alpha(1)-e_\alpha(1-\alpha)
\right|
=
\left|
\bigl(h(1)-h(1-\alpha)\bigr)
-
\bigl(g_\alpha(1)-g_\alpha(1-\alpha)\bigr)
\right|
\ge
\frac12-K\alpha.
\end{aligned}
\]
Consequently,
\(
|e_\alpha(1)|+|e_\alpha(1-\alpha)|
\ge
\frac12-K\alpha\),
and hence
\[
\sup_{t\in[0,1]}
|h(t)-g_\alpha(t)|
\ge
\frac12
\left(
\frac12-K\alpha
\right).
\]
Taking the infimum over \(h\in\mathcal H_K\), we obtain
\[
a_\alpha(K)
\ge
\frac12
\left(
\frac12-K\alpha
\right).
\]
Combining the previous bounds,
we get
\[
\begin{aligned}
a_\alpha(K)-a_{\rm bal}(K)
&\ge
\frac12
\left(
\frac12-K\alpha
\right)
-
\frac{1-K}{2}
=
\frac12
\left(
K(1-\alpha)-\frac12
\right).
\end{aligned}
\]
Therefore, if
\(
K(1-\alpha)>\frac12\),
then
\(
a_\alpha(K)>a_{\rm bal}(K)\).
This proves the uniform approximation separation.

We now prove the squared \(L^2\) lower bound for the imbalanced target. Define the squared \(L^2(\mu)\) approximation
errors
\[
A_{\rm bal}(K)
=
\inf_{h\in\mathcal H_K}
\bE_\mu
\left[
\left(
h(\Lambda_{\eta^*}(X))-\Lambda_{\eta^*}(X)
\right)^2
\right],
\]
and
\[
A_\alpha(K)
=
\inf_{h\in\mathcal H_K}
\bE_\mu
\left[
\left(
h(\Lambda_{\eta^*}(X))-g_\alpha(\Lambda_{\eta^*}(X))
\right)^2
\right].
\]
Let
\[
I_L=[1-\alpha,1-3\alpha/4],
\qquad
I_H=[1-\alpha/4,1].
\]
For \(t\in I_L\) and for \(t\in I_H\), respectively, by the range of $t$ and $\alpha$,
\[
g_\alpha(t)\le \frac23,
\quad t\in I_L;\qquad g_\alpha(t)\ge \frac45,
\quad t\in I_H.
\]
It follows that, for every \(t_L\in I_L\) and \(t_H\in I_H\),
\[
g_\alpha(t_H)-g_\alpha(t_L)
\ge
\frac45-\frac23
=
\frac{2}{15}.
\]
Since \(h\in\mathcal H_K\) and \(|t_H-t_L|\le \alpha\),
\(
|h(t_H)-h(t_L)|
\le
K\alpha\).
Therefore,
\[
\begin{aligned}
|e_\alpha(t_H)|+|e_\alpha(t_L)|
&\ge
|e_\alpha(t_H)-e_\alpha(t_L)|
\\
&=
\left|
\bigl(h(t_H)-h(t_L)\bigr)
-
\bigl(g_\alpha(t_H)-g_\alpha(t_L)\bigr)
\right|
\ge
\frac{2}{15}-K\alpha.
\end{aligned}
\]
Using
\(
(a+b)^2\le 2(a^2+b^2)
\),
we obtain
\[
e_\alpha(t_H)^2+e_\alpha(t_L)^2
\ge
\frac12
\left(
\frac{2}{15}-K\alpha
\right)^2.
\]
Averaging this inequality over \(t_L\in I_L\) and \(t_H\in I_H\), and using
\(
|I_L|=|I_H|=\frac{\alpha}{4}\)
gives
\[
\int_{I_L\cup I_H}
e_\alpha(t)^2\,dt
\ge
\frac{\alpha}{8}
\left(
\frac{2}{15}-K\alpha
\right)^2.
\]
Since the density \(q\) of \(T=\Lambda_{\eta^*}(X)\) under \(\mu\) satisfies
\(q(t)\ge \underline q\) on \(I_L\cup I_H\), we have
\[
\bE_\mu
\left[
\left(
h(\Lambda_{\eta^*}(X))-g_\alpha(\Lambda_{\eta^*}(X))
\right)^2
\right]
\ge
\frac{\underline q\,\alpha}{8}
\left(
\frac{2}{15}-K\alpha
\right)^2.
\]
Taking the infimum over \(h\in\mathcal H_K\), we obtain
\[
A_\alpha(K)
\ge
\frac{\underline q\,\alpha}{8}
\left(
\frac{2}{15}-K\alpha
\right)^2.
\]
If \(K\alpha\le 1/15\), then
\(
\frac{2}{15}-K\alpha
\ge
\frac{1}{15}\),
and hence
\[
A_\alpha(K)
\ge
\frac{\underline q\,\alpha}{8}
\cdot
\frac{1}{225}
=
\frac{\underline q}{1800}\alpha.
\]

For the balanced target, use again
\(
h_K(t)
=
Kt+\frac{1-K}{2}\),
then
\[
|h_K(t)-t|
\le
\frac{1-K}{2},
\qquad
\forall\, t\in[0,1].
\]
Therefore,
\[
A_{\rm bal}(K)
\le
\bE_\mu
\left[
\left(
h_K(\Lambda_{\eta^*}(X))-\Lambda_{\eta^*}(X)
\right)^2
\right]
\le
\frac{(1-K)^2}{4}.
\]
Combining the two bounds, if
\(
\frac{(1-K)^2}{4}
<
\frac{\underline q}{1800}\alpha
\),
then
\(
A_{\rm bal}(K)<A_\alpha(K)\).
Solving gives the range of $\alpha$ for which this inequality holds.
This finishes the proof.
\end{proof}

\begin{proof}[Proof of Example \ref{example:relu-network}]
We first prove the approximation upper bound for the balanced target. For each
coordinate \(j=1,\ldots,p\), partition \([0,1]\) into \(K+1\) equal intervals, where
\(
K:=\left\lfloor \frac{m}{p}\right\rfloor\).
We show that a piecewise affine interpolant has a ReLU representation.
Let
\[
0=u_0<u_1<\cdots<u_K<u_{K+1}=1
\]
be the grid points of the partition, with \(u_k=k/(K+1)\). Let
\(\widetilde\lambda_{j,K}\) be the continuous piecewise affine interpolant of
\(\lambda_j\) on this grid. That is,
\(\widetilde\lambda_{j,K}\) agrees with \(\lambda_j\) at the grid points:
\[
\widetilde\lambda_{j,K}(u_k)=\lambda_j(u_k),
\qquad
k=0,\ldots,K+1,
\]
and is affine on each interval \([u_{k-1},u_k]\). Denote the slope of
\(\widetilde\lambda_{j,K}\) on \([u_{k-1},u_k]\) by
\[
r_{j,k}
=
\frac{\lambda_j(u_k)-\lambda_j(u_{k-1})}{u_k-u_{k-1}},
\qquad
k=1,\ldots,K+1.
\]
Thus, on the first interval \([u_0,u_1]\),
\(
\widetilde\lambda_{j,K}(t)
=
\lambda_j(0)+r_{j,1}t\).
At the interior grid point \(u_k\), the slope changes from \(r_{j,k}\) to
\(r_{j,k+1}\). Hence the slope jump at \(u_k\) is
\(
c_{j,k}=r_{j,k+1}-r_{j,k}\),
 \(
k=1,\ldots,K\).
The shifted ReLU \((t-u_k)_+\) is zero for \(t\le u_k\) and has derivative one
for \(t>u_k\). Therefore, adding
\(
c_{j,k}(t-u_k)_+\)
changes the slope by exactly \(c_{j,k}\) after the point \(u_k\), while leaving the
function unchanged before \(u_k\). Applying this correction at every interior grid
point gives
\[
\widetilde\lambda_{j,K}(t)
=
\lambda_j(0)
+
r_{j,1}t
+
\sum_{k=1}^K
(r_{j,k+1}-r_{j,k})(t-u_k)_+.
\]
Indeed, for \(t\in[u_{\ell-1},u_\ell]\), only the terms with \(k\le \ell-1\) are
active, so the derivative of the right-hand side is
\[
r_{j,1}
+
\sum_{k=1}^{\ell-1}(r_{j,k+1}-r_{j,k})
=
r_{j,\ell},
\]
which is exactly the slope of the interpolant on \([u_{\ell-1},u_\ell]\). The two
functions also agree at \(t=0\), since both equal \(\lambda_j(0)\). Therefore the
displayed formula is exactly the piecewise affine interpolant.
Consequently, \(\widetilde\lambda_{j,K}\) admits the univariate ReLU representation
\[
\widetilde\lambda_{j,K}(t)
=
b_{j,0}+b_{j,1}t+\sum_{k=1}^K c_{j,k}(t-\tau_{j,k})_+,
\]
with
\(
b_{j,0}=\lambda_j(0),
~
b_{j,1}=r_{j,1},
~
\tau_{j,k}=u_k,
~
c_{j,k}=r_{j,k+1}-r_{j,k}
\).
Thus the interpolant has at most \(K\) breakpoints and can be represented by a
univariate ReLU network with \(K\) hidden units.
Define
\[
s_m(x)
=
\frac1p\sum_{j=1}^p \widetilde\lambda_{j,K}(x_j).
\]
Then
\[
s_m(x)
=
\frac1p\sum_{j=1}^p b_{j,0}
+
\sum_{j=1}^p \frac{b_{j,1}}{p}x_j
+
\frac1p
\sum_{j=1}^p\sum_{k=1}^K c_{j,k}(x_j-\tau_{j,k})_+.
\]
This is a \(p\)-dimensional shallow ReLU network. Each term
\((x_j-\tau_{j,k})_+\) is a ReLU unit with weight vector equal to the \(j\)-th
coordinate vector. The total number of hidden units is \(pK\le m\).

We now verify that this approximant belongs to the norm-constrained class
\(\mathcal N_{m,A}^{(p)}\). The linear coefficient vector of \(s_m\) is
\(
\beta_m=
\left(
\frac{b_{1,1}}{p},\ldots,\frac{b_{p,1}}{p}
\right)\).
Since each \(b_{j,1}\) is a slope of \(\lambda_j\), the mean-value theorem gives
\(
|b_{j,1}|\le \|\lambda_j'\|_{L^\infty([0,1])}\le B_1
\).
Therefore
\[
\|\beta_m\|_2
\le
\left\{
\sum_{j=1}^p \left(\frac{B_1}{p}\right)^2
\right\}^{1/2}
=
\frac{B_1}{\sqrt p}.
\]

Next, the ReLU coefficient attached to \((x_j-\tau_{j,k})_+\) is \(c_{j,k}/p\),
and the corresponding weight vector has Euclidean norm one. Hence the ReLU part of
the path norm is
\[
\sum_{j=1}^p\sum_{k=1}^K
\left|\frac{c_{j,k}}{p}\right|
=
\frac1p\sum_{j=1}^p\sum_{k=1}^K |c_{j,k}|.
\]
For the piecewise affine interpolant, \(c_{j,k}\) is the jump in slope at the \(k\)-th
breakpoint. Let \(h=1/(K+1)\), and let \(r_{j,k}\) be the slope of
\(\widetilde\lambda_{j,K}\) on the \(k\)-th interpolation interval. Then
\(
c_{j,k}=r_{j,k+1}-r_{j,k}\).
We have
\[
r_{j,k+1}-r_{j,k}
=
\frac{1}{h}\int_{0}^{h}
\left\{
\lambda_j'(u_k+t)-\lambda_j'(u_{k-1}+t)
\right\}\,dt .
\]
Because \(\|\lambda_j''\|_\infty\le B_2\), the derivative \(\lambda_j'\) is \(B_2\)-Lipschitz. Hence, for every \(t\in[0,h]\),
\[
\left|
\lambda_j'(u_k+t)-\lambda_j'(u_{k-1}+t)
\right|
\le B_2h.
\] 
Therefore
\[
\sum_{k=1}^K |c_{j,k}|
=
\sum_{k=1}^K |r_{j,k+1}-r_{j,k}|
\le
K B_2h
\le
B_2.
\]
It follows that
\(
\sum_{j=1}^p\sum_{k=1}^K
\left|\frac{c_{j,k}}{p}\right|
\le
B_2\).
Combining the linear and ReLU contributions, we obtain
\[
\|\beta_m\|_2
+
\sum_{j=1}^p\sum_{k=1}^K
\left|\frac{c_{j,k}}{p}\right|
\le
\frac{B_1}{\sqrt p}+B_2
\le A.
\]
Thus
\(
s_m\in\mathcal N_{m,A}^{(p)}\).
To bound the approximation error, we use the standard interpolation error bound.
Fix \(t\in[u_{k-1},u_k]\). By the classical error formula for linear interpolation,
there exists a point \(\xi_t\in(u_{k-1},u_k)\) such that
\[
\lambda_j(t)-\widetilde\lambda_{j,K}(t)
=
\frac{\lambda_j''(\xi_t)}{2}
(t-u_{k-1})(t-u_k).
\]
Hence
\[
|\lambda_j(t)-\widetilde\lambda_{j,K}(t)|
\le
\frac{\|\lambda_j''\|_{L^\infty([u_{k-1},u_k])}}{2}
(t-u_{k-1})(u_k-t).
\]
Since
\(
(t-u_{k-1})(u_k-t)
\le
\frac{(u_k-u_{k-1})^2}{4}
=
\frac{h^2}{4}\),
we get
\[
|\lambda_j(t)-\widetilde\lambda_{j,K}(t)|
\le
\frac{B_2}{2}\cdot \frac{h^2}{4}
=
\frac{B_2h^2}{8}.
\]
Because \(h=1/(K+1)\), this becomes
\[
|\lambda_j(t)-\widetilde\lambda_{j,K}(t)|
\le
\frac{B_2}{8(K+1)^2}.
\]
Taking the supremum over all \(t\in[0,1]\) yields
\[
\|\widetilde\lambda_{j,K}-\lambda_j\|_{L^\infty([0,1])}
\le
\frac{B_2}{8(K+1)^2}.
\]
Therefore,
\[
\begin{aligned}
\|s_m-\Lambda\|_{L^\infty([0,1]^p)}
&=
\sup_{x\in[0,1]^p}
\left|
\frac1p\sum_{j=1}^p
\{\widetilde\lambda_{j,K}(x_j)-\lambda_j(x_j)\}
\right| \\
&\le
\frac1p\sum_{j=1}^p
\|\widetilde\lambda_{j,K}-\lambda_j\|_{L^\infty([0,1])} \le
\frac{B_2}{8(K+1)^2}.
\end{aligned}
\]
Since \(\eta_{1/2}=\Lambda\), this proves
\[
\inf_{s\in\mathcal N_{m,A}^{(p)}}
\|s-\eta_{1/2}\|_{L^\infty([0,1]^p)}
\le
\frac{B_2}{8(K+1)^2}.
\]

We next derive the Lipschitz property from the network norm constraint. Let
\[
s(x)
=
\beta_0+\beta^\top x+\sum_{\ell=1}^m a_\ell
\sigma(w_\ell^\top x-t_\ell)
\in \mathcal N_{m,A}^{(p)}.
\]
Since \(\sigma\) is $1$-Lipschitz,
\(
|\sigma(w_\ell^\top x-t_\ell)-\sigma(w_\ell^\top z-t_\ell)|
\le
\|w_\ell\|_2\|x-z\|_2\).
Therefore,
\[
\begin{aligned}
|s(x)-s(z)|
&\le
|\beta^\top(x-z)|
+
\sum_{\ell=1}^m |a_\ell|
|\sigma(w_\ell^\top x-t_\ell)-\sigma(w_\ell^\top z-t_\ell)| \\
&\le
\left(
\|\beta\|_2+\sum_{\ell=1}^m |a_\ell|\,\|w_\ell\|_2
\right)
\|x-z\|_2 \le
A\|x-z\|_2.
\end{aligned}
\]
Hence every \(s\in\mathcal N_{m,A}^{(p)}\) is \(A\)-Lipschitz:
\(
\operatorname{Lip}(s)\le A.
\)

We now prove the lower bound for the imbalanced target. Fix
\(
0<\alpha\le \lambda_- r_0
\).
Define two points in \([0,1]^p\):
\[
x^+=(1,\ldots,1),
\qquad
x^-=
\left(1-\frac{\alpha}{\lambda_-},\ldots,
1-\frac{\alpha}{\lambda_-}\right).
\]
Since \(\alpha\le \lambda_-r_0\), each coordinate of \(x^-\) lies in
\([1-r_0,1]\). For each \(j\), using \(\lambda_j(1)=1\) and
\(\lambda_j'(t)\ge\lambda_-\) on \([1-r_0,1]\), we obtain
\[
1-\lambda_j\left(1-\frac{\alpha}{\lambda_-}\right)
=
\lambda_j(1)
-
\lambda_j\left(1-\frac{\alpha}{\lambda_-}\right)
=
\int_{1-\alpha/\lambda_-}^{1}\lambda_j'(u)\,du
\ge
\alpha.
\]
Therefore
\(
\lambda_j\left(1-\frac{\alpha}{\lambda_-}\right)
\le
1-\alpha\).
Averaging over \(j=1,\ldots,p\), we get
\[
\Lambda(x^-)
=
\frac1p
\sum_{j=1}^p
\lambda_j\left(1-\frac{\alpha}{\lambda_-}\right)
\le
1-\alpha.
\]
On the other hand,
\[
\Lambda(x^+)
=
\frac1p\sum_{j=1}^p\lambda_j(1)
=
1.
\]
Since \(g_\alpha\) is increasing,
\(
\eta_\alpha(x^-)
=
g_\alpha(\Lambda(x^-))
\le
g_\alpha(1-\alpha)\).
A direct calculation gives
\(
g_\alpha(1-\alpha)
=
\frac12\).
Also,
\(
\eta_\alpha(x^+)
=
g_\alpha(\Lambda(x^+))
=
g_\alpha(1)
=
1
\).
Thus
\[
\eta_\alpha(x^+)-\eta_\alpha(x^-)\ge \frac12.
\]

Now take any \(s\in\mathcal N_{m,A}^{(p)}\), and define
\(
e(x)=\eta_\alpha(x)-s(x)\).
Since \(\operatorname{Lip}(s)\le A\), as proved above,
\[
|s(x^+)-s(x^-)|
\le
A\|x^+-x^-\|_2.
\]
Moreover,
\[
\|x^+-x^-\|_2
=
\left\{
p\left(\frac{\alpha}{\lambda_-}\right)^2
\right\}^{1/2}
=
\frac{\sqrt p\,\alpha}{\lambda_-}.
\]
Hence
\[
|s(x^+)-s(x^-)|
\le
\frac{A\sqrt p\,\alpha}{\lambda_-}.
\]
Using the decomposition
\[
\eta_\alpha(x^+)-\eta_\alpha(x^-)
=
\{s(x^+)-s(x^-)\}
+
\{e(x^+)-e(x^-)\},
\]
we obtain
\[
\begin{aligned}
\frac12
\le
\eta_\alpha(x^+)-\eta_\alpha(x^-) \le
|s(x^+)-s(x^-)|
+
|e(x^+)|+|e(x^-)| 
\le
\frac{A\sqrt p\,\alpha}{\lambda_-}
+
2\|e\|_{L^\infty([0,1]^p)}.
\end{aligned}
\]
Therefore,
\[
\|s-\eta_\alpha\|_{L^\infty([0,1]^p)}
=
\|e\|_{L^\infty([0,1]^p)}
\ge
\frac14-\frac{A\sqrt p\,\alpha}{2\lambda_-}.
\]
Since the left-hand side is nonnegative, we may write
\[
\|s-\eta_\alpha\|_{L^\infty([0,1]^p)}
\ge
\left(
\frac14-\frac{A\sqrt p\,\alpha}{2\lambda_-}
\right)_+.
\]
Taking the infimum over \(s\in\mathcal N_{m,A}^{(p)}\), we obtain
\[
\inf_{s\in\mathcal N_{m,A}^{(p)}}
\|s-\eta_\alpha\|_{L^\infty([0,1]^p)}
\ge
\left(
\frac14-\frac{A\sqrt p\,\alpha}{2\lambda_-}
\right)_+.
\]
Finally, combine the two bounds. We have shown that
\[
\inf_{s\in\mathcal N_{m,A}^{(p)}}
\|s-\eta_{1/2}\|_{L^\infty([0,1]^p)}
\le
\frac{B_2}{8(K+1)^2},
\quad
\inf_{s\in\mathcal N_{m,A}^{(p)}}
\|s-\eta_\alpha\|_{L^\infty([0,1]^p)}
\ge
\frac14-\frac{A\sqrt p\,\alpha}{2\lambda_-},
\]
whenever the right-hand side is positive. Therefore, if
\[
0<\alpha_{\rm raw}
<
\min\left\{
\lambda_- r_0,\,
\frac{2\lambda_-}{A\sqrt p}
\left(\frac14-\frac{B_2}{8(K+1)^2}\right)
\right\},
\]
then
\[
\frac14-\frac{A\sqrt p\,\alpha_{\rm raw}}{2\lambda_-}
>
\frac{B_2}{8(K+1)^2}.
\]
Consequently,
\[
\inf_{s\in\mathcal N_{m,A}^{(p)}}
\|s-\eta_{1/2}\|_{L^\infty([0,1]^p)}
<
\inf_{s\in\mathcal N_{m,A}^{(p)}}
\|s-\eta_{\alpha_{\rm raw}}\|_{L^\infty([0,1]^p)}.
\]
This completes the proof.
\end{proof}

\begin{proof}[Proof of Corollary \ref{cor:mle-well-specified-augmentation}]
We first consider this model at the population level for general class weights. For brevity, denote the log-likelihood as $\ell^{(\mathrm{MLE})}$. Define the weights between class as $w=(w_0,w_1)$. Then the population estimator is defined as
\begin{align*}
    \theta_w^*\in\argmax_\theta ~\left\{w_0\bE_{P_0}\ell^{(\mathrm{MLE})}(s_\theta(X),0)+w_1\bE_{P_1}\ell^{(\mathrm{MLE})}(s_\theta(X),1)\right\}.
\end{align*}
We then show that the induced scores $s_{\theta_w^*}$ for general weights $w$ are indeed increasingly transformed from the likelihood ratio. By the Bernoulli model, we can explicitly write the population objective as
\begin{align*}
    \theta_w^* &\in\argmax_\theta ~\left\{w_0\bE_{P_0}\ell^{(\mathrm{MLE})}(s_\theta(X),0)+w_1\bE_{P_1}\ell^{(\mathrm{MLE})}(s_\theta(X),1)\right\}\\
    &=\argmax_\theta ~\left\{w_0\bE_{P_0}\log(1-s_\theta(X))+w_1\bE_{P_1}\log(s_\theta(X))\right\}\\
    &=\argmax_\theta w_0\int p_0\log(1-s_\theta(x))dx+w_1\int p_1\log(s_\theta(x))dx.
\end{align*}
For each $x$, denote
\begin{align*}
    \phi(x):=w_0p_0(x)\log(1-s_\theta(x))+w_1 p_1(x)\log(s_\theta(x)).
\end{align*}
Taking the derivative with respect to $t=s_\theta(x)$,
\begin{align*}
    \frac{\partial \phi}{\partial t}=-\frac{w_0p_0(x)}{1-t}+\frac{w_1p_1(x)}{t}=0, \qquad \frac{\partial^2 \phi}{\partial t^2}<0.
\end{align*}
Because the loss depends on the parameter $\theta$ only through the score $s_\theta$, and the score maximizer is attainable by the model class by assumption, the score maximizer is attained by a corresponding parameter maximizer. We derive the maximizer
\begin{align*}
    s_{\theta_w^*}(x)=\frac{w_1p_1(x)}{w_1p_1(x)+w_0p_0(x)}=\frac{w_1L_{\eta^*}(x)}{w_1L_{\eta^*}(x)+w_0}.
\end{align*}
    Define the function 
    \[
    g(t)=\frac{w_1t}{w_1t+w_0},
    \]
    which is a strictly increasing transformation. Thus the population score
    \begin{align*}
    s_{\theta_w^*}(x)=g(L_{\eta^*}(x)),
\end{align*}
is a strictly increasing transformation of $\Lambda_{\eta^*}$, because $L_{\eta^*}$ is a strictly increasing transformation of $\Lambda_{\eta^*}$.

Therefore, we have verified that Assumption \ref{assumption:population-ERM-well-specify} holds for the MLE model. Thus, the bounds follow by essentially the same argument as in the proofs of Theorems~\ref{thm:ERM-well-specified}, \ref{thm:AUPRC-converge-ERM}, \ref{thm:parametric-ERM-best-BA} and \ref{thm:parametric-ERM-F1}. This completes the proof.
\end{proof}

\section{Proofs of Lemmas}
\begin{proof}[Proof of Lemma \ref{thm:likelihood-ratio-optimality}]
We prove the results for the four metrics separately. For the four metrics considered, the proofs of Theorems \ref{thm:AUC-invariant} and \ref{thm:doesnotimproveBA} show that these metrics are invariant to strictly increasing transformations.
Note the relationship, 
\[
\Lambda_{\eta^*}(x)=\frac{L_{\eta^*}(x)}{1+L_{\eta^*}(x)}=g_1(L_{\eta^*}(x)),\quad L_{\eta^*}(x)=\frac{\Lambda_{\eta^*}(x)}{1-\Lambda_{\eta^*}(x)}=g_2(\Lambda_{\eta^*}(x)),
\]
where $g_1(t)=\frac{t}{1+t}$ and $g_2(t)=\frac{t}{1-t}$ are strictly increasing transformations. Therefore, the values of the four metrics for $\Lambda_{\eta^*}(x)$ are equal to those for $L_{\eta^*}(x)$. Thus, for brevity, we can also use $L_{\eta^*}(x)$ as the optimal score.

\medskip
\noindent\textbf{Optimality for AUROC:}

We first show that the likelihood-ratio threshold is ROC-optimal at every false-positive level.
For any measurable set \(S \subseteq \mathcal{X}\), define
\(
\operatorname{FPR}(S) := P_0(S)\),
\(\operatorname{TPR}(S) := P_1(S)\).
For \(\alpha \in (0,1)\), define the best achievable true-positive rate at false-positive level \(\alpha\) by
\[
\beta^*(\alpha)
:=
\sup\Bigl\{
\mathbb P_1(S)
:
S \subseteq \mathcal{X}\ \text{measurable},\;
P_0(S) \le \alpha
\Bigr\}.
\]
Define the set
\(
C_c := \{x : L_{\eta^*}(x) \ge c\}\).
Because \(L(X_0)\) is continuous under \(P_0\), for every \(\alpha \in (0,1)\) there exists \(c_\alpha\) such that
\(
\mathbb P_0(C_{c_\alpha}) = \alpha\).
We claim that \(C_{c_\alpha}\) achieves \(\beta^*(\alpha)\), that is,
\[
\mathbb P_1(S) \le \mathbb P_1(C_{c_\alpha})
\qquad
\text{for every measurable } S \text{ with } \mathbb P_0(S) \le \alpha.
\]

To prove this, first consider any measurable \(S\) with \(\mathbb P_0(S)=\alpha\). Since
\[
L(x) \ge c_\alpha \quad \text{on } C_{c_\alpha},
\qquad
L(x) \le c_\alpha \quad \text{on } C_{c_\alpha}^c,
\]
we have
\[
p_1(x) - c_\alpha p_0(x) \ge 0 \quad \text{on } C_{c_\alpha},
\qquad
p_1(x) - c_\alpha p_0(x) \le 0 \quad \text{on } C_{c_\alpha}^c.
\]
Therefore,
\[
\int_{C_{c_\alpha}} \bigl(p_1 - c_\alpha p_0\bigr)\,dx
\ge
\int_S \bigl(p_1 - c_\alpha p_0\bigr)\,dx.
\]
Rearranging gives
\[
\mathbb P_1(C_{c_\alpha}) - c_\alpha \mathbb P_0(C_{c_\alpha})
\ge
\mathbb P_1(S) - c_\alpha \mathbb P_0(S).
\]
Since \(\mathbb P_0(C_{c_\alpha}) = \mathbb P_0(S) = \alpha\), the \(c_\alpha \alpha\) terms cancel, yielding
\(
\mathbb P_1(C_{c_\alpha}) \ge \mathbb P_1(S)\).
If instead \(\mathbb P_0(S) < \alpha\), then \(S\) is still feasible for the constraint \(\mathbb P_0(S)\le \alpha\), and 
\[
\mathbb P_1(C_{c_\alpha}) - c_\alpha \mathbb P_0(C_{c_\alpha})
\ge
\mathbb P_1(S) - c_\alpha \mathbb P_0(S)\geq \mathbb P_1(S) - c_\alpha \alpha.
\]
We thus still have
\(
\mathbb P_1(C_{c_\alpha}) \ge \mathbb P_1(S)\).
Hence by definition of the supremum $\beta^*(\alpha)$,
\(
\beta^*(\alpha) = \mathbb P_1(C_{c_\alpha})\).
So we conclude that the likelihood-ratio threshold rule is pointwise optimal on the ROC plane.

Next we show that the ROC curve of any score is dominated by that of \(L_{\eta^*}\).
For any measurable score \(s\), define its threshold sets
\(
S_t(s) := \{x : s(x) \ge t\}\).
Its ROC curve can be written as
\[
\operatorname{ROC}_s(\alpha)
:=
\sup\Bigl\{
\mathbb P_1(S_t(s))
:
\mathbb P_0(S_t(s)) \le \alpha,\;
t \in \mathbb{R}
\Bigr\}.
\]
Every threshold set \(S_t(s)\) is just a measurable set, so by previous analysis,
\[
\mathbb  P_1(S_t(s)) \le \beta^*(\alpha) = \operatorname{ROC}_{L_{\eta^*}}(\alpha)
\qquad
\text{whenever } \mathbb  P_0(S_t(s)) \le \alpha.
\]
Taking the supremum over all such thresholds \(t\), we obtain
\[
\operatorname{ROC}_s(\alpha) \le \operatorname{ROC}_{L_{\eta^*}}(\alpha)
\qquad
\text{for every } \alpha \in (0,1).
\]

Assume that \(s(X_0)\) has a continuous distribution under \(P_0\). For each threshold \(t\in\mathbb R\), define
\(
S_t(s):=\{x:s(x)\ge t\}\).
Recall that
\[
\operatorname{ROC}_s(\alpha)
:=
\sup\Bigl\{
\mathbb  P_1(S_t(s)):\; \mathbb  P_0(S_t(s))\le \alpha,\ t\in\mathbb R
\Bigr\},
\qquad \alpha\in(0,1).
\]
Recall that the AUROC is defined as
\(
\AUC(s)=\mathbb P\left(s(X_1)>s(X_0)\right)+\frac12 \mathbb P\left(s(X_1)=s(X_0)\right)
\).
We next prove that this definition is equivalent to
\(
\operatorname{AUC}(s)=\int_0^1 \operatorname{ROC}_s(\alpha)\,d\alpha\),
based on the definition of the ROC curve.
To achieve this, first define the survival functions of the score under the two classes by
\[
\bar F_0(t):=\mathbb P\bigl(s(X_0)\ge t\bigr)=\mathbb  P_0(S_t(s)),
\qquad
\bar F_1(t):=\mathbb P\bigl(s(X_1)\ge t\bigr)=\mathbb  P_1(S_t(s)).
\]
Because \(s(X_0)\) has a continuous distribution, the map \(t\mapsto \bar F_0(t)\) is continuous and nonincreasing, with limits decreasing from $1$ to $0$.
Therefore, for every \(\alpha\in(0,1)\), there exists a threshold \(t_\alpha\) such that
\(
\mathbb  P_0(S_{t_\alpha}(s))=\bar F_0(t_\alpha)=\alpha\).
We now show that this threshold attains the supremum in the definition of \(\operatorname{ROC}_s(\alpha)\). The threshold sets are nested:
\[
t_1\le t_2
\quad\Longrightarrow\quad
S_{t_2}(s)\subseteq S_{t_1}(s).
\]
Hence both \(\mathbb  P_0(S_t(s))\) and \(\mathbb  P_1(S_t(s))\) are nonincreasing functions of \(t\). Now let \(t \in \mathbb{R}\) be any threshold satisfying
\[
    \mathbb P_0\bigl(S_t(s)\bigr)
    \leq
    \alpha
    =
    \mathbb P_0\bigl(S_{t_\alpha}(s)\bigr).
\]
If \(t \geq t_\alpha\), then
\(
    S_t(s) \subseteq S_{t_\alpha}(s)\),
and hence
\(
    \mathbb P_1\bigl(S_t(s)\bigr)
    \leq
    \mathbb P_1\bigl(S_{t_\alpha}(s)\bigr)\).
If \(t < t_\alpha\), then
\(
    S_{t_\alpha}(s) \subseteq S_t(s)\).
Together with
\(
    \mathbb P_0\bigl(S_t(s)\bigr)
    \leq
    \mathbb P_0\bigl(S_{t_\alpha}(s)\bigr)\),
this implies
\(
    \mathbb P_0\bigl(S_t(s) \setminus S_{t_\alpha}(s)\bigr)=0\).
Since \(P_1 \ll P_0\), it follows that
\(
    \mathbb P_1\bigl(S_t(s) \setminus S_{t_\alpha}(s)\bigr)=0\),
and therefore
\[
    \mathbb P_1\bigl(S_t(s)\bigr)
    =
    \mathbb P_1\bigl(S_{t_\alpha}(s)\bigr).
\]
Thus among all thresholds with false-positive rate at most \(\alpha\), the maximal true-positive rate is attained at a threshold with equality, and so
\[
\operatorname{ROC}_s(\alpha)=\mathbb  P_1(S_{t_\alpha}(s))
=\mathbb P\bigl(s(X_1)\ge t_\alpha\bigr).
\]
Next let
\(
F_0(t):=\mathbb P\bigl(s(X_0)\le t\bigr)\)
be the distribution function of \(s(X_0)\), and define its generalized inverse by
\[
Q_0(u):=\inf\{t:F_0(t)\ge u\},\qquad u\in(0,1).
\]
Because \(s(X_0)\) is continuous, we may take
\(
t_\alpha=Q_0(1-\alpha)\),
and then
\[
\mathbb  P_0(S_{t_\alpha}(s))
=
\mathbb P\bigl(s(X_0)\ge Q_0(1-\alpha)\bigr)
=
\alpha.
\]
Hence
\[
\operatorname{ROC}_s(\alpha)
=
\mathbb P\bigl(s(X_1)\ge Q_0(1-\alpha)\bigr).
\]
Therefore,
\[
\int_0^1 \operatorname{ROC}_s(\alpha)\,d\alpha
=
\int_0^1 \mathbb P\bigl(s(X_1)\ge Q_0(1-\alpha)\bigr)\,d\alpha.
\]
Substituting \(u=1-\alpha\), so that \(d\alpha=-du\), gives
\[
\int_0^1 \operatorname{ROC}_s(\alpha)\,d\alpha
=
\int_0^1 \mathbb P\bigl(s(X_1)\ge Q_0(u)\bigr)\,du.
\]
Now apply Tonelli's theorem:
\[
\int_0^1 \mathbb P\bigl(s(X_1)\ge Q_0(u)\bigr)\,du
=
\int_0^1 \mathbb E\bigl[\mathbf 1\{s(X_1)\ge Q_0(u)\}\bigr]\,du
=
\mathbb E\left[\int_0^1 \mathbf 1\{Q_0(u)\le s(X_1)\}\,du\right].
\]
So it remains to identify the quantity
\(
\int_0^1 \mathbf 1\{Q_0(u)\le z\}\,du\)
for a fixed \(z\in\mathbb R\).
It is easy to see that
\(
\{u\in[0,1]:Q_0(u)\le z\}=\{u\in[0,1]:u\le F_0(z)\}\).
Thus we obtain
\[
\int_0^1 \mathbf 1\{Q_0(u)\le z\}\,du
=\int_0^1 \mathbf 1\{u\le F_0(z)\}\,du=
F_0(z).
\]
Applying this with \(z=s(X_1)\), we get
\[
\int_0^1 \operatorname{ROC}_s(\alpha)\,d\alpha
=
\mathbb E\bigl[F_0(s(X_1))\bigr].
\]
We now identify this expectation as a comparison probability. By definition of \(F_0\),
\[
F_0(s(X_1))
=
\mathbb P\bigl(s(X_0)\le s(X_1)\mid X_1\bigr).
\]
Taking expectations and using independence of \(X_0\) and \(X_1\), we obtain
\[
\mathbb E\bigl[F_0(s(X_1))\bigr]
=
\mathbb P\bigl(s(X_0)\le s(X_1)\bigr).
\]
Hence
\[
\int_0^1 \operatorname{ROC}_s(\alpha)\,d\alpha
=
\mathbb P\bigl(s(X_0)\le s(X_1)\bigr).
\]
Finally, because \(s(X_0)\) has a continuous distribution, for every fixed \(a\in\mathbb R\),
\(
\mathbb P\bigl(s(X_0)=a\bigr)=0\).
Therefore
\[
\mathbb P\bigl(s(X_0)=s(X_1)\bigr)
=
\mathbb E\!\left[\mathbb P\bigl(s(X_0)=s(X_1)\mid X_1\bigr)\right]
=
0.
\]
We then conclude that
\(
\operatorname{AUC}(s)
=
\mathbb P\bigl(s(X_1)>s(X_0)\bigr)\).
Combining the previous displays yields
\(
\int_0^1 \operatorname{ROC}_s(\alpha)\,d\alpha
=
\operatorname{AUC}(s)\).
Thus we have proved the integral representation for AUC that
\[
\operatorname{AUC}(s)=\int_0^1 \operatorname{ROC}_s(\alpha)\,d\alpha.
\]
Now integrate over \(\alpha\). It follows that,
\[
\operatorname{AUC}(s)
=
\int_0^1 \operatorname{ROC}_s(\alpha)\,d\alpha
\le
\int_0^1 \operatorname{ROC}_{L_{\eta^*}}(\alpha)\,d\alpha
=
\operatorname{AUC}(L_{\eta^*})=\operatorname{AUC}(\Lambda_{\eta^*}).
\]
This completes the proof of AUROC optimality of the likelihood ratio.

\medskip
\noindent\textbf{Optimality for AUPRC:}

We first show that for each fixed recall level, a likelihood-ratio threshold minimizes the false-positive rate.
Fix \(r\in(0,1)\). Since \(L_{\eta^*}(X_1)\) has a continuous CDF under \(P_1\),
there exists a threshold \(c_r\) such that
\[
\mathbb  P_1(L_{\eta^*}(X)\ge c_r)=r.
\]
Define the likelihood-ratio threshold set
\(
B_r := \{x:L_{\eta^*}(x)\ge c_r\}\).
Then
\(
\mathbb  P_1(B_r)=r\).

We claim that among all measurable sets \(S\subseteq\mathcal X\) satisfying
\(
\mathbb  P_1(S)=r\),
the set \(B_r\) minimizes \(\mathbb  P_0(S)\).

Let \(S\) be any measurable set with \(\mathbb  P_1(S)=r\). Since
\(L_{\eta^*}(x)=p_1(x)/p_0(x)\), we have
\[
p_1(x)-c_r p_0(x)\ge 0
\quad \text{on } B_r;\qquad 
p_1(x)-c_r p_0(x)\le 0
\quad \text{on } B_r^c.
\]
Therefore,
\[
\int_{B_r\setminus S} (p_1-c_r p_0)\,dx \ge 0,
\qquad
\int_{S\setminus B_r} (p_1-c_r p_0)\,dx \le 0.
\]
Subtracting the second inequality from the first gives
\[
\int_{B_r} (p_1-c_r p_0)\,dx
\ge
\int_S (p_1-c_r p_0)\,dx.
\]
Expanding both sides,
\[
\mathbb  P_1(B_r)-c_r\mathbb  P_0(B_r)
\ge
\mathbb  P_1(S)-c_r\mathbb  P_0(S).
\]
Since \(\mathbb  P_1(B_r)=\mathbb  P_1(S)=r\), it holds that
\(
c_r\mathbb  P_0(B_r)
\le
c_r\mathbb  P_0(S)\).
Because $r\in(0,1)$ and $L_{\eta^*}(X)>0$, we have \(c_r>0\). Therefore we conclude
\[
\mathbb  P_0(B_r)\le \mathbb  P_0(S).
\]
So at each fixed recall level \(r\), the likelihood-ratio threshold set \(B_r\)
has the smallest false-positive rate among all measurable decision sets with
that recall.

We can then show that, at each recall level, the precision of any score is bounded by the precision of \(L_{\eta^*}\).
Now fix any score \(s\). Since \(s(X_1)\) has a continuous CDF under
\(P_1\), for each \(r\in(0,1)\) there exists a threshold \(t_r(s)\) such that
\(
\mathbb  P_1\bigl(s(X)\ge t_r(s)\bigr)=r\).
Define
\(
S_r(s) := \{x:s(x)\ge t_r(s)\}\).
Then
\(
\mathbb  P_1(S_r(s))=r\).
By the preceding argument, among all measurable sets with \(\mathbb  P_1(\cdot)=r\), the set \(B_r\)
minimizes the false-positive rate. Hence
\[
\mathbb  P_0\bigl(S_r(s)\bigr)
\ge
\mathbb  P_0(B_r).
\]
At recall \(r\), the precision of \(s\) and the precision of the likelihood-ratio score are
\[
\operatorname{Prec}_s^*(r)
=
\frac{\pi_1 r}
{\pi_1 r+\pi_0 \mathbb  P_0(S_r(s))},\qquad \operatorname{Prec}_{L_{\eta^*}}^*(r)
=
\frac{\pi_1 r}
{\pi_1 r+\pi_0 \mathbb  P_0(B_r)}.
\]
Since \(\mathbb  P_0(S_r(s))\ge \mathbb  P_0(B_r)\), we have
\[
\operatorname{Prec}_s^*(r)
\le
\operatorname{Prec}_{L_{\eta^*}}^*(r)
\qquad
\text{for every } r\in(0,1).
\]
Thus, the precision--recall curve of any admissible score is pointwise dominated,
as a function of recall, by that of the likelihood ratio. Finally, integrating the pointwise bound over \(r\in[0,1]\), we obtain
\[
\operatorname{AUPRC}(s)
=
\int_0^1 \operatorname{Prec}_s^*(r)\,dr
\le
\int_0^1 \operatorname{Prec}_{L_{\eta^*}}^*(r)\,dr
=
\operatorname{AUPRC}(L_{\eta^*})=\operatorname{AUPRC}(\Lambda_{\eta^*}).
\]

\medskip
    \noindent\textbf{Optimality for Best-Threshold Balanced Accuracy:}

First, we rewrite balanced accuracy in set form.
Let \(S\subseteq\mathcal X\) be any measurable set. By definition,
\(
\BA(S)=\frac12\bigl(\mathbb  P_1(S)+\mathbb  P_0(S^c)\bigr)\).
Using densities,
\(
\mathbb  P_1(S)=\int_S p_1(x)\,dx\),
\(\mathbb  P_0(S^c)=1-\int_S p_0(x)\,dx.
\)
Hence
\[
\BA(S)
=
\frac12
\left(
\int_S p_1(x)\,dx
+
1
-
\int_S p_0(x)\,dx
\right)
=
\frac12
+
\frac12
\int_S\bigl(p_1(x)-p_0(x)\bigr)\,dx.
\]
Therefore, for a thresholded score,
\[
\BA(s,\tau)
=
\BA(S_\tau(s))
=
\frac12
+
\frac12
\int_{S_\tau(s)}
\bigl(p_1(x)-p_0(x)\bigr)\,dx,
\]
where $S_\tau(s):=\{x:s(x)\ge \tau\}$.

Next we show that the Bayes set maximizes balanced accuracy over all measurable sets.
We have
\(
\BA(S)
=
\frac12
+
\frac12
\int_S(p_1-p_0)\).
To maximize this expression over measurable sets \(S\), the maximizing set includes exactly those \(x\) for which
\(
p_1(x)-p_0(x)\ge 0\).
Thus the optimal set is
\[
B^*
:=
\{x:p_1(x)\ge p_0(x)\}.
\]
Equivalently,
\[
p_1(x)\ge p_0(x)
\iff
\frac{p_1(x)}{p_0(x)}\ge 1
\iff
L_{\eta^*}(x)\ge 1,
\]
so
\(
B^*=\{x:L_{\eta^*}(x)\ge 1\}\).
Now let \(S\) be any measurable set. Then
\[
\BA(B^*)-\BA(S)
=
\frac12\int_{B^*}(p_1-p_0)\,dx
-
\frac12\int_S(p_1-p_0)\,dx.
\]
Splitting the difference into the two set differences gives
\[
\BA(B^*)-\BA(S)
=
\frac12\int_{B^*\setminus S}(p_1-p_0)\,dx
-
\frac12\int_{S\setminus B^*}(p_1-p_0)\,dx.
\]
On \(B^*\), we have \(p_1-p_0\ge 0\). On \(S\setminus B^*\subseteq (B^*)^c\), we have \(p_1-p_0<0\). Therefore,
\[
\BA(B^*)-\BA(S)
=
\frac12\int_{B^*\setminus S}|p_1-p_0|\,dx
+
\frac12\int_{S\setminus B^*}|p_1-p_0|\,dx.
\]
Hence
\[
\BA(B^*)-\BA(S)
=
\frac12
\int_{B^*\triangle S}
|p_1(x)-p_0(x)|\,dx
\ge 0.
\]
Thus \(B^*\) maximizes balanced accuracy over all measurable decision sets.

Let \(s\) be any measurable score. For every threshold \(\tau\),
\(
\BA(s,\tau)=\BA(S_\tau(s))\).
Since \(B^*\) maximizes balanced accuracy over all measurable sets,
\(
\BA(s,\tau)\le \BA(B^*)\),
for every $\tau$.
Taking the supremum over \(\tau\) yields
\[
\sup_{\tau\in\mathbb R}\BA(s,\tau)
\le
\BA(B^*).
\]
Also, taking the likelihood ratio itself as the score,
\[
\BA(B^*)
=
BA\bigl(L_{\eta^*},1)
\le
\sup_{\tau\in\mathbb R}\BA(L_{\eta^*},\tau).
\]
On the other hand, every threshold set of \(L_{\eta^*}\) is a measurable set, so it also gives
\[
\BA(L_{\eta^*},\tau)\le \BA(B^*)
\qquad
\text{for every }\tau.
\]
Taking the supremum over \(\tau\),
\[
\sup_{\tau\in\mathbb R}\BA(L_{\eta^*},\tau)
\le
\BA(B^*).
\]
Combining the two inequalities,
\begin{equation}\label{eq:best-BA-event}
\sup_{\tau\in\mathbb R}\BA(L_{\eta^*},\tau)
=
\BA(B^*).
\end{equation}
Therefore, this proves the claim for balanced accuracy.

\medskip
    \noindent\textbf{Optimality for Best-Threshold $\F_1$ Score:}

We first rewrite $\F_1$ as a function of recall and false-positive rate. For any measurable set \(S\), define
\(
r:=\mathbb  P_1(S)\),
\(q:=\mathbb  P_0(S)\).
Then
\(
\F_1(S)
=
\frac{2\pi_1 r}
{\pi_1(1+r)+\pi_0 q}\).
Define
\[
G(r,q)
:=
\frac{2\pi_1 r}
{\pi_1(1+r)+\pi_0 q},
\qquad
(r,q)\in[0,1]^2.
\]
Then
\(
\F_1(S)=G(\mathbb  P_1(S),\mathbb  P_0(S))\).
The partial derivatives of \(G\) are
\[
\frac{\partial G}{\partial r}
=
\frac{2\pi_1(\pi_1+\pi_0 q)}
{\bigl(\pi_1(1+r)+\pi_0 q\bigr)^2}
>0,\qquad \frac{\partial G}{\partial q}
=
-\frac{2\pi_0\pi_1 r}
{\bigl(\pi_1(1+r)+\pi_0 q\bigr)^2}
\le 0.
\]
Thus $\F_1$ is increasing in recall \(r\) and decreasing in false-positive rate \(q\). Therefore, among sets with the same recall, the one with smaller false-positive rate has at least as large $\F_1$.

Next we show that, for each fixed recall level, a likelihood-ratio threshold minimizes false-positive rate.
Fix \(r\in(0,1)\). Since \(L_{\eta^*}(X_1)\) has a continuous distribution under \(P_1\), there exists a threshold \(c_r\) such that
\(
\mathbb  P_1(L_{\eta^*}(X)\ge c_r)=r\).
Define
\(
B_r:=\{x:L_{\eta^*}(x)\ge c_r\}\).
Then
\(
\mathbb  P_1(B_r)=r\).
We claim that \(B_r\) minimizes \(\mathbb  P_0(S)\) among all measurable sets \(S\) satisfying \(\mathbb  P_1(S)=r\). Let \(S\) be any measurable set with \(\mathbb  P_1(S)=r\). On \(B_r\), we have \(L_{\eta^*}(x)\ge c_r\), and hence
\[
p_1(x)-c_rp_0(x)\ge 0.
\]
On \(B_r^c\), we have \(L_{\eta^*}(x)<c_r\), and hence
\[
p_1(x)-c_rp_0(x)\le 0.
\]
Therefore,
\[
\int_{B_r\setminus S}(p_1-c_rp_0)\,dx\ge 0,
\qquad
\int_{S\setminus B_r}(p_1-c_rp_0)\,dx\le 0.
\]
It follows by subtraction that
\[
\int_{B_r}(p_1-c_rp_0)\,dx
\ge
\int_S(p_1-c_rp_0)\,dx.
\]
Expanding both sides gives
\[
\mathbb  P_1(B_r)-c_r\mathbb  P_0(B_r)
\ge
\mathbb  P_1(S)-c_r\mathbb  P_0(S).
\]
Since \(\mathbb  P_1(B_r)=\mathbb  P_1(S)=r\), this becomes
\[
r-c_r\mathbb  P_0(B_r)
\ge
r-c_r\mathbb  P_0(S).
\]
Because \(r\in(0,1)\), the corresponding threshold satisfies \(c_r>0\). Therefore,
\[
\mathbb  P_0(B_r)\le \mathbb  P_0(S).
\]
Thus, among all measurable sets with recall \(r\), the likelihood-ratio threshold set \(B_r\) has the smallest false-positive rate.

Next, we show that every measurable set is $\F_1$-dominated by likelihood-ratio threshold sets.
Let \(S\subseteq\mathcal X\) be any measurable set and set
\(
r:=\mathbb  P_1(S)\).
If \(r=0\), then
\(
\F_1(S)=0\),
so \(S\) cannot improve over the supremum of the likelihood-ratio threshold family.
Now suppose \(r\in(0,1)\). There exists a likelihood-ratio threshold set \(B_r\) such that
\[
\mathbb  P_1(B_r)=r,
\qquad
\mathbb  P_0(B_r)\le \mathbb  P_0(S).
\]
Using the monotonicity of \(G\),
\[
\F_1(B_r)
=
G(\mathbb  P_1(B_r),\mathbb  P_0(B_r))
=
G(r,\mathbb  P_0(B_r))
\ge
G(r,\mathbb  P_0(S))
=
\F_1(S).
\]
It remains to consider the case \(r=1\). Let \(r_n\uparrow 1\) with \(r_n\in(0,1)\). For each \(n\), let
\[
B_{r_n}:=\{x:L_{\eta^*}(x)\ge c_{r_n}\}
\]
satisfy
\(
\mathbb  P_1(B_{r_n})=r_n\).
Since \(\mathbb  P_1(S)=1\), we have
\[
\mathbb  P_1(S\cap B_{r_n})=\mathbb  P_1(B_{r_n})=r_n.
\]
Also, we know that \(B_{r_n}\) minimizes \(P_0\) among all measurable sets with \(P_1\)-mass \(r_n\), so
\[
\mathbb  P_0(B_{r_n})
\le
\mathbb  P_0(S\cap B_{r_n})
\le
\mathbb  P_0(S).
\]
Therefore,
\[
\F_1(B_{r_n})
=
G(r_n,\mathbb  P_0(B_{r_n}))
\ge
G(r_n,\mathbb  P_0(S)).
\]
Letting \(n\to\infty\), continuity of \(G\) gives
\[
\sup_c \F_1(\{L_{\eta^*}\ge c\})
\ge
\limsup_{n\to\infty}\F_1(B_{r_n})
\ge
G(1,\mathbb  P_0(S))
=
\F_1(S).
\]
Hence every measurable set \(S\) is dominated, in $\F_1$ value, by the likelihood-ratio threshold family in the sense that
\[
\F_1(S)
\le
\sup_c \F_1(\{L_{\eta^*}\ge c\}).
\]
Taking the supremum over all measurable $B\subseteq \mathcal X$ gives
\[
\sup_{B\subseteq \mathcal X}\F_1(B)
\le
\sup_c \F_1(\{L_{\eta^*}\ge c\}).
\]
The reverse inequality is immediate since every set $\{L_{\eta^*}\geq c\}$ is measurable. Hence,
\[
\sup_{B\subseteq \mathcal X}\F_1(B)
=
\sup_c \F_1(\{L_{\eta^*}\ge c\})=\sup_{\tau\in\mathbb R}\F_1(\Lambda_{\eta^*},\tau).
\]

Now let \(s\) be any measurable score and let \(\tau\in\mathbb R\). Its threshold set
\(
S_\tau(s)=\{x:s(x)\ge \tau\}\)
is measurable, so
\[
\F_1(s,\tau)
=
\F_1(S_\tau(s))
\le
\sup_c \F_1(\{L_{\eta^*}\ge c\})
=
\sup_{\tau\in\mathbb R}\F_1(L_{\eta^*},\tau).
\]
Taking the supremum over \(\tau\) gives
\[
\sup_{\tau\in\mathbb R}\F_1(s,\tau)
\le\sup_{\tau\in\mathbb R}\F_1(L_{\eta^*},\tau)=\sup_{\tau\in\mathbb R}\F_1(\Lambda_{\eta^*},\tau)=\sup_{B\subseteq \mathcal X}\F_1(B).
\]
This proves the lemma.
\end{proof}

\begin{proof}[Proof of Lemma \ref{lemma:empiricalestimation}]
For the empirical estimation error
$\|\hat{\theta}_{\rm raw}-\theta_{\rm raw}^*\|$, first note that
by the boundedness condition on the centered individual Hessians and the
matrix Bernstein inequality, for sufficiently large $n_0+n_1$, with
probability at least $1-\delta/7$,
\[
\left\|
\nabla^2\hat R_{\rm raw}(\theta_{\rm raw}^*)
-
\nabla^2R_{\rm raw}(\theta_{\rm raw}^*)
\right\|_{\rm op}
\le \frac{\lambda}{4}.
\]
Indeed, by the definition of $\hat R_{\rm raw}$ and
$R_{\rm raw}$,
\begin{align*}
&\nabla^2\hat R_{\rm raw}(\theta_{\rm raw}^*)
-
\nabla^2R_{\rm raw}(\theta_{\rm raw}^*) \\
&=
\frac{1}{n_0+n_1}
\sum_{i=1}^{n_0}
\left\{
\nabla^2\ell(\theta_{\rm raw}^*;x_i^{(0)},0)
-
E_{P_0}\nabla^2\ell(\theta_{\rm raw}^*;X,0)
\right\} \\
&\quad+
\frac{1}{n_0+n_1}
\sum_{i=1}^{n_1}
\left\{
\nabla^2\ell(\theta_{\rm raw}^*;x_i^{(1)},1)
-
E_{P_1}\nabla^2\ell(\theta_{\rm raw}^*;X,1)
\right\}.
\end{align*}
The matrices inside the two sums are independent, centered, and
self-adjoint. By the bounded centered-Hessian condition, each summand
after multiplication by $1/(n_0+n_1)$ has operator norm at most
$B_H/(n_0+n_1)$ almost surely. Furthermore, their matrix variance
proxy is bounded by
\[
\frac{(n_0+n_1)B_H^2}{(n_0+n_1)^2}
=
\frac{B_H^2}{n_0+n_1}.
\]
Therefore, the self-adjoint matrix Bernstein inequality (e.g. Theorem 1.4 of \cite{tropp2012user}) implies that,
for every $t>0$,
\[
P\left(
\left\|
\nabla^2\hat R_{\rm raw}(\theta_{\rm raw}^*)
-
\nabla^2R_{\rm raw}(\theta_{\rm raw}^*)
\right\|_{\rm op}
\ge t
\right)
\le
2d\exp\left\{
-
\frac{(n_0+n_1)t^2}
{2B_H^2+\frac{2}{3}B_Ht}
\right\}.
\]
Taking $t=\lambda/4$ gives
\[
P\left(
\left\|
\nabla^2\hat R_{\rm raw}(\theta_{\rm raw}^*)
-
\nabla^2R_{\rm raw}(\theta_{\rm raw}^*)
\right)
\ge \frac{\lambda}{4}
\right\|_{\rm op}
\le
2d\exp\left\{
-
\frac{(n_0+n_1)\lambda^2}
{32B_H^2+\frac{8}{3}B_H\lambda}
\right\}.
\]
Consequently, if
\[
n_0+n_1
\ge
\left(
\frac{32B_H^2}{\lambda^2}
+
\frac{8B_H}{3\lambda}
\right)
\log\left(\frac{14d}{\delta}\right),
\]
then, with probability at least $1-\delta/7$,
\[
\left\|
\nabla^2\hat R_{\rm raw}(\theta_{\rm raw}^*)
-
\nabla^2R_{\rm raw}(\theta_{\rm raw}^*)
\right\|_{\rm op}
\le
\frac{\lambda}{4}.
\]
On this event, for every $\theta$ satisfying
$\|\theta-\theta_{\rm raw}^*\|\le r$, the local Lipschitz condition on
the individual Hessians gives
\[
\left\|
\nabla^2\hat R_{\rm raw}(\theta)
-
\nabla^2\hat R_{\rm raw}(\theta_{\rm raw}^*)
\right\|_{\rm op}
\le
L_H\|\theta-\theta_{\rm raw}^*\|
\le
L_Hr
\le
\frac{\lambda}{4}.
\]
Together with
$\nabla^2R_{\rm raw}(\theta_{\rm raw}^*)\succeq \lambda I$, this yields
\begin{align*}
\nabla^2\hat R_{\rm raw}(\theta)
&\succeq
\nabla^2R_{\rm raw}(\theta_{\rm raw}^*)
-
\left\|
\nabla^2\hat R_{\rm raw}(\theta_{\rm raw}^*)
-
\nabla^2R_{\rm raw}(\theta_{\rm raw}^*)
\right\|_{\rm op} I \\
&\qquad
-
\left\|
\nabla^2\hat R_{\rm raw}(\theta)
-
\nabla^2\hat R_{\rm raw}(\theta_{\rm raw}^*)
\right\|_{\rm op} I \\
&\succeq
\left(\lambda-\frac{\lambda}{4}-\frac{\lambda}{4}\right)I
=
\frac{\lambda}{2}I
\end{align*}
for all $\theta$ satisfying
$\|\theta-\theta_{\rm raw}^*\|\le r$.

We next show that $\hat\theta_{\rm raw}$ lies in this local
neighborhood. Assume temporarily that
\begin{equation}\label{eq:temporo-assume-small}
\left\|
\nabla\hat R_{\rm raw}(\theta_{\rm raw}^*)
\right\|
\le
\frac{\lambda r}{4};
\end{equation}
this condition will be verified below by the gradient concentration
bound. For every $\theta$ satisfying
$\|\theta-\theta_{\rm raw}^*\|=r$, the fundamental theorem of calculus
gives
\begin{align*}
&(\theta-\theta_{\rm raw}^*)^\top
\nabla\hat R_{\rm raw}(\theta) \\
&=
(\theta-\theta_{\rm raw}^*)^\top
\nabla\hat R_{\rm raw}(\theta_{\rm raw}^*)
+
\int_0^1
(\theta-\theta_{\rm raw}^*)^\top
\nabla^2\hat R_{\rm raw}
\bigl(
\theta_{\rm raw}^*
+t(\theta-\theta_{\rm raw}^*)
\bigr)
(\theta-\theta_{\rm raw}^*)
\,dt \\
&\ge
-r
\left\|
\nabla\hat R_{\rm raw}(\theta_{\rm raw}^*)
\right\|
+
\frac{\lambda r^2}{2}
\ge
\frac{\lambda r^2}{4}
>0.
\end{align*}
Since $\hat R_{\rm raw}$ is convex, the preceding outward-gradient
condition implies that
\[
\|\hat\theta_{\rm raw}-\theta_{\rm raw}^*\|\le r.
\]
Indeed, suppose to the contrary that
$\|\hat\theta_{\rm raw}-\theta_{\rm raw}^*\|>r$, and let
\[
\theta
=
\theta_{\rm raw}^*
+
\frac{r}{\|\hat\theta_{\rm raw}-\theta_{\rm raw}^*\|}
\bigl(\hat\theta_{\rm raw}-\theta_{\rm raw}^*\bigr).
\]
Then $\|\theta-\theta_{\rm raw}^*\|=r$ and
\[
\hat\theta_{\rm raw}-\theta
=
\left(
\frac{\|\hat\theta_{\rm raw}-\theta_{\rm raw}^*\|}{r}-1
\right)
(\theta-\theta_{\rm raw}^*),
\]
where the coefficient is strictly positive. Since
$\hat\theta_{\rm raw}$ minimizes the convex function
$\hat R_{\rm raw}$,
\[
0
\ge
\hat R_{\rm raw}(\hat\theta_{\rm raw})
-
\hat R_{\rm raw}(\theta)
\ge
\nabla\hat R_{\rm raw}(\theta)^\top
(\hat\theta_{\rm raw}-\theta).
\]
On the other hand, the outward-gradient condition gives
\begin{align*}
\nabla\hat R_{\rm raw}(\theta)^\top
(\hat\theta_{\rm raw}-\theta)
&=
\left(
\frac{\|\hat\theta_{\rm raw}-\theta_{\rm raw}^*\|}{r}-1
\right)
\nabla\hat R_{\rm raw}(\theta)^\top
(\theta-\theta_{\rm raw}^*)>0,
\end{align*}
which is a contradiction. Therefore,
\[
\|\hat\theta_{\rm raw}-\theta_{\rm raw}^*\|\le r.
\]
Consequently, the entire line segment joining
$\theta_{\rm raw}^*$ and $\hat\theta_{\rm raw}$ is contained in the
region on which
\[
\nabla^2\hat R_{\rm raw}(\theta)\succeq \frac{\lambda}{2}I.
\]

Since
$\nabla\hat R_{\rm raw}(\hat\theta_{\rm raw})=0$, another
application of the fundamental theorem of calculus yields
\[
-\nabla\hat R_{\rm raw}(\theta_{\rm raw}^*)
=
\int_0^1
\nabla^2\hat R_{\rm raw}
\bigl(
\theta_{\rm raw}^*
+t(\hat\theta_{\rm raw}-\theta_{\rm raw}^*)
\bigr)
(\hat\theta_{\rm raw}-\theta_{\rm raw}^*)
\,dt.
\]
Taking the inner product with
$\hat\theta_{\rm raw}-\theta_{\rm raw}^*$, we obtain
\begin{align*}
&-
(\hat\theta_{\rm raw}-\theta_{\rm raw}^*)^\top
\nabla\hat R_{\rm raw}(\theta_{\rm raw}^*) \\
&\qquad =
\int_0^1
(\hat\theta_{\rm raw}-\theta_{\rm raw}^*)^\top
\nabla^2\hat R_{\rm raw}
\bigl(
\theta_{\rm raw}^*
+t(\hat\theta_{\rm raw}-\theta_{\rm raw}^*)
\bigr)
(\hat\theta_{\rm raw}-\theta_{\rm raw}^*)
\,dt \\
&\qquad \ge
\frac{\lambda}{2}
\left\|
\hat\theta_{\rm raw}-\theta_{\rm raw}^*
\right\|^2.
\end{align*}
On the other hand, by the Cauchy--Schwarz inequality,
\[
-
(\hat\theta_{\rm raw}-\theta_{\rm raw}^*)^\top
\nabla\hat R_{\rm raw}(\theta_{\rm raw}^*)
\le
\left\|
\hat\theta_{\rm raw}-\theta_{\rm raw}^*
\right\|
\left\|
\nabla\hat R_{\rm raw}(\theta_{\rm raw}^*)
\right\|.
\]
Therefore,
\[
\left\|
\hat\theta_{\rm raw}-\theta_{\rm raw}^*
\right\|
\le
\frac{2}{\lambda}
\left\|
\nabla\hat R_{\rm raw}(\theta_{\rm raw}^*)
\right\|.
\]

Next we give a bound for $\left\|
\nabla\hat R_{\rm raw}(\theta_{\rm raw}^*)
\right\|$ to verify the previous assumption and obtain a convergence bound for the parameter.
    
    By optimality, $\bE\nabla\hat{\cR}_\raw (\theta_\raw^*)=\nabla \cR_\raw (\theta_\raw^*)=0$. Therefore,
\begin{align}\label{eq:1}
    \nonumber&\left\|\nabla\hat{\cR}_\raw(\theta_\raw^*)\right\|=\left\|\frac{1}{n_0+n_1}\sum_{i=1}^{n_0}\nabla\ell(\theta_\raw^*;x_i^{(0)},0)+\frac{1}{n_0+n_1}\sum_{i=1}^{n_1}\nabla\ell(\theta_\raw^*,x_i^{(1)},1)-\bE\nabla\hat{\cR}_\raw (\theta_\raw^*)\right\|\\
    \nonumber&\leq\frac{n_0}{n_0+n_1}\left\|\frac{1}{n_0}\sum_{i=1}^{n_0}\nabla\ell(\theta_\raw^*;x_i^{(0)},0)-\bE_{\cP_0}\left[\nabla\ell(\theta_\raw^*;x_i^{(0)},0)\right]\right\|\\
    &\qquad\qquad+\frac{n_1}{n_0+n_1}\left\|\frac{1}{n_1}\sum_{i=1}^{n_1}\nabla\ell(\theta_\raw^*,x_i^{(1)},1)-\bE_{\cP_1}\left[\nabla\ell(\theta_\raw^*,x_i^{(1)},1)\right]\right\|
\end{align}
By union bound and Hoeffding's inequality, for fixed $\theta$,
\begin{align*}
&\mathbb{P}\left(\left\|\frac{1}{n_0}\sum_{i=1}^{n_0}\nabla\ell(\theta_\raw^*;x_i^{(0)},0)-\bE_{\cP_0}\left[\nabla\ell(\theta_\raw^*;x_i^{(0)},0)\right]\right\|>t\right)\\
&\leq 
\mathbb{P}\left(\exists i\in|d|,s.t.\left|\left[\frac{1}{n_0}\sum_{i=1}^{n_0}\nabla\ell(\theta_\raw^*;x_i^{(0)},0)\right]_i-\left[\bE_{\cP_0}\left[\nabla\ell(\theta_\raw^*;x_i^{(0)},0)\right]\right]_i\right|>\frac{t}{\sqrt{d}}\right)\\
&\leq \sum_{i=1}^d \mathbb{P}\left(\left|\left[\frac{1}{n_0}\sum_{i=1}^{n_0}\nabla\ell(\theta_\raw^*;x_i^{(0)},0)\right]_i-\left[\bE_{\cP_0}\left[\nabla\ell(\theta_\raw^*;x_i^{(0)},0)\right]\right]_i\right|>\frac{t}{\sqrt{d}}\right)\leq 2d\exp\left\{-\frac{t^2n_0}{2dB^2}\right\},
\end{align*}
where $B>0$ is the constant from the bounded-support assumption. Taking \[
t=\sqrt{\frac{2B^2d\log(2d/\delta)}{n_0}},
\]
we have with probability at least $1-\delta$, 
\begin{align}\label{eq:2}
\left\|\frac{1}{n_0}\sum_{i=1}^{n_0}\nabla\ell(\theta_\raw^*;x_i^{(0)},0)-\bE_{\cP_0}\left[\nabla\ell(\theta_\raw^*;x_i^{(0)},0)\right]\right\|\leq \sqrt{\frac{2B^2d\log(2d/\delta)}{n_0}}.
\end{align}
By essentially the same procedure, we have with probability at least $1-\delta$,
\begin{align}\label{eq:3}
\left\|\frac{1}{n_1}\sum_{i=1}^{n_1}\nabla\ell(\theta_\raw^*;x_i^{(1)},1)-\bE_{\cP_1}\left[\nabla\ell(\theta_\raw^*;x_i^{(1)},1)\right]\right\|\leq \sqrt{\frac{2B^2d\log(2d/\delta)}{n_1}}.
\end{align}
Combining the bounded-Hessian event and equations (\ref{eq:1}), (\ref{eq:2}) and (\ref{eq:3}) yields that with probability at least $1-\frac{3}{7}\delta$,
\begin{align*}
    \left\|\nabla\hat{\cR}_\raw(\theta_\raw^*)\right\|\leq \frac{\sqrt{n_0}+\sqrt{n_1}}{n_0+n_1}\sqrt{2B^2d\log(14d/\delta)}\leq \sqrt{\frac{4B^2d\log(14d/\delta)}{n_0+n_1}}.
\end{align*}
Therefore, equation (\ref{eq:temporo-assume-small}) is verified for large enough sample size.
Furthermore, with probability at least $1-\frac{3}{7}\delta$, it holds by absorbing constants into some constant $c>0$ that
    \begin{align}\label{eq:4}
        \left\|\hat{\theta}_{\raw}-\theta^*_{\raw}\right\|\leq \frac{c}{\lambda}\sqrt{\frac{B^2d\log(14d/\delta)}{n_0+n_1}}.    \end{align}
The same argument gives
\begin{align*}
        \left\|\hat{\theta}_{\aug}-\theta^*_{\aug}\right\|\leq \frac{2}{\lambda}\left\|\nabla\hat{\cR}_\aug (\theta_\aug^*)\right\|.
    \end{align*}
    Also,
\begin{align*}
    \nonumber&\left\|\nabla\hat{\cR}_\aug(\theta_\aug^*)\right\|\leq\frac{n_0}{n_0+n_1+\tilde{n}}\left\|\frac{1}{n_0}\sum_{i=1}^{n_0}\nabla\ell(\theta_\aug^*;x_i^{(0)},0)-\bE_{\cP_0}\left[\nabla\ell(\theta_\aug^*;x_i^{(0)},0)\right]\right\|\\
    &\qquad\qquad+\frac{n_1}{n_0+n_1+\tilde{n}}\left\|\frac{1}{n_1}\sum_{i=1}^{n_1}\nabla\ell(\theta_\aug^*,x_i^{(1)},1)-\bE_{\cP_1}\left[\nabla\ell(\theta_\aug^*,x_i^{(1)},1)\right]\right\|\\
    &\qquad\qquad+\frac{\tilde{n}}{n_0+n_1+\tilde{n}}\left\|\frac{1}{\tilde{n}}\sum_{i=1}^{\tilde{n}}\nabla\ell(\theta_\aug^*;\tilde{x}_i,1)-\bE_{\cP_{\rm syn}}\left[\nabla\ell(\theta_\aug^*;\tilde{x}_i,1)\right]\right\|.
\end{align*}
Similarly, by a union bound and Hoeffding’s inequality, and combining with the bounded-Hessian event, with probability at least $1-\frac{4}{7} \delta$,
\begin{align*}
    \left\|\nabla\hat{\cR}_\aug(\theta_\aug^*)\right\|\leq \frac{\sqrt{n_0}+\sqrt{n_1}+\sqrt{\tilde{n}}}{n_0+n_1+\tilde{n}}\sqrt{2B^2d\log(14d/\delta)}\leq \sqrt{\frac{6B^2d\log(14d/\delta)}{n_0+n_1+\tilde{n}}}.
\end{align*}
That is, after absorbing the constants,
\begin{align}\label{eq:5}
    \left\|\hat{\theta}_{\aug}-\theta^*_{\aug}\right\|\leq \frac{c}{\lambda}\sqrt{\frac{B^2d\log(14d/\delta)}{n_0+n_1+\tilde{n}}}.
\end{align}
Finally, combining equations (\ref{eq:4}) and (\ref{eq:5}), we have for some constant $c>0$, with probability at least $1-\delta$, the following holds simultaneously:
\begin{align*}
        \left\|\hat{\theta}_{\raw}-\theta^*_{\raw}\right\|\leq \frac{c}{\lambda}\sqrt{\frac{B^2d\log(14d/\delta)}{n_0+n_1}}, \quad
        \left\|\theta^*_{\aug}-\hat{\theta}_{\aug}\right\|\leq \frac{c}{\lambda}\sqrt{\frac{B^2d\log(14d/\delta)}{n_0+n_1+\tilde{n}}}.    \end{align*}
The individual bounds hold by adjusting the probability parameter $\delta$.
\end{proof}

\begin{proof}[Proof of Lemma \ref{lemma:sytheticerror}]
    For a differentiable function $\cR_\rho(\theta)$ with minimizer $\theta_\rho^*$ satisfying
    \[
    \cR_\rho(\theta_\rho^*)\geq \cR_\rho(\theta)+\nabla \cR_\rho(\theta)^T(\theta_\rho^*-\theta)+\frac{\mu}{2}\left\|\theta_\rho^*-\theta\right\|^2,
    \]
    by assumption, we have
    \begin{align*}
    \nonumber \cR_\rho(\theta)-\cR_\rho(\theta_{\rho}^*)&\leq \nabla \cR_\rho(\theta)^\top (\theta-\theta_\rho^*)-\frac{\mu}{2}\|\theta-\theta_\rho^*\|^2\\
    &\leq \frac{1}{2\mu}\|\nabla \cR_\rho(\theta)\|^2+\frac{\mu}{2}\|\theta-\theta_\rho^*\|^2-\frac{\mu}{2}\|\theta-\theta_\rho^*\|^2=\frac{1}{2\mu}\|\nabla \cR_\rho(\theta)\|^2,
    \end{align*}
    where the second inequality follows from the fact that $0\leq \|\frac{1}{\sqrt{\mu}}\nabla \cR_{\rho}(\theta)-\sqrt{\mu}(\theta-\theta_\rho^*)\|^2$.
    Taking $\theta_{\aug}^*$, we have
    \begin{align*}
       \nonumber \cR_\rho(\theta_\aug^*)-\cR_\rho(\theta_{\rho}^*)\leq \frac{1}{2\mu}\|\nabla \cR_\rho(\theta_\aug^*)\|^2. 
    \end{align*}
    Recalling the definitions of $\cR_\rho(\theta)$ and $\cR_{\aug}(\theta)$, we have
    \begin{align*}
        \cR_{\aug}(\theta)=\cR_\rho(\theta)+\frac{\tilde{n}}{n_0+n_1+\tilde{n}}\left[\mathbb{E}_{P_{\mathrm{syn}}}\ell(\theta;x,1)-\mathbb{E}_{P_{1}}\ell(\theta;x,1)\right].
    \end{align*}
    Taking the gradient with respect to the parameter and evaluating at $\theta_\aug^*$, by the first order optimality, we have
    \begin{align*}
        0=\nabla\cR_{\aug}(\theta_\aug^*)=\nabla\cR_\rho(\theta_\aug^*)+\frac{\tilde{n}}{n_0+n_1+\tilde{n}}\left[\mathbb{E}_{P_{\mathrm{syn}}}\nabla\ell(\theta_\aug^*;x,1)-\mathbb{E}_{P_{1}}\nabla\ell(\theta_\aug^*;x,1)\right].
    \end{align*}
    That is
    \begin{align*}
        \|\nabla\cR_\rho(\theta_\aug^*)\|=\frac{\tilde{n}}{n_0+n_1+\tilde{n}}\left\|\mathbb{E}_{P_{\mathrm{syn}}}\nabla\ell(\theta_\aug^*;x,1)-\mathbb{E}_{P_{1}}\nabla\ell(\theta_\aug^*;x,1)\right\|
    \end{align*}
    By Kantorovich-Rubinstein duality, for two distributions $P$ and $Q$,
    \begin{align*}
      W_1(P,Q)  =\sup_{\phi:Lip(\phi)\leq 1}\left\|\bE_P[\phi]-\bE_Q[\phi]\right\|.
    \end{align*}
    Define $\phi:=\nabla\ell/L_g$. Then $Lip(\phi)\leq 1$, hence,
    \begin{align*}
        \left\|\mathbb{E}_{P_{\mathrm{syn}}}\nabla\ell(\theta_\aug^*;x,1)-\mathbb{E}_{P_{1}}\nabla\ell(\theta_\aug^*;x,1)\right\|\leq L_g W_1(P_{\mathrm{syn}},P_1)\leq L_g\epsilon_{\mathrm{syn}}.
    \end{align*}
Combining this with the quadratic-growth assumption near the minimizer,
    \begin{align*}
        \left\|\theta_\aug^*-\theta_\rho^*\right\|^2\leq \frac{1}{L_{\rho}}\left(\cR_\rho(\theta_\aug^*)-\cR_\rho(\theta_\rho^*)\right)\leq \frac{1}{2\mu L_{\rho}}\|\nabla\cR_\rho(\theta_\aug^*)\|^2.
    \end{align*}
    Thus, we have
    \begin{align*}
        \left\|\theta_\aug^*-\theta_\rho^*\right\|\leq \sqrt{\frac{1}{2\mu L_{\rho}}}\frac{\tilde{n}}{n_0+n_1+\tilde{n}}\left\|\mathbb{E}_{P_{\mathrm{syn}}}\nabla\ell(\theta_\aug^*;x,1)-\mathbb{E}_{P_{1}}\nabla\ell(\theta_\aug^*;x,1)\right\|\\
        \leq \frac{L_g}{\sqrt{2\mu L_{\rho}}}\frac{\tilde{n}}{n_0+n_1+\tilde{n}}\epsilon_{\mathrm{syn}}.
    \end{align*}
\end{proof}

\end{document}